\documentclass[11pt]{article}

\usepackage[numbers,compress]{natbib}

\usepackage[top=1in, left=1in, right=1in, bottom=1in]{geometry}

\usepackage[utf8]{inputenc}
\usepackage[T1]{fontenc}
\usepackage{fix-cm}
\usepackage{microtype}

\usepackage{amsmath}
\usepackage{amsfonts}
\usepackage{amssymb}
\usepackage{amsthm}
\usepackage{mathtools}

\theoremstyle{plain}
\newtheorem{theorem}{Theorem}
\newtheorem{lemma}{Lemma}

\newtheorem{corollary}{Corollary}
\theoremstyle{definition}
\newtheorem{definition}{Definition}
\newtheorem{assumption}{Assumption}

\usepackage{booktabs}
\usepackage{multirow}
\usepackage{makecell}
\usepackage{colortbl}
\usepackage{graphicx}
\usepackage{subfig}
\usepackage{wrapfig}
\usepackage{float}
\usepackage{placeins}
\usepackage{rotating}
\usepackage{tikz}
\usetikzlibrary{arrows.meta,positioning}

\usepackage{enumitem}

\usepackage{listings}

\usepackage[vlined,linesnumbered,ruled,resetcount]{algorithm2e}

\usepackage{xcolor}
\usepackage[colorlinks,linkcolor=magenta,filecolor=blue,citecolor=blue,urlcolor=blue]{hyperref}
\usepackage{footnotehyper}
\usepackage{tablefootnote}
\usepackage{url}

\usepackage{pifont}
\usepackage{nicefrac}
\usepackage{xspace}

\newcommand{\name}{BEHAVE\xspace}
\newcommand{\behavesim}{BEHAVE-Sim\xspace}

\newcommand{\datasetsize}{600\xspace}
\newcommand{\trainingdatasetsize}{540\xspace}
\newcommand{\evaluationdatasetsize}{60\xspace}

\newcommand{\selfimprovementrounds}{5\xspace}

\lstdefinestyle{papercode}{
  basicstyle=\ttfamily\footnotesize,
  keywordstyle=\bfseries,
  commentstyle=\itshape,
  columns=fullflexible,
  keepspaces=true,
  showstringspaces=false,
  tabsize=4,
  numbers=none,
  frame=single,
  framerule=0.4pt,
  rulecolor=\color{black!30},
  framesep=4pt,
  xleftmargin=4pt,
  xrightmargin=4pt,
  breaklines=true,
  breakatwhitespace=true,
  breakindent=1em,
  aboveskip=\smallskipamount,
  belowskip=\smallskipamount
}

\title{BEHAVE: Functional \underline{Be}havior Modeling Enables Self-Improving Agents for \underline{Ha}rdware Design and \underline{Ve}rification}

\author{%
  Yuheng Wu\textsuperscript{1,}\thanks{Equal contribution.}\quad
  Berk Gokmen\textsuperscript{1,*}\quad
  Sujeeth Jinesh\textsuperscript{1,*}\quad
  Lauren McLane\textsuperscript{1,*}\\[4pt]
  Aarav Wattal\textsuperscript{1}\quad
  Qi Yang Huang\textsuperscript{2}\quad
  Zhaozhuo Xu\textsuperscript{2}\quad
  Thierry Tambe\textsuperscript{1}\\[6pt]
  {\normalsize
    \textsuperscript{1}Stanford University\quad
    \textsuperscript{2}Workato, Inc.}\\[3pt]
  {\small\ttfamily
    \{yuhengwu,ttambe\}@stanford.edu}%
}
\date{September 28, 2026}

\begin{document}

\maketitle
\vspace{-8pt} 

\begin{abstract}
Developing agents for hardware design and verification requires reliable correctness feedback. As a hardware specification may permit correct implementations with different latencies, matching design and reference outputs cycle by cycle can reject valid designs. To address this, we introduce \name, an agentic framework for multi-turn joint hardware design and verification through \emph{functional behavior modeling}. We define \emph{Behavior IR} to express task functionality as executable \emph{behavior models} without prescribing implementation timing beyond the specification. The agent iteratively develops a register-transfer-level (RTL) design and a behavior model as the design's verification reference. Our evaluator, \behavesim, checks both artifacts separately against a hidden golden behavior model using input stimuli generated by random sampling and solver-guided search. \name thus supports power, performance, and area (PPA) exploration across task-permitted latencies and microarchitectures. During training, the same evaluator provides verifiable reinforcement learning (RL) rewards from specification-behavior pairs without reference RTL. For self-improvement, the agent continually searches for high-level implementations relevant to its capability gaps, constructs and checks specification-behavior pairs, and trains on the expanded task pool. We release BEHAVE-Train and BEHAVE-Eval with \datasetsize human-reviewed specification-behavior pairs for realistic hardware workloads. Starting from 60 seed tasks and acquiring 100 new tasks, self-improvement raises Qwen3.8-27B's RTL pass@1 on BEHAVE-Eval from 55.0\% to 75.0\%, reaching performance comparable to RL using a 540-task pool.
\end{abstract}

\begin{center}
\small
Code \& datasets: \url{https://github.com/joel-wu/BEHAVE}\\
Model: \url{https://huggingface.co/joooelw/BEHAVE-27B}
\end{center}

\section{Introduction}
\label{sec:intro}

Design and verification are key stages of digital chip development~\citep{rabaey2003digital, bergeron2003writing}. Figure~\ref{fig:gap}(a) illustrates a common workflow: engineers develop a register-transfer-level (RTL) design and a reference model for verification, with a testbench driving both and comparing their outputs~\citep{uvm}. Output mismatches help engineers debug the design or reference model~\citep{wile2005functional}. Large language model (LLM) agents~\citep{yang2024sweagent, wei2025swerl} open a path to automating this process. We study how to develop an agent that jointly generates RTL designs and reference models during inference and learns both capabilities during training.

Developing such an agent requires reliable correctness feedback for both its RTL designs and reference models. Comparing RTL outputs cycle by cycle with a reference implementation provides feedback, but ties it to that implementation's timing~\citep{yu2026chipmate, ye2025chatmodel, mu2025faver, tan2026autoverifixplus, zhao2025prov}. When the specification permits different latencies, this can reject valid designs with preferable power, performance, and area (PPA) trade-offs, as illustrated in Figure~\ref{fig:gap}(b). The challenge is therefore to provide executable supervision that checks task-defined behavior while preserving the timing flexibility allowed by the specification.

\begin{figure}[!t]
\centering
\includegraphics[width=\textwidth]{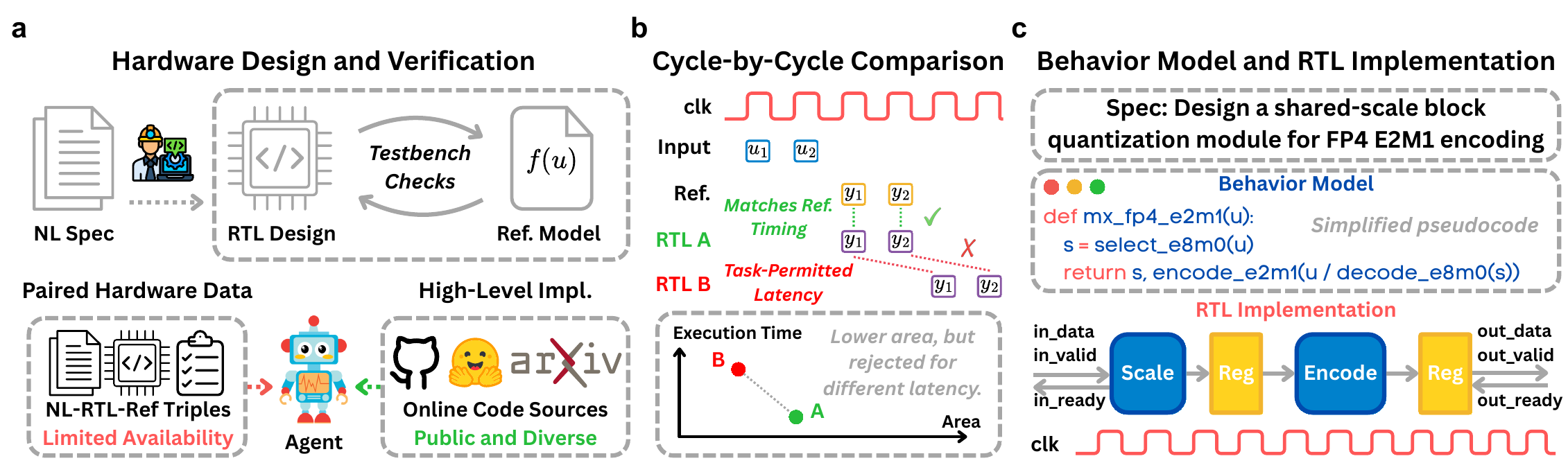}
\caption{\textbf{Motivation for functional behavior modeling.} (a) A common development workflow and the availability gap between paired hardware data and public high-level implementations. (b) Cycle-by-cycle comparison rejects a correct RTL design with task-permitted latency differences, despite its lower area. (c) A behavior model and pipelined RTL for MXFP4 block quantization. See Appendix~\ref{app:behave-contracts} for task details and code excerpts of the Python behavior model and pipelined RTL.}
\label{fig:gap}
\end{figure}

We address this challenge through \emph{functional behavior modeling}, using executable \emph{behavior models} to describe task functionality (Figure~\ref{fig:gap}(c)). We express them in \emph{Behavior IR}, a typed representation. One behavior model can serve as a reference for RTL implementations with task-permitted differences in latency and microarchitecture. For a specification \(S\), the agent develops an RTL design \(\hat R\) and a behavior model \(\hat B\) as the design's verification reference. Our evaluator, \behavesim, independently checks both artifacts against a hidden \emph{golden behavior model} \(B^\star\) following task-defined rules.

Building on this, we develop \name, an agentic framework for multi-turn joint design and verification. The agent iteratively refines \(\hat R\) and \(\hat B\) using its own testbench and development tools, and can use EDA feedback to explore area-time trade-offs. For independent evaluation, \behavesim first applies random test inputs to \(B^\star\) and the submitted artifact \(A \in \{\hat R,\hat B\}\), comparing their outputs under task-defined rules. Feedback from both executions then drives solver-guided search for additional inputs targeting unhit \emph{goals}, such as branch outcomes or assertion violations.

During training, \behavesim's artifact scores provide verifiable rewards for reinforcement learning (RL). The agent can thus learn to generate RTL designs and verification references from \((S,B^\star)\) pairs without reference RTL. While RTL training data are scarce~\citep{thakur2024verigen}, high-level implementations are widely available online~\citep{kocetkov2022stack} and serve as sources for \((S,B^\star)\) pairs (Figure~\ref{fig:gap}(a)). To acquire new training tasks, the agent searches for high-level implementations relevant to its capability gaps and adapts them into specification-behavior pairs. It checks these pairs and continues training with accepted ones. Rollout feedback guides the next acquisition round, supporting continued search-based self-improvement on realistic hardware workloads. Our contributions are:

\begin{itemize}
\item \textbf{Functional behavior modeling.} We develop \name, an agentic framework for multi-turn co-development of RTL designs and behavior models for verification. \behavesim uses golden behavior models to evaluate the agent's RTL designs and behavior models under task-defined rules. Functional modeling preserves task-permitted choices of latency and microarchitecture, enabling agents to explore area-time trade-offs with EDA feedback.

\item \textbf{Verifiable training and self-improvement.} We enable verifiable RL from specification-behavior pairs derived from a growing body of public high-level implementations, without requiring reference RTL. The agent iteratively acquires new tasks based on capability gaps and trains on the expanded task pool, supporting continued search-based self-improvement.

\item \textbf{Open-source datasets.} We release BEHAVE-Train and BEHAVE-Eval, comprising \datasetsize human-reviewed specification-behavior pairs covering realistic hardware workloads in AI/ML, numerics, signal processing, cryptography, control, compression, and data systems.
\end{itemize}

\section{Background}
\label{sec:background}

This section reviews RTL design and simulation-based verification in Subsection~\ref{sec:bg-design-verification}, then explains the roles of reference models and testbenches in Subsection~\ref{sec:bg-reference-testbenches}.

\subsection{Hardware Design and Verification}
\label{sec:bg-design-verification}

\paragraph{RTL Design.}
In digital circuit design, an RTL design (e.g., in Verilog or SystemVerilog) specifies combinational logic and sequential state updates, typically governed by clocks and resets. Unlike an untimed high-level implementation (e.g., in Python), RTL explicitly describes how computation and data transfers are organized across clock cycles~\citep{rabaey2003digital}.

\paragraph{RTL Verification Step 1: Stimulus Generation.}
Simulation-based verification first generates input stimuli to exercise the RTL design under test (DUT). Directed testing uses manually selected input sequences to target known scenarios, constrained-random testing samples a broader input space under user-defined constraints~\citep{naveh2007constraint}, and concolic testing combines concrete execution with symbolic constraint solving to steer inputs toward unexplored paths~\citep{sen2005cute}.

\paragraph{RTL Verification Step 2: Correctness Checking.}
The DUT's responses to these stimuli are checked for correctness. Outputs may be compared with expected values specified directly or computed by a reference model on the same inputs~\citep{bergeron2003writing}. Beyond output-value comparison, assertions can check protocol and temporal properties during execution~\citep{wile2005functional}.

\subsection{Reference Models and Testbenches}
\label{sec:bg-reference-testbenches}

\begin{figure}[!htbp]
\centering
\includegraphics[width=\textwidth]{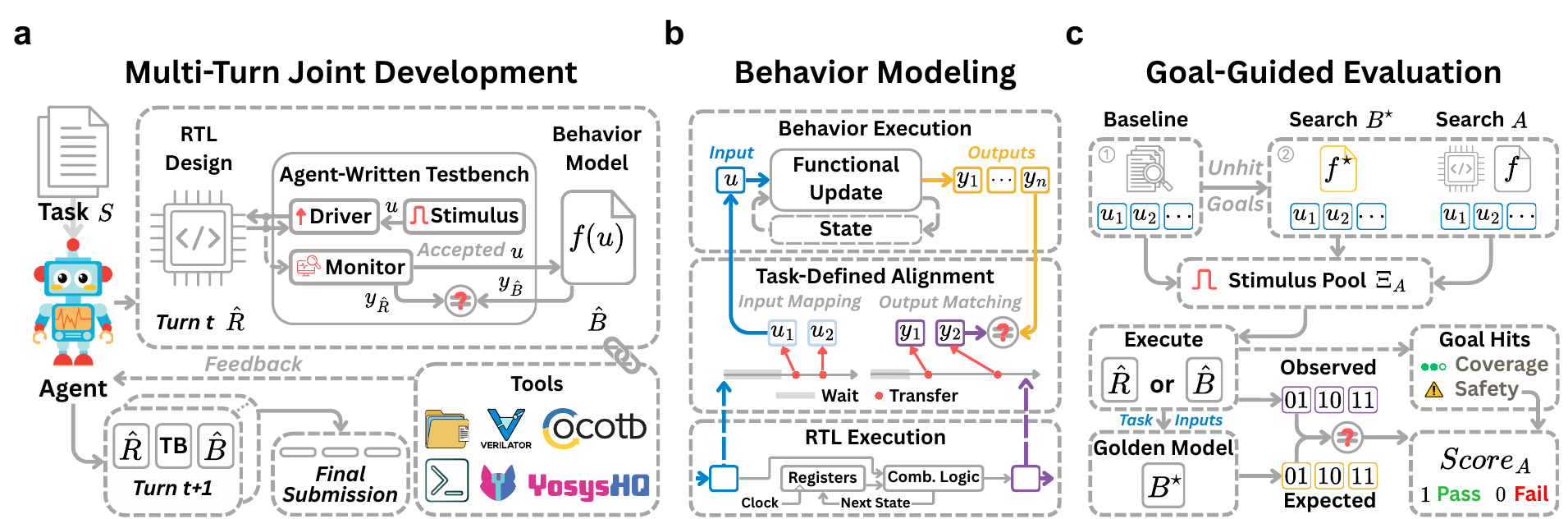}
\caption{\textbf{Joint development and evaluation.}
(a) The agent iteratively develops RTL \(\hat R\), a behavior model \(\hat B\),
and a testbench.
(b) Task-defined alignment matches RTL outputs to behavior-model predictions
from accepted inputs, allowing task-permitted latency differences.
(c) For each submitted artifact \(A\in\{\hat R,\hat B\}\), baseline tests and
goal-guided searches on \(B^\star\) and \(A\) provide stimuli for independent
evaluation against \(B^\star\), checking output agreement and safety-goal hits.}
\label{fig:method-overview}
\end{figure}

\paragraph{Reference Models and High-Level Implementations.}
A reference model is an executable description of the DUT's expected behavior. Reference models vary in language and range from cycle-accurate to untimed~\citep{bergeron2003writing}. For a high-level implementation to serve as a reference model, details such as bit widths, overflow behavior, and state initialization must be explicit and consistent with the task specification. In our implementation, reference models are expressed in Behavior IR, a typed executable representation with defined semantics and a restricted set of constructs.

\paragraph{UVM Testbenches.}
A testbench coordinates the stimulus generation and correctness checking described above, as in Figure~\ref{fig:method-overview}(a). The Universal Verification Methodology (UVM) provides a standardized framework for organizing these functions into components~\citep{uvm}. \emph{Sequences} generate stimulus transactions, \emph{drivers} translate them into DUT interface signals, and \emph{monitors} record the DUT's input and output transactions. \emph{Scoreboards} compare observed outputs with reference-model predictions for the recorded inputs, while \emph{coverage collectors} track exercised scenarios.

\section{BEHAVE}
\label{sec:method}

We first describe functional behavior modeling for multi-turn
joint design and verification (Section~\ref{sec:method-alignment}). We then present \behavesim's
goal-guided evaluation
(Section~\ref{sec:method-generation}). Finally, we describe
PPA exploration (Section~\ref{sec:method-ppa}) and verifiable RL with
search-based self-improvement (Section~\ref{sec:method-rl}).

\subsection{Functional Behavior Modeling}
\label{sec:method-coevolution}
\label{sec:method-alignment}

\paragraph{Task setup.}
Each task pairs a natural-language (NL) specification \(S\) with a golden behavior
model \(B^\star\). The agent receives \(S\), while \(B^\star\) is reserved for
evaluation and hidden during development. The agent develops an RTL design
\(\hat R\) and a behavior model
\(\hat B\) as the design's verification reference.

\paragraph{Multi-turn development.}
As shown in Figure~\ref{fig:method-overview}(a), the agent iteratively edits
and tests \(\hat R\) and \(\hat B\) with development tools, using its own
testbench to check their agreement and guide revisions. After the agent
submits its final \(\hat R\) and \(\hat B\), \behavesim evaluates each
independently against \(B^\star\).
\par

\paragraph{Behavior models.}
The golden model \(B^\star\) and the agent-written model
\(\hat B\) are represented in Behavior IR, which supports typed operations,
arrays, conditionals, loops, and persistent state. Invocations
take task-defined inputs and state, producing updated state and
outputs or an explicit failure.
Appendix~\ref{app:behave-contracts} provides an MXFP4 example.
Appendix~\ref{app:bs-ir} gives the syntax and semantics.

\paragraph{Task-defined alignment.}
For RTL-model comparison, \(\hat R\) is checked against \(\hat B\) during
development and \(B^\star\) during evaluation. As shown in
Figure~\ref{fig:method-overview}(b), the reference behavior model processes
the inputs accepted by the RTL and task-defined events such as reset.
Its predicted outputs are matched to RTL outputs, allowing latency differences
permitted by the task.
For model-model comparison, \(\hat B\) and \(B^\star\) receive the same
task-defined inputs, and their outputs are compared directly.

\paragraph{Ready/valid example.}
Consider a ready/valid task that requires one result per accepted input but
leaves response latency unspecified. \behavesim records the RTL's input and output data
at clock edges when the data are marked valid and the receiver is ready to
accept them. It feeds the recorded input values to the reference behavior model and
compares its predictions with the recorded RTL outputs.
Appendix~\ref{app:behave-support} summarizes supported protocols and limitations.

\subsection{Goal-Guided Evaluation}
\label{sec:method-generation}

\paragraph{Evaluation setup.}
\behavesim evaluates each submitted artifact \(A\in\{\hat R,\hat B\}\)
on a finite stimulus pool \(\Xi_A\), comparing its outputs with \(B^\star\).
Each stimulus \(\xi\in\Xi_A\) consists of finite input sequences and an
execution schedule specifying input presentation, clock events, and resets.

\paragraph{Goals.}
A goal is a predefined condition on artifact execution.
It is hit when a test stimulus drives the artifact to satisfy that condition.
Coverage goals mark execution cases, such as branch outcomes, while safety
goals mark failures, such as violations of existing assertions.
\behavesim automatically extracts goals from artifact structure and predefined
protocol checks, while tasks can explicitly declare additional goals.
Appendix~\ref{app:bs-goals} details goal types and hit conditions.

\paragraph{Stimulus generation.}
As shown in Figure~\ref{fig:method-overview}(c), \behavesim first runs
edge-case and seeded random tests. It then uses execution feedback
from both \(B^\star\) and \(A\) to search for inputs targeting unhit
goals. The search targets goals whose triggering conditions can be encoded
as solver constraints on inputs. Solver-generated tests are executed
to check outputs and update goal hits, guiding further search within a
shared time budget. Appendix~\ref{app:bs-generation} details the procedure.

\begin{figure}[!t]
\centering
\includegraphics[width=\textwidth]{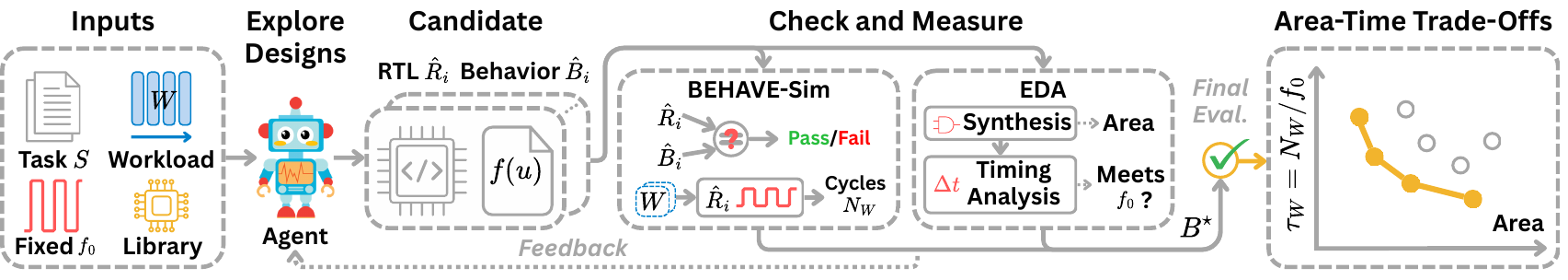}
\caption{\textbf{Timing-flexible PPA exploration.}
The agent freely explores area-time trade-offs through multi-turn tool interaction.
Orange points mark the empirical Pareto frontier.}
\label{fig:ppa-workflow}
\end{figure}

\paragraph{Artifact score.}
An artifact passes on \(\Xi_A\) when its outputs match \(B^\star\) and it
hits no safety goal:
\[
  \operatorname{score}_A(\Xi_A)
  =
  \mathbf 1\!\left[
    n^A_{\mathrm{mis}}(\Xi_A)=0
    \land
    \mathrm{Hit}_A(\Xi_A)\cap\Gamma_A^{\mathrm{safe}}=\emptyset
  \right],
\]
where \(n^A_{\mathrm{mis}}\) counts output mismatches, \(\mathrm{Hit}_A\)
collects observed goal hits, and \(\Gamma_A^{\mathrm{safe}}\) is \(A\)'s
safety-goal set. Goal coverage is reported separately and does not affect
this score (Appendix~\ref{app:bs-goals}).
\par

\begin{theorem}[Bounded goal analysis]
\label{thm:main-bounded-conclusions}
Under the assumptions in Appendix~\ref{app:behave-guarantees}, for either
submitted artifact, every goal reported as reached has an execution that
hits it, confirmed by replay. A goal reported as bounded-unreachable cannot
be hit by any allowed execution within the fixed bounds. No conclusion is
drawn for unresolved goals. This analysis does not change the artifact score.
\end{theorem}

\noindent\textit{Proof.}
Apply Theorem~\ref{thm:closure-sound} to the exact bounded encodings
in Theorems~\ref{thm:rtl-exact} and~\ref{thm:exact-b}
(Appendix~\ref{app:behave-guarantees}).\qed
\par

\subsection{Timing-Flexible PPA Exploration}
\label{sec:method-ppa}

\paragraph{Candidate search.}
Given \(S\), a fixed workload \(W\), frequency \(f_0\), and cell library,
the agent seeks low area and short workload execution time. It can vary
parallelism and other choices when revising
\(\hat R\) and \(\hat B\) with \behavesim and EDA feedback
(Figure~\ref{fig:ppa-workflow}). \behavesim checks their agreement and
measures workload cycles, while synthesis and timing analysis report cell
area and timing.

\paragraph{Area-time trade-offs.}
Workload execution time is \(\tau_W=N_W/f_0\), where \(N_W\) counts RTL
cycles to complete \(W\). Finally, RTL candidates are evaluated
against the hidden reference \(B^\star\). We report the empirical Pareto frontier
among passing candidates that complete \(W\) and meet timing at \(f_0\).
\par

\subsection{Verifiable RL and Self-Improvement}
\label{sec:method-rl}
\label{sec:method-self-improvement}

\begin{figure}[!t]
\centering
\includegraphics[width=\textwidth]{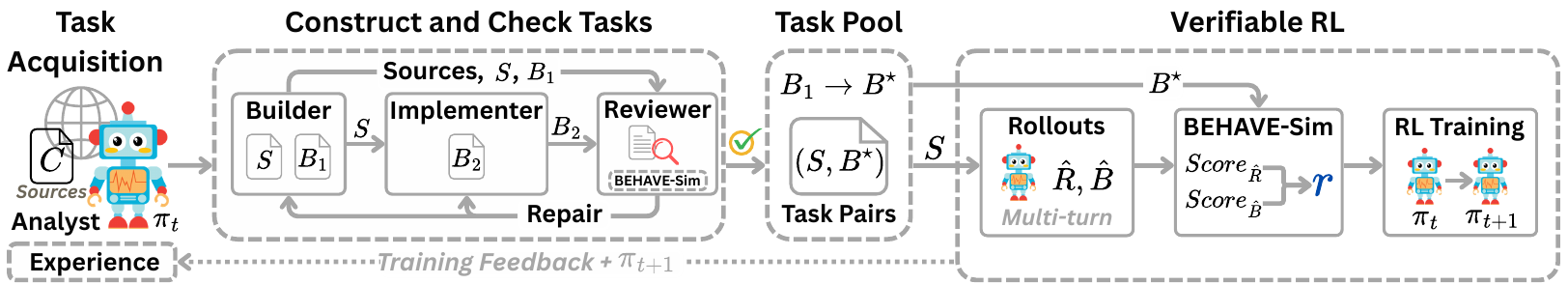}
\caption{\textbf{Task acquisition and verifiable RL.}
All roles use the current policy \(\pi_t\) in separate contexts.
Accepted task pairs support RL, and training feedback guides subsequent
task acquisition.}
\label{fig:training-workflow}
\end{figure}

\paragraph{Verifiable RL.}
Given \((S,B^\star)\) pairs, we train the agent with GRPO~\citep{shao2024deepseekmath} on
multi-turn joint-development rollouts. \behavesim evaluates
\(\hat R\) and \(\hat B\) against \(B^\star\), yielding the rollout reward
\[
  r=\frac{
    \operatorname{score}_{\hat B}(\Xi_{\hat B})
    +\operatorname{score}_{\hat R}(\Xi_{\hat R})
  }{2}
  \in\left\{0,\frac12,1\right\}.
\]

\paragraph{Feedback-guided task acquisition.}
To expand the training set, the current policy \(\pi_t\) acts as an Analyst
(Figure~\ref{fig:training-workflow}). It reviews completed training rollouts
and their feedback to identify capability gaps, then selects relevant
high-level implementations and sets objectives for new tasks. The Analyst
records useful findings as experience to guide subsequent acquisition rounds.

\paragraph{Task construction and continued training.}
Builder constructs \(S\) and \(B_1\) from the sources and objectives.
In a separate context, Implementer produces \(B_2\) from \(S\) and public
interface documentation, without the sources or \(B_1\).
Reviewer checks the task against its sources and compares the models with
\behavesim, requesting repairs when needed. Accepted \((S,B_1)\) pairs
enter the cumulative training pool with \(B_1\) as \(B^\star\).
Training continues from \(\pi_t\), yielding an updated policy
\(\pi_{t+1}\) that uses rollout feedback to guide subsequent acquisition.

\section{Experiments and Discussion}
\label{sec:experimentsrxiv}

\begin{table}[!t]
\centering
\caption{\textbf{Main results on hardware design benchmarks (\%).}
RTL and Behav. report RTL and Behavior pass@1 under \behavesim.
Cov. is line coverage
achieved by the submitted testbench on the submitted RTL, averaged over all
tasks. Missing submissions or compilation/execution failures score zero.
RL and SI denote fixed-dataset RL and self-improvement (see Section~\ref{sec:exp-self-improvement}).}
\label{tab:main_results}
\newcommand{\behaveMainResultRows}{%
Qwen3.5-4B & 34.6 & 28.8 & 6.9 & 12.0 & 20.0 & 0.0 & 4.3 & 4.3 & 0.9 & 0.0 & 6.7 & 0.0 \\
Qwen3.5-9B & 46.8 & 40.4 & 13.4 & 18.0 & 22.0 & 2.1 & 7.0 & 9.6 & 2.4 & 0.0 & 15.0 & 1.8 \\
Claude Haiku 4.5 & 92.9 & 89.1 & 62.7 & 88.0 & 84.0 & 87.2 & 72.2 & 54.8 & 83.8 & 36.7 & 70.0 & 48.1 \\
GPT-5.6 Luna & 94.2 & 94.9 & 84.1 & 98.0 & 98.0 & 85.3 & 93.0 & 93.9 & 91.6 & 68.3 & 73.3 & 89.7 \\
Claude Sonnet 5 & 92.3 & 92.3 & 58.1 & 84.0 & 86.0 & 75.7 & 64.3 & 60.9 & 66.9 & 68.3 & 65.0 & 60.7 \\
GPT-5.6 Terra & 99.4 & 99.4 & 71.0 & 98.0 & 98.0 & 85.6 & 95.7 & 97.4 & 95.9 & 85.0 & 90.0 & 95.7 \\
DeepSeek-V4.1-Flash & 96.8 & 96.2 & 79.9 & 98.0 & 100.0 & 94.0 & 93.0 & 90.4 & 95.5 & 88.3 & 83.3 & 90.7 \\
\midrule
Qwen3.8-27B & 97.4 & 95.5 & 77.3 & 90.0 & 90.0 & 87.9 & 81.7 & 82.6 & 87.0 & 55.0 & 56.7 & 59.1 \\
}

\newcommand{\behaveTrainingResultRow}{%
Qwen3.8-27B-RL (ours) & 98.1 & 96.8 & 66.7 & 96.0 & 92.0 & 85.7 & 91.3 & 87.8 & 92.0 & 73.3 & 71.7 & 80.6 \\
}

\newcommand{\behaveSynResultRow}{%
Qwen3.8-27B-SI (ours) & 98.7 & 96.2 & 80.6 & 96.0 & 92.0 & 85.8 & 93.0 & 92.2 & 91.5 & 75.0 & 75.0 & 77.5 \\
}

\fontsize{8}{9}\selectfont
\setlength{\tabcolsep}{0.7pt}
\begin{tabular*}{\textwidth}{@{\extracolsep{\fill}}l*{12}{c}@{}}
\toprule
& \multicolumn{3}{c}{VerilogEval-v2}
& \multicolumn{3}{c}{RTLLM-v2}
& \multicolumn{3}{c}{CVDP-cid003}
& \multicolumn{3}{c}{BEHAVE-Eval} \\
\cmidrule(lr){2-4}\cmidrule(lr){5-7}\cmidrule(lr){8-10}
\cmidrule(lr){11-13}
Model & RTL & Behav. & Cov. & RTL & Behav. & Cov.
& RTL & Behav. & Cov. & RTL & Behav. & Cov. \\
\midrule
\behaveMainResultRows
\behaveTrainingResultRow
\behaveSynResultRow
\bottomrule
\end{tabular*}
\end{table}

\begin{figure}[!t]
\centering
\includegraphics[width=\linewidth]{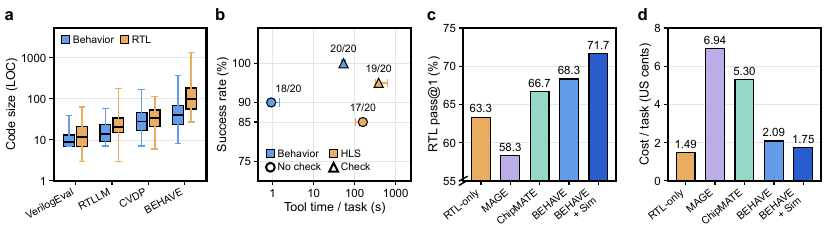}
\caption{\textbf{Benchmark construction and pipeline comparison.}
(a) RTL and Behavior code sizes across benchmarks. Whiskers show the full range.
(b) End-to-end success and median tool time (horizontal bars: interquartile range) for Behavior IR and HLS
on 20 tasks with GPT-5.6 Luna. Check adds development-time correctness
feedback. All settings retain frontend/synthesis feedback and final evaluation
against hidden golden Behavior models.
(c) RTL pass@1 across pipelines on BEHAVE-Eval.
(d) Mean API cost for the same pipelines.}
\label{fig:pipeline-comparison}
\end{figure}

We first introduce the datasets and evaluation settings
(Section~\ref{sec:exp-setup}), then examine how well models jointly perform
hardware design and verification using \name (Section~\ref{sec:exp-overall}).
We compare \name with other development pipelines and test whether
providing \behavesim as a stimulus-generation tool improves success and
reduces cost (Section~\ref{sec:exp-pipelines}). We then test whether agents
can explore area-time trade-offs under a fixed workload and frequency
(Section~\ref{sec:exp-ppa}), and whether verifiable RL and continued task
acquisition improve their capabilities (Section~\ref{sec:exp-self-improvement}).

\subsection{Benchmarks and Evaluation Settings}
\label{sec:exp-setup}

\paragraph{Benchmarks.}
We construct \datasetsize human-reviewed specification-behavior pairs from
high-level implementations across seven application domains, split into
BEHAVE-Train (\trainingdatasetsize tasks) and BEHAVE-Eval (\evaluationdatasetsize tasks).
We extend VerilogEval-v2~\citep{pinckney2025verilogeval},
RTLLM-v2~\citep{liu2024openllmrtl}, and CVDP-cid003~\citep{pinckney2025cvdp}
with golden Behavior models and resolve specification inconsistencies.
Behavior models typically use less code than RTL by omitting hardware
implementation details. BEHAVE has the largest median RTL code size
(Figure~\ref{fig:pipeline-comparison}(a)).
Appendix~\ref{app:benchmark-review} details sources,
construction, review, splits, and licensing.

\paragraph{Behavior IR vs.\ HLS.}
High-level synthesis (HLS) converts designs written in high-level
languages, such as C++, into RTL.
On 20 tasks, the agent translates the same high-level implementations
into Behavior IR or synthesizable C++.
\behavesim and Vitis HLS provide compilation and synthesis feedback, respectively.
With or without \behavesim correctness feedback,
Behavior IR achieves higher success rates and lower median tool time
(Figure~\ref{fig:pipeline-comparison}(b)).
Appendix~\ref{app:exp-hls} provides details.
Our framework is not tied to Behavior IR and can use other executable
reference representations.

\paragraph{Evaluation settings.}
Models share a development environment with a 128K-token context
window and at most 60 turns per rollout. \behavesim evaluates final RTL and
Behavior submissions. Appendices~\ref{app:exp-environment}
and~\ref{app:exp-setup} detail the environment, model configurations and evaluation budgets.

\subsection{Hardware Design and Verification Performance}
\label{sec:exp-overall}

\paragraph{Model capabilities.}
Table~\ref{tab:main_results} shows that frontier foundation models already
exhibit strong hardware design and verification capabilities. Given a
tool-equipped development environment, they can autonomously write, test,
and revise RTL designs and their Behavior references, using tool feedback to
decide their next steps. Smaller models still
struggle with both artifacts, and testbench coverage remains uneven even
among models with high RTL and Behavior pass@1.

\paragraph{Benchmark saturation.}
In our interactive setting, strong models approach saturation on these extended
benchmarks, while BEHAVE-Eval remains more challenging. GPT-5.6 Terra
achieves 95.7-99.4\% RTL pass@1 across the three external benchmarks,
compared with 85.0\% on BEHAVE-Eval.
As algorithms and workloads evolve, we can construct harder
tasks from high-level implementations.

\subsection{Comparison of Development Workflows}
\label{sec:exp-pipelines}

\paragraph{Comparison setup.}
Using GPT-5.6 Luna, we compare RTL-only development,
ChipMATE~\citep{yu2026chipmate}, MAGE~\citep{zhao2025mage}, \name,
and \name with \behavesim on BEHAVE-Eval.
In the last setting, \behavesim generates stimuli and compares the agent's
RTL and Behavior model during development, replacing the agent-written testbench.
All final RTL submissions are evaluated against golden Behavior models.
Appendix~\ref{app:exp-pipelines} details the settings and baseline adaptations.

\paragraph{Performance and cost.}
\name achieves higher RTL pass@1 than RTL-only development, suggesting that
constructing an executable reference can help the agent develop RTL
(Figure~\ref{fig:pipeline-comparison}(c)).
With a single tool-using agent, \name achieves comparable RTL
pass@1 to ChipMATE and MAGE at lower API cost
(Figure~\ref{fig:pipeline-comparison}(d)).
Adding \behavesim
improves success while reducing cost, supporting the value of automated
testing within joint development.

\paragraph{Why functional behavior modeling.}
ChipMATE generates Python reference models for RTL verification.
Reproducing the RTL's cycle-by-cycle execution in Python, however, can
retain its control logic and intermediate state. In an unsigned-division
example, ChipMATE's Python model maintains control state across clock-level
calls, whereas the Behavior reference completes the computation in one
call (see Appendix~\ref{app:exp-additional}). Expressing the same computation at different
abstraction levels allows joint development to separate functional intent
from architectural choices.

\subsection{Timing-Flexible PPA Exploration}
\label{sec:exp-ppa}

\begin{figure}[!htbp]
\centering
\includegraphics[width=\linewidth]{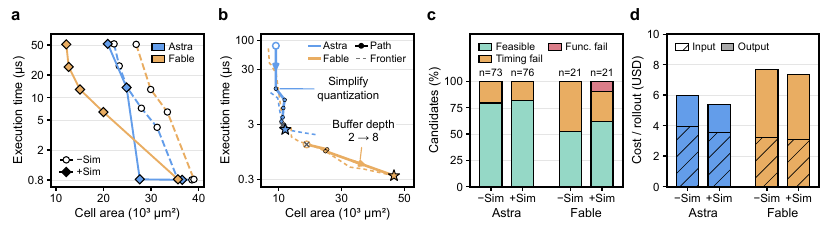}
\caption{\textbf{Timing-flexible PPA exploration.}
(a) MXFP4 frontiers with and without \behavesim (SkyWater 130\,nm, 80\,MHz).
(b) NVFP4 frontiers (dashed) and optimization paths (solid)
(FreePDK45-based Nangate 45\,nm library, 400\,MHz).
Stars mark endpoints and crosses timing failures.
(c) MXFP4 candidate outcomes ($n$: collected candidates).
(d) MXFP4 mean estimated input/output API cost per search with cross-turn cache reuse.
Execution time is simulated cycles / target frequency for
64 continuously offered blocks of 32 (MXFP4) or 16 (NVFP4) BF16 values.}
\label{fig:ppa-exploration}
\end{figure}

\paragraph{Experimental setup.}
We compare \name with and without \behavesim for MXFP4 area-time exploration
using GPT-6 Astra and Claude Fable 5.1. Both use EDA feedback under the same
workload, clock frequency, and SkyWater 130\,nm cell
library~\citep{skywater2020pdk}.
We also explore NVFP4 using the
FreePDK45-based Nangate 45\,nm standard-cell
library~\citep{nangate2008library,stine2007freepdk}. Appendix~\ref{app:exp-ppa}
provides details.

\paragraph{Area-time trade-offs.}
Both models produce correct designs with different area-time trade-offs
(Figure~\ref{fig:ppa-exploration}(a)).
Figure~\ref{fig:ppa-exploration}(b) traces NVFP4 revisions to
selected frontier points. Astra simplifies quantization using comparisons
in place of division. Fable increases buffering from two to eight blocks,
reducing execution time at the cost of area. Within each task, the same
golden Behavior reference validates designs with different architectures
and execution times.

\paragraph{Effect of simulation feedback.}
\behavesim yields lower-area designs at comparable execution times and a
higher fraction of feasible candidates
(Figure~\ref{fig:ppa-exploration}(a,c)).
Estimated mean API cost per search is modestly lower
with \behavesim (Figure~\ref{fig:ppa-exploration}(d)). These results suggest that reusable
stimulus-generation and comparison tools can be more effective than asking
agents to write testbenches.

\subsection{Verifiable RL and Self-Improvement}
\label{sec:exp-self-improvement}

\paragraph{Experimental setup.}
We train Qwen3.8-27B for 40 updates, sampling from the \trainingdatasetsize-task BEHAVE-Train pool.
Self-improvement starts from 60 seed tasks and adds new tasks
through \selfimprovementrounds acquisition rounds, with 18 training updates in total.
The seed-only baseline trains only on the same 60 seed tasks for 18 updates.
Appendix~\ref{app:exp-training}
details the training and acquisition settings.

\begin{figure}[!t]
\centering
\includegraphics[width=\linewidth]{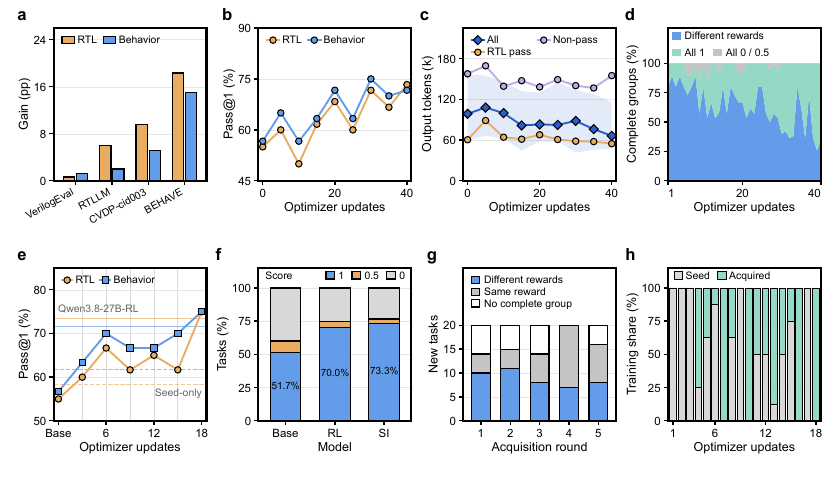}
\caption{\textbf{RL and self-improvement with Qwen3.8-27B.}
(a) RL gains across benchmarks.
(b,c) BEHAVE-Eval pass@1 and median rollout output length.
(d) Reward-group composition during RL.
(e) Self-improvement performance on BEHAVE-Eval.
Horizontal lines show final fixed-dataset RL (solid) and seed-only (dashed) results.
(f) Joint-score distributions on BEHAVE-Eval: 1, 0.5, and 0 indicate both,
one, or neither artifact passing.
(g) Reward variation across eight candidate solutions per newly acquired task.
(h) Seed/acquired task shares in each training update.}
\label{fig:training-self-improvement}
\end{figure}

\paragraph{Learning performance.}
Verifiable RL improves both RTL design and reference-model construction,
with the largest gains on BEHAVE-Eval. RTL pass@1 rises from 55.0\% to
73.3\%, and Behavior pass@1 from 56.7\% to 71.7\%
(Figure~\ref{fig:training-self-improvement}(a-b)). These gains are accompanied
by shorter generations: median total generation length per rollout falls
from 98.6k to 66.2k tokens (Figure~\ref{fig:training-self-improvement}(c)).

\paragraph{Self-improvement.}
The final self-improved model reaches 75.0\% for both
RTL and Behavior pass@1 on BEHAVE-Eval
(Figure~\ref{fig:training-self-improvement}(e)).
Both artifacts pass on 44 of
60 tasks, versus 31 before training and 42 after fixed-dataset RL
(Figure~\ref{fig:training-self-improvement}(f)).
The seed-only baseline reaches 58.3\% RTL and 61.7\% Behavior pass@1.
Starting from 60 seed tasks,
self-improvement reaches the performance of RL using a 540-task pool.
Appendix~\ref{app:exp-self-improvement-case} presents a recorded
self-improvement case.

\paragraph{Learning signal.}
Later fixed-dataset updates contain more sampled all-1 groups,
leaving no reward differences within those groups
(Figure~\ref{fig:training-self-improvement}(d)). Self-improvement expands the pool
with tasks constructed from retrieved software. Newly acquired tasks provide
nonconstant rewards across acquisition rounds
(Figure~\ref{fig:training-self-improvement}(g)) and contribute to subsequent
updates (Figure~\ref{fig:training-self-improvement}(h)).

\section{Related Work}
\label{sec:related}

\paragraph{LLMs for EDA.}
LLMs have been applied to a wide range of digital design and verification tasks,
including RTL and testbench generation, debugging, and design
optimization~\citep{chang2023chipgpt, thakur2023benchmarking, yu2025spec2rtl, qiu2025correctbench, qiu2024autobench, qiu2025confibench, zhang2025llm4dv, xin2026gogotb, yan2025assertllm, mali2024chiraag, orenesvera2023autosva2, bhandari2024fsmtb, ma2024verilogreader, tsai2024rtlfixer, yao2024hdldebugger, yubeaton2026agentic, yu2026agentic, islam2026verigraphi, chang2026specloop, wu2026llmfsm, li2026gensoc, zhang2026refevo, ye2025uvm2, ye2026uvmarvel, zhao2026ctosynth, wang2025symrtlo, tasnia2025veriopt, wang2025mcp4eda, fang2026drrtl, ping2026poet, wang2026veriagent, thorat2025verippa, cuyckens2026ares, chen2026heurigym, hu2025uvllm, tan2026autoverifix, cui2026chia, hemadri2025veriloc}.
A parallel line develops training datasets and post-trains open-source LLMs for
these tasks~\citep{liu2023chipnemo, liu2024rtlcoder, thakur2024verigen, liu2025craftrtl, wang2025codestructrl, yubeaton2025verithoughts, zhu2025codevr1, zhang2026llm4cov, cui2024origen, zhao2024codev, akyash2025rtlpp, calzada2025verilogdb, chen2025chipseek, wang2025verireason, wei2025vericoder, zhang2025salv, deng2025scalertl, zhang2026rtlseek, deng2026acertl, chen2026siliconmind, shi2025earl, wang2025tbfeedback, fu2026chatsva, wu2026codevsva}.
Some methods also use high-level representations to support RTL
training. BetterV translates RTL into cycle-accurate C models for paired
instruction tuning~\citep{pei2024betterv}. AutoVCoder uses Python execution to
filter synthetic combinational RTL examples before supervised
fine-tuning~\citep{gao2024autovcoder}. \name instead uses executable behavior
models to provide correctness rewards during RL.

\paragraph{Verification-Guided Hardware Agents.} MAGE generates Verilog testbenches for RTL debugging, ChatModel generates SystemC reference models, and PRO-V-R1 trains agents to generate Python reference models and testbenches using golden RTL supervision~\citep{zhao2025mage, ye2025chatmodel, zhao2025prov}. FAVer and AutoVeriFix+ use Python reference models to guide RTL generation and repair~\citep{mu2025faver, tan2026autoverifixplus}. ChipMATE trains RTL and Python agents for cross-verification, with the Python agent supervised by RTL-derived, cycle-accurate reference models~\citep{yu2026chipmate}. Cycle-aligned checking couples verification to reference timing, limiting reference reuse across RTL designs with task-permitted latency differences (Table~\ref{tab:related_comparison}).

\paragraph{High-Level Representations and References.} CorrectHDL checks generated RTL against HLS-generated references derived from C/C++~\citep{xu2025correcthdl}, while FormalRTL uses C reference models for equivalence checking~\citep{li2026formalrtl}. HINT generates an executable IR with a frozen transaction and timing contract for RTL verification~\citep{cheng2026hint}. \citet{yin2026hierarchical} use hierarchical IRs to describe module connectivity, operations, and I/O behavior. CUDA Agent and DRTriton compare generated kernels with PyTorch references after kernel completion~\citep{dai2026cudaagent, guo2026drtriton}. \name matches RTL outputs to behavior-model predictions according to task-defined rules, preserving specification-permitted timing flexibility.

\begin{table}[!t]
\centering
\small
\setlength{\tabcolsep}{2pt}
\caption{\textbf{Comparison of related methods.} \emph{Task} lists evaluated capabilities. RM and TB denote reference model and testbench. G and P mark agent-generated verification artifacts and supplied references. \emph{Timing Flex.}\ indicates whether RTL latency or kernel runtime may vary. \emph{Self-improvement} indicates whether the updated agent autonomously acquires new tasks for further training.}
\label{tab:related_comparison}
\resizebox{\textwidth}{!}{%
\begin{tabular}{@{}lccccccc@{}}
\toprule
& \multicolumn{5}{c}{\textbf{Inference}} & \multicolumn{2}{c}{\textbf{Training}} \\
\cmidrule(lr){2-6} \cmidrule(lr){7-8}
\textbf{Method}
& \textbf{Task}
& \textbf{Design}
& \textbf{Verification}
& \textbf{Alignment}
& \makecell{\textbf{Timing}\\\textbf{Flex.}}
& \makecell{\textbf{Training}\\\textbf{Source}}
& \makecell{\textbf{Self-}\\\textbf{Improv.}} \\
\midrule
MAGE~\citep{zhao2025mage} & Design & RTL & Verilog TB\textsuperscript{G/P} & Cycle & No & - & No \\
ChatModel~\citep{ye2025chatmodel} & Verification & - & SystemC RM\textsuperscript{G} & Cycle & No & - & No \\
PRO-V-R1~\citep{zhao2025prov} & Verification & - & Python RM\textsuperscript{G} & Cycle & No & Spec-RTL & No \\
FAVer~\citep{mu2025faver} & Design & RTL & Python RM\textsuperscript{G} & Cycle & No & Spec-RTL & No \\
AutoVeriFix+~\citep{tan2026autoverifixplus} & Both & RTL & Python RM\textsuperscript{G} & Cycle & No & - & No \\
ChipMATE~\citep{yu2026chipmate} & Both & RTL & Python RM\textsuperscript{G} & Cycle & No & Spec-RTL & No \\
CorrectHDL~\citep{xu2025correcthdl} & Design & RTL & HLS RTL\textsuperscript{P} & Task-specific & Yes & - & No \\
FormalRTL~\citep{li2026formalrtl} & Design & RTL & C RM\textsuperscript{P} & Fixed-Latency & No & - & No \\
HINT~\citep{cheng2026hint} & Design & RTL & HINT IR\textsuperscript{G} & Task-specific & No & - & No \\
CUDA Agent~\citep{dai2026cudaagent} & Design & CUDA & PyTorch RM\textsuperscript{P} & Completion & Yes & PyTorch Ops & No \\
DRTriton~\citep{guo2026drtriton} & Design & Triton & PyTorch RM\textsuperscript{P} & Completion & Yes & PyTorch Ops & No \\
\textbf{\name (ours)} & \textbf{Both} & \textbf{RTL} & \textbf{Behavior RM\textsuperscript{G}} & \textbf{Task-specific} & \textbf{Yes} & \textbf{Spec-Behavior} & \textbf{Yes} \\
\bottomrule
\end{tabular}%
}
\end{table}

\paragraph{Self-Improving Agents.}
Agents refine memory
and skills~\citep{zhang2026accelopt,yan2026openskill}, evolve agent code
and harnesses~\citep{zhang2025dgm}, or update model weights through
training~\citep{zhang2026espl,wang2025cure}. Some training methods let
models propose new problems and learn from solving
them~\citep{zhao2025absolute,huang2026rzero,wang2025socratic,yang2026ttcs,lu2026searchselfplay,dong2026agentworld}.
\name acquires new tasks with executable references
from public high-level implementations. The updated agent guides acquisition
using training feedback, and checked specification-behavior pairs support
further RL.

\paragraph{Benchmarks for RTL Design and Verification.}
RTLLM~\citep{lu2024rtllm,liu2024openllmrtl} and
VerilogEval~\citep{liu2023verilogeval,pinckney2025verilogeval} evaluate small
module generation from natural-language specs, and
CVDP~\citep{pinckney2025cvdp} extends the format to other RTL tasks such as
debugging and comprehension.
ArchXBench, NotSoTiny, and RealBench target larger or more
realistic RTL designs~\citep{purini2025archxbench,ghorab2025notsotiny,jin2025realbench}.
RTL-Repo evaluates repository-level code completion~\citep{allam2024rtlrepo},
while CktEvo focuses on repository-level design evolution~\citep{shi2026cktevo}.
ChipBench covers RTL generation, debugging, and reference-model
generation~\citep{yu2026chipbench}, while RTL-OPT targets RTL
optimization~\citep{lu2026rtlopt}.
IC-RTL and Pluto support PPA evaluation across different latencies, but rely
on a dedicated testbench for each
task~\citep{hsin2026evolve,abdelatty2025pluto}. For evaluation, \behavesim
instantiates standardized testbenches from our dataset's golden behavior
models and interface declarations, without requiring separately authored
testbench code for each task.

\section{Conclusion}
\label{sec:conclusion}

We presented \name, a framework for joint hardware design and verification
through functional behavior modeling.
\behavesim independently evaluates RTL designs and behavior
models against golden references while preserving task-permitted timing
flexibility, providing verifiable RL rewards without reference RTL.
Experiments demonstrate multi-turn joint development, area-time exploration
across workloads and cell libraries, and improvements in RTL and
reference-model generation through training. Starting from a small
human-reviewed seed set, the agent constructs and checks new specification-behavior
pairs for continued training, using rollout feedback to guide task acquisition.
This self-improvement process reaches performance comparable to full-dataset RL
on BEHAVE-Eval. By turning high-level implementations into executable
supervision, \name offers a path to continued self-improvement as algorithms
and workloads evolve.

\bibliographystyle{unsrtnat}
\bibliography{references}

\appendix


\clearpage
\section*{Appendix Contents}

\begingroup
\hypersetup{linkcolor=black,pdfborder={0 0 0}}
\setlist[itemize]{label={},leftmargin=1.6em,itemsep=2pt,parsep=0pt,
  topsep=4pt,partopsep=0pt}

\noindent\hyperref[app:statements]{\textbf{Appendix~\ref*{app:statements}: AI Use, Ethics, and Reproducibility}}
\dotfill\pageref{app:statements}

\vspace{0.5\baselineskip}
\noindent\hyperref[app:behave-system]{\textbf{Appendix~\ref*{app:behave-system}: \name: System Design}}
\dotfill\pageref{app:behave-system}
\begin{itemize}
  \item \hyperref[app:behave-workflow]{\ref*{app:behave-workflow} End-to-End Workflow}
    \dotfill\pageref{app:behave-workflow}
  \item \hyperref[app:behave-contracts]{\ref*{app:behave-contracts} Task and Artifact Example}
    \dotfill\pageref{app:behave-contracts}
  \item \hyperref[app:behave-architecture]{\ref*{app:behave-architecture} Execution Architecture}
    \dotfill\pageref{app:behave-architecture}
  \item \hyperref[app:behave-support]{\ref*{app:behave-support} Supported Features and Limitations}
    \dotfill\pageref{app:behave-support}
  \item \hyperref[app:bs-setting]{\ref*{app:bs-setting} Transaction Mapping and Output Matching}
    \dotfill\pageref{app:bs-setting}
  \item \hyperref[app:bs-goals]{\ref*{app:bs-goals} Goals, Coverage and Scoring}
    \dotfill\pageref{app:bs-goals}
  \item \hyperref[app:bs-generation]{\ref*{app:bs-generation} Goal-Guided Exploration and Evaluation}
    \dotfill\pageref{app:bs-generation}
\end{itemize}

\vspace{0.5\baselineskip}
\noindent\hyperref[app:behave-guarantees]{\textbf{Appendix~\ref*{app:behave-guarantees}: \behavesim Formal Semantics and Bounded Guarantees}}
\dotfill\pageref{app:behave-guarantees}
\begin{itemize}
  \item \hyperref[app:bs-generic]{\ref*{app:bs-generic} Trace Domains and Goal Reachability}
    \dotfill\pageref{app:bs-generic}
  \item \hyperref[app:bs-representation]{\ref*{app:bs-representation} Finite Representations and Exact Encodings}
    \dotfill\pageref{app:bs-representation}
  \item \hyperref[app:bs-closure]{\ref*{app:bs-closure} Bounded Reachability Analysis}
    \dotfill\pageref{app:bs-closure}
  \item \hyperref[app:bs-signal]{\ref*{app:bs-signal} RTL Event-Slot Semantics and Protocol Adapters}
    \dotfill\pageref{app:bs-signal}
  \item \hyperref[app:bs-rtl-domain]{\ref*{app:bs-rtl-domain} RTL Closed-Loop Semantics and Trace Domain}
    \dotfill\pageref{app:bs-rtl-domain}
  \item \hyperref[app:bs-rtl]{\ref*{app:bs-rtl} RTL Finite Representation, Exact Encoding and Audits}
    \dotfill\pageref{app:bs-rtl}
  \item \hyperref[app:bs-ir]{\ref*{app:bs-ir} Behavior IR and Procedure Semantics}
    \dotfill\pageref{app:bs-ir}
  \item \hyperref[app:bs-encoding]{\ref*{app:bs-encoding} Behavior Finite Representation and Exact Encoding}
    \dotfill\pageref{app:bs-encoding}
  \item \hyperref[app:bs-assumptions]{\ref*{app:bs-assumptions} Behavior Frontend Fidelity and Audits}
    \dotfill\pageref{app:bs-assumptions}
  \item \hyperref[app:bs-artifact-reporting]{\ref*{app:bs-artifact-reporting} Artifact Instances and Reported Bounded Conclusions}
    \dotfill\pageref{app:bs-artifact-reporting}
  \item \hyperref[app:bs-tcb]{\ref*{app:bs-tcb} Trusted Boundary, Claim Scope and Conclusion}
    \dotfill\pageref{app:bs-tcb}
\end{itemize}

\vspace{0.5\baselineskip}
\noindent\hyperref[app:exp-details]{\textbf{Appendix~\ref*{app:exp-details}: Implementation, Experimental Details and Additional Results}}
\dotfill\pageref{app:exp-details}
\begin{itemize}
  \item \hyperref[app:exp-environment]{\ref*{app:exp-environment} Agent Environment and Tools}
    \dotfill\pageref{app:exp-environment}
  \item \hyperref[app:exp-setup]{\ref*{app:exp-setup} Inference and Evaluation Settings}
    \dotfill\pageref{app:exp-setup}
  \item \hyperref[app:exp-search-coverage]{\ref*{app:exp-search-coverage} Effect of Goal-Guided Search on RTL Coverage}
    \dotfill\pageref{app:exp-search-coverage}
  \item \hyperref[app:exp-hls]{\ref*{app:exp-hls} Behavior IR and HLS Comparison Settings}
    \dotfill\pageref{app:exp-hls}
  \item \hyperref[app:exp-pipelines]{\ref*{app:exp-pipelines} Pipeline Comparison and PPA Exploration Settings}
    \dotfill\pageref{app:exp-pipelines}
  \item \hyperref[app:exp-training]{\ref*{app:exp-training} Training and Self-Improvement Settings}
    \dotfill\pageref{app:exp-training}
  \item \hyperref[app:exp-resources]{\ref*{app:exp-resources} Computational Resources and Cost}
    \dotfill\pageref{app:exp-resources}
  \item \hyperref[app:exp-additional]{\ref*{app:exp-additional} Pipeline Comparison Example}
    \dotfill\pageref{app:exp-additional}
  \item \hyperref[app:exp-self-improvement-case]{\ref*{app:exp-self-improvement-case} Additional Results for RL Training and Self-Improvement}
    \dotfill\pageref{app:exp-self-improvement-case}
\end{itemize}

\vspace{0.5\baselineskip}
\noindent\hyperref[app:benchmark-review]{\textbf{Appendix~\ref*{app:benchmark-review}: Dataset Construction and Validation}}
\dotfill\pageref{app:benchmark-review}
\begin{itemize}
  \item \hyperref[app:benchmark-sources]{\ref*{app:benchmark-sources} Dataset Overview and Sources}
    \dotfill\pageref{app:benchmark-sources}
  \item \hyperref[app:benchmark-authored]{\ref*{app:benchmark-authored} Source-Based Task Construction}
    \dotfill\pageref{app:benchmark-authored}
  \item \hyperref[app:benchmark-construction]{\ref*{app:benchmark-construction} External Benchmark Adaptation}
    \dotfill\pageref{app:benchmark-construction}
  \item \hyperref[app:benchmark-checks]{\ref*{app:benchmark-checks} Human Review}
    \dotfill\pageref{app:benchmark-checks}
  \item \hyperref[app:benchmark-splits]{\ref*{app:benchmark-splits} Data Splits and Release}
    \dotfill\pageref{app:benchmark-splits}
\end{itemize}

\endgroup
\hypersetup{linkcolor=magenta}
\clearpage


\section{AI Use, Ethics, and Reproducibility}
\label{app:statements}

\paragraph{AI use statement.}
\label{app:ai-use}
We used generative AI tools to polish the manuscript and improve
the presentation of the mathematical claims in Appendix~\ref{app:behave-guarantees}.
These tools also assisted literature retrieval and review for
Section~\ref{sec:related}.
LLMs also assisted the retrieval of high-level source implementations
and the construction of specification-behavior pairs, including drafting
specifications, generating behavior models and auxiliary RTL, and revising
these artifacts using execution feedback. BEHAVE benchmark construction and
external-benchmark adaptation are detailed in Appendices~\ref{app:benchmark-authored}
and~\ref{app:benchmark-construction}, respectively. Human review of these benchmark
tasks is described in Appendix~\ref{app:benchmark-checks}. The autonomous acquisition
of additional training tasks during self-improvement is described separately
in Section~\ref{sec:method-self-improvement} and Appendix~\ref{app:exp-self-improvement}.
The authors reviewed the AI-assisted contributions and remain
responsible for the final text, mathematical claims, code, datasets,
experiments, and reported results.

\paragraph{Ethics statement.}
\label{app:ethics}
Our benchmark is derived from publicly available technical material and does
not intentionally collect personal or sensitive data. We record source
provenance and licensing, and will release only human-reviewed benchmark
items with established redistribution permissions. Hardware-generation
systems are dual-use, and generated designs may contain functional, safety or
security defects. The checks in \name reduce but do not eliminate these risks.
Generated designs therefore require independent validation before deployment.

\paragraph{Reproducibility statement.}
\label{app:reproducibility}
We will publicly release the \name source code, benchmark, golden Behavior
models, auxiliary reference RTL, prompts, training and evaluation configurations,
model checkpoints and per-run records under open licenses. Reference RTL is
not required for reward computation or autonomous task acquisition.
The release records tool
and dependency versions, model identifiers, random seeds, run policies and
artifact digests needed to reproduce each result. Source-specific provenance,
licensing and attribution are documented in
Appendix~\ref{app:benchmark-review}.


\section{\name: System Design}
\label{app:behave-system}
\label{app:behave-sim}

\paragraph{Purpose.}
We describe how \name supports joint RTL and behavior development
and independently evaluates artifacts against the golden behavior
model under task-defined protocol requirements.

\paragraph{Overview.}
Subsections~\ref{app:behave-workflow} through~\ref{app:behave-support}
describe the workflow, an artifact example, execution architecture and supported
features. Subsections~\ref{app:bs-setting} through~\ref{app:bs-generation}
define transaction mapping, goals, scoring and evaluation,
from baseline execution to goal-guided search and continuous workloads.
\par

\subsection{End-to-End Workflow}
\label{app:behave-workflow}
\paragraph{Development and evaluation.}
Figure~\ref{fig:end-to-end-workflow} connects joint development to independent
evaluation. Given \(S\), the agent develops \(\hat R\), \(\hat B\), and a testbench through
tool interaction. After submission, \behavesim uses the task interface and
hidden \(B^\star\) to evaluate \(\hat R\) and \(\hat B\) separately.
Stimulus generation and comparison follow Figure~\ref{fig:method-overview}(c)
and Appendix~\ref{app:bs-generation}, which describes baseline-first search
and continuous workloads. We separately report the testbench's coverage
of the agent's own RTL.

\begin{figure}[!t]
\centering
\resizebox{\textwidth}{!}{
\begingroup
\definecolor{wfEnv}{HTML}{5A7082}%
\definecolor{wfRtl}{HTML}{38434C}%
\definecolor{wfObs}{HTML}{7A858E}%
\begin{tikzpicture}[
  x=1cm, y=1cm, text=black!82,
  font=\normalfont\footnotesize,
  heading/.style={font=\normalfont\footnotesize\bfseries},
  small/.style={font=\normalfont\footnotesize},
  block/.style={draw=wfEnv!85, fill=white, line width=0.6pt},
  group/.style={draw=wfEnv!55, fill=wfEnv!3, line width=0.6pt},
  flow/.style={-{Latex[length=1.7mm]}, draw=black!65, line width=0.8pt},
  feedback/.style={-{Latex[length=1.6mm]}, draw=wfEnv,
    densely dashed, line width=0.7pt},
  wirelabel/.style={small, fill=white, inner sep=1.6pt, align=center}
]
\path[use as bounding box] (0,-0.40) rectangle (16,10.95);
\path[draw=black!35, fill=white, line width=0.6pt]
  (0.02,-0.38) rectangle (15.98,10.93);
\node[heading, anchor=west] at (0.70,10.67) {End-to-end workflow};
\draw[draw=black!18, line width=0.45pt]
  (0.18,10.45) -- (15.82,10.45);

\path[block] (0.75,8.90) rectangle (4.55,10.15);
\node[heading] at (2.65,9.91) {Task \((S,B^\star)\)};
\node[small] at (2.65,9.52) {Specification and interface};
\node[small] at (2.65,9.16) {Golden behavior \(B^\star\)};
\coordinate (wfTaskSpec) at (4.55,9.39);
\coordinate (wfTaskEval) at (2.65,8.90);
\coordinate (wfNewTaskIn) at (0.75,9.25);

\path[group] (5.10,8.95) rectangle (14.55,10.15);
\node[heading, anchor=west] at (5.35,9.94) {Multi-turn joint development};
\node[small, anchor=east] at (14.27,9.94)
  {Figure~\ref{fig:method-overview}a};
\node[block, draw=wfRtl, minimum width=1.60cm, minimum height=0.55cm,
  inner sep=2pt] (wfAgent) at (7.65,9.39) {Agent};
\node[block, small, minimum width=3.60cm, minimum height=0.55cm,
  inner sep=2pt] (wfTools) at (12.10,9.39) {Tools / local tests};
\draw[flow] (wfTaskSpec) -- (wfAgent.west)
  node[pos=0.12, above=2pt, small] {\(S\)};
\draw[flow] ([yshift=0.08cm]wfAgent.east) --
  ([yshift=0.08cm]wfTools.west);
\draw[feedback] ([yshift=-0.08cm]wfTools.west) --
  ([yshift=-0.08cm]wfAgent.east);

\node[block, draw=wfRtl, line width=0.85pt, align=center,
  minimum width=4.20cm, minimum height=0.78cm, inner sep=3pt]
  (wfSubmission) at (7.65,8.08)
  {\textbf{Submission}\\[2pt]\(\hat R,\;\hat B,\;\mathrm{TB}\)};
\draw[flow] (wfAgent.south) -- (wfSubmission.north);

\path[group] (4.05,4.78) rectangle (14.55,7.25);
\node[heading, anchor=west] at (4.35,6.95) {BEHAVE-Sim};
\node[small, anchor=east] at (14.27,6.95)
  {Figure~\ref{fig:method-overview}c};
\draw[flow] (wfSubmission.south) -- (7.65,7.25);
\draw[flow] (wfTaskEval) -- (2.65,6.12) -- (4.05,6.12);

\path[block, draw=wfObs] (4.35,5.68) rectangle (8.60,6.55);
\node[heading] at (6.475,6.12) {Stimulus generation};
\path[block, draw=wfObs] (9.25,5.68) rectangle (14.25,6.55);
\node[heading] at (11.75,6.30) {Execute and compare};
\node[small] at (11.75,5.94)
  {\(\hat R\) vs. \(B^\star\)\qquad\(\hat B\) vs. \(B^\star\)};
\draw[flow] (8.60,6.12) -- (9.25,6.12);
\path[block, draw=wfEnv!65] (4.35,4.98) rectangle (14.25,5.38);
\node[small] at (9.30,5.18)
  {\(\operatorname{score}_{\hat R},\;\operatorname{score}_{\hat B}\) and execution records};
\draw[flow] (11.75,5.68) -- (11.75,5.38);

\coordinate (wfRoute) at (7.65,4.35);
\draw[draw=black!65, line width=0.8pt] (7.65,4.98) -- (wfRoute);
\node[small, anchor=west] at (7.92,4.59) {By run type};
\draw[draw=black!65, line width=0.8pt] (2.50,4.35) -- (12.80,4.35);
\fill[black!65] (wfRoute) circle (1pt);

\node[block, heading, minimum width=3.40cm, minimum height=0.53cm,
  inner sep=2pt] (wfReport) at (2.50,2.985) {Benchmark report};
\node[block, heading, minimum width=4.20cm, minimum height=0.53cm,
  inner sep=2pt] (wfGaps) at (7.65,2.985) {Capability gaps};
\node[block, heading, minimum width=3.50cm, minimum height=0.53cm,
  inner sep=2pt] (wfReward) at (12.80,2.985) {Reward \(r\)};
\draw[flow] (2.50,4.35) -- (wfReport.north);
\draw[flow] (wfRoute) -- (wfGaps.north);
\draw[flow] (12.80,4.35) -- (wfReward.north);
\node[wirelabel] at (2.50,3.85) {Held-out\\evaluation};
\node[wirelabel] at (7.65,3.85) {Self-improvement};
\node[wirelabel] at (12.80,3.85) {Training rollouts};

\node[block, heading, minimum width=3.50cm, minimum height=0.55cm,
  inner sep=2pt] (wfTrain) at (12.80,1.975) {RL training};
\draw[flow] (wfReward.south) -- (wfTrain.north);
\draw[feedback] (wfTrain.east) -- (15.55,1.975)
  -- (15.55,8.72) -- (8.15,8.72) -- ([xshift=0.50cm]wfAgent.south);
\node[wirelabel, text=wfEnv] at (11.75,8.72) {Updated agent};

\path[group] (4.95,0.82) rectangle (10.35,2.23);
\node[heading] at (7.65,2.00) {Autonomous task acquisition};
\draw[flow] (wfGaps.south) -- (7.65,2.23);
\path[block] (5.17,1.03) rectangle (7.55,1.77);
\node[small, align=center] at (6.36,1.40) {Search code\\sources};
\path[block] (7.98,1.03) rectangle (10.13,1.77);
\node[small, align=center] at (9.055,1.40) {Construct and\\check tasks};
\draw[flow] (7.55,1.40) -- (7.98,1.40);
\draw[feedback] (9.055,1.03) -- (9.055,0.52)
  -- (0.34,0.52) -- (0.34,9.25) -- (wfNewTaskIn);
\node[wirelabel, text=wfEnv] at (5.20,0.52)
  {New training tasks \((S,B^\star)\)};

\draw[flow] (3.40,-0.10) -- (4.00,-0.10);
\node[small, anchor=west] at (4.13,-0.10) {Data flow};
\draw[feedback] (9.30,-0.10) -- (9.90,-0.10);
\node[small, anchor=west] at (10.03,-0.10) {Feedback / updates};
\end{tikzpicture}
\endgroup}
\caption{\textbf{End-to-end workflow.} The shared evaluator supports held-out
reporting, capability diagnosis for self-improvement, and reward
computation for training rollouts.}
\label{fig:end-to-end-workflow}
\end{figure}
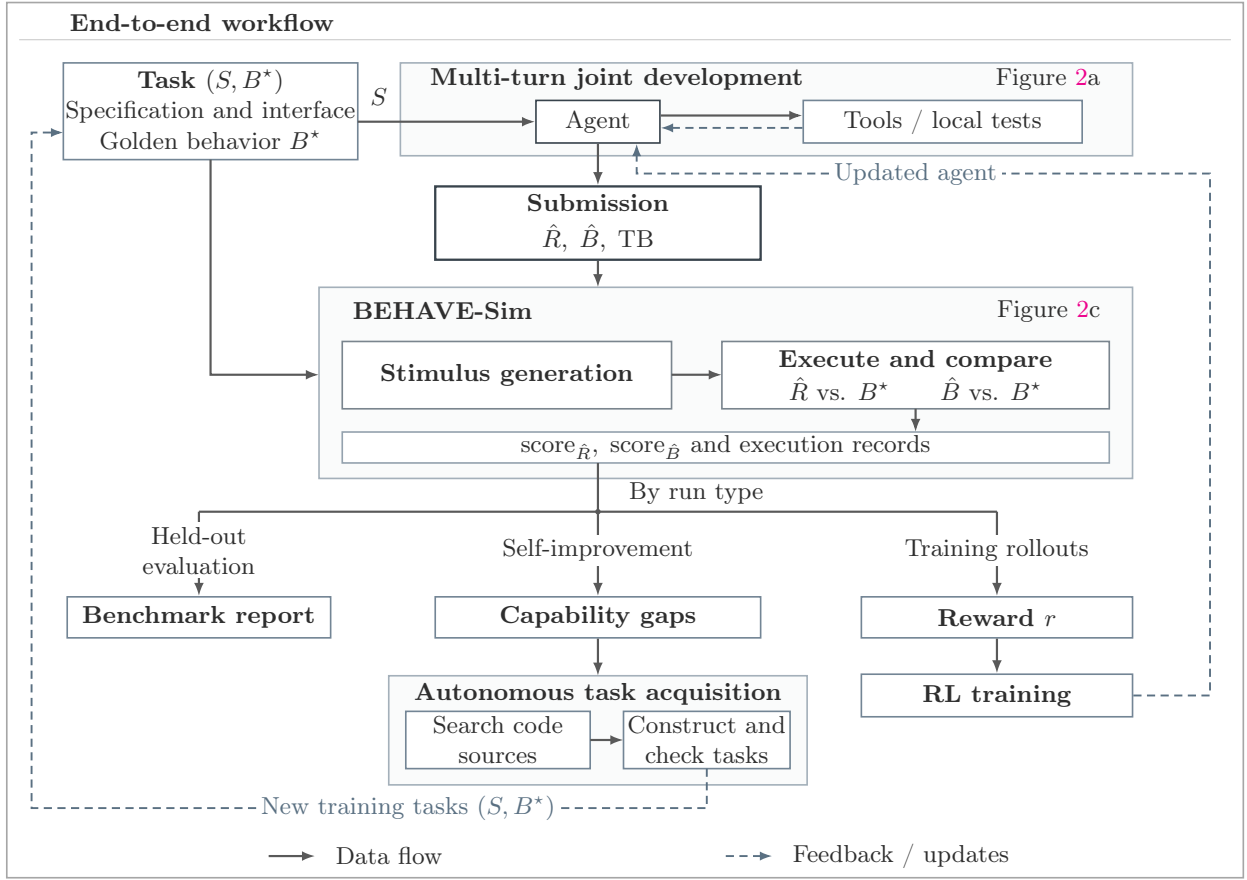

\paragraph{Training and self-improvement.}
Training rollouts supply reward \(r\) for RL updates.
Devel\-opment-set feedback guides acquisition of new
\((S,B^\star)\) pairs for further
training, and the updated agent guides subsequent acquisition rounds.
The held-out set remains fixed for evaluation. Optional bounded analysis is
separate from scoring and rewards, as
described in Appendix~\ref{app:behave-guarantees}.

\subsection{Task and Artifact Example}
\label{app:behave-contracts}
\paragraph{Task.}
The MXFP4 block-quantization task accepts 32 BF16 encodings and returns a
shared E8M0 scale and 32 E2M1 codes packed into 16 bytes. The specification
defines scale selection, round-to-nearest-even quantization, saturation and
packing order. Its non-finite-input policy assigns scale code 255 and zero
data to any block containing NaN or infinity.

\paragraph{Behavior and RTL.}
Figures~\ref{fig:behavior-example} and~\ref{fig:rtl-example} show the behavior
model and one RTL implementation. The model selects the shared scale,
quantizes the inputs and packs the outputs in \texttt{process()}, without
retaining state across calls. The RTL registers each block and
its scale before quantization and output packing, forming a two-stage pipeline.
Both stages stall under output backpressure.

\begin{figure}[!t]
\begin{lstlisting}[style=papercode, language=Python]
class ai_ml_quant_mx_mxfp4_block_quant:
    def __init__(self):
        self.reset()

    def reset(self):
        pass

    def process(self, in_x0, ..., in_x31):
        xs = [in_x0, ..., in_x31]
        maximum = 0
        for i in range(32):
            magnitude = xs[i] & 32767
            if magnitude > maximum:
                maximum = magnitude
        scale = (maximum >> 7) - 2
        if scale < 0:
            scale = 0
        if maximum >= 32640:
            scale = 255
        q = [0] * 32
        for i in range(32):
            q[i] = quantize_one(xs[i], scale)
        return {
            "out_scale": scale,
            "out_q0": q[0] | (q[1] << 4),
            # Remaining output fields follow the same packing rule.
            "out_q15": q[30] | (q[31] << 4)
        }
\end{lstlisting}
\caption{\textbf{Behavior model structure and computation.}
Ellipses abbreviate repeated interface fields. The bit-exact element
quantizer \texttt{quantize\_one} is omitted.}
\label{fig:behavior-example}
\end{figure}

\paragraph{Interface alignment.}
Figure~\ref{fig:rtl-example} shows an RTL implementation.
Other designs may use different pipeline depths or resource sharing while respecting the
task's protocol and timing constraints. The same behavior model can check
these designs through protocol-defined input mapping and output matching.
\par

\begin{figure}[!t]
\begin{lstlisting}[style=papercode, language=Verilog,
  morekeywords={logic,always_comb,always_ff}]
logic valid_scale, advance;
logic [15:0] xs [0:31], xs_s1 [0:31];
logic [14:0] maximum;
logic [7:0] scale, scale_s1;
logic [3:0] q [0:31];

assign advance = !out_valid || out_ready;
assign in_ready = advance;

always_comb begin
    xs[0] = in_x0;
    // The remaining input lanes are wired identically.
    xs[31] = in_x31;

    maximum = 0;
    for (integer i = 0; i < 32; i = i + 1)
        if (xs[i][14:0] > maximum)
            maximum = xs[i][14:0];
    scale = maximum[14:7] < 2 ? 0 : maximum[14:7] - 8'd2;
    if (maximum >= 15'h7F80) scale = 255;
    for (integer i = 0; i < 32; i = i + 1)
        q[i] = quantize(xs_s1[i], scale_s1);
end

always_ff @(posedge clk or posedge rst) begin
    if (rst) begin
        valid_scale <= 0;
        out_valid <= 0;
        scale_s1 <= 0;
        for (integer i = 0; i < 32; i = i + 1) xs_s1[i] <= 0;
        out_scale <= 0;
        out_q0 <= 0;
        // out_q1 through out_q14 are also cleared.
        out_q15 <= 0;
    end else if (advance) begin
        valid_scale <= in_valid;
        out_valid <= valid_scale;
        if (in_valid) begin
            scale_s1 <= scale;
            for (integer i = 0; i < 32; i = i + 1)
                xs_s1[i] <= xs[i];
        end
        if (valid_scale) begin
            out_scale <= scale_s1;
            out_q0 <= {q[1], q[0]};
            // out_q1 through out_q14 pack the remaining pairs.
            out_q15 <= {q[31], q[30]};
        end
    end
end
\end{lstlisting}
\caption{\textbf{Two-stage pipelined RTL.}
Port declarations, repeated lane assignments and the bit-exact element
quantizer \texttt{quantize} are omitted.}
\label{fig:rtl-example}
\end{figure}

\subsection{Execution Architecture}
\label{app:behave-architecture}
\label{app:behavesim-architecture}

\paragraph{RTL execution.}
As shown in Figure~\ref{fig:artifact-execution-architecture}, \behavesim
executes \(\hat R\) with an input driver, an output receiver, clock/reset
control, and a task-declared external-memory model when needed. The stimulus
supplies input transactions and schedules their presentation, output readiness,
and clock/reset events. Passive interface observations feed task-defined
input/output mapping and protocol/timing monitors. \behavesim records logical
I/O and collects goal hits from monitors and RTL instrumentation.
\par

\paragraph{Behavior execution.}
For \(\hat B\), logical inputs and task-visible events directly drive model
invocations. State persists between invocations, with resets applied as
declared by the task. \behavesim records outputs and goal hits, while
\(B^\star\) runs with independent state on the same input history to compute
expected outputs. Behavior and RTL evaluation use the same output-matching
and scoring rules.

\begin{figure}[!t]
\centering
\resizebox{\textwidth}{!}{
\definecolor{beEnv}{HTML}{5A7082}%
\definecolor{beRtl}{HTML}{38434C}%
\definecolor{beObs}{HTML}{7A858E}%
\begin{tikzpicture}[
  x=1cm, y=1cm,
  font=\normalfont\fontsize{9}{11}\selectfont,
  ink/.style={text=black!82},
  heading/.style={ink, font=\normalfont\bfseries\fontsize{9}{11}\selectfont},
  small/.style={ink, font=\normalfont\fontsize{8}{9.5}\selectfont},
  block/.style={line width=0.6pt},
  environment/.style={block, draw=beEnv!85, fill=white},
  logic/.style={block, draw=beRtl, fill=white},
  check/.style={block, draw=beObs, fill=white},
  envwire/.style={-{Latex[length=1.7mm]}, draw=beEnv, line width=0.85pt},
  rtlwire/.style={-{Latex[length=1.7mm]}, draw=beRtl, line width=0.85pt},
  datawire/.style={-{Latex[length=1.7mm]}, draw=black!65, line width=0.75pt},
  observe/.style={draw=beObs, line width=0.65pt, densely dashed},
  observation/.style={observe, -{Latex[length=1.6mm]}},
  config/.style={-{Latex[length=1.4mm]}, draw=beEnv, line width=0.55pt},
  wirelabel/.style={small, fill=white, inner sep=1.5pt},
  pinlabel/.style={font=\ttfamily\fontsize{8}{9.5}\selectfont,
    fill=white, inner sep=1.5pt}
]
\path[use as bounding box] (0,0) rectangle (16,8.55);
\path[draw=black!35, fill=white, line width=0.6pt]
  (0.02,0.02) rectangle (15.98,8.53);

\begin{scope}[yshift=-3cm]
\path[fill=beEnv!3] (0.18,5.47) rectangle (15.82,11.02);
\node[anchor=west, heading] at (0.45,11.27) {BEHAVE-Sim};
\draw[draw=black!18, line width=0.45pt] (0.18,11.02) -- (15.82,11.02);
\node[heading, anchor=east, align=right] at (15.40,10.50)
  {BEHAVE-Sim environment};

\path[environment] (0.60,9.95) rectangle (3.20,10.78);
\node[heading] at (1.90,10.50) {Stimulus \(\xi\)};
\node[small] at (1.90,10.18) {inputs and schedule};
\path[environment] (6.35,10.05) rectangle (10.20,10.78);
\node[heading] at (8.275,10.50) {Clock / reset control};
\draw[draw=beEnv, line width=0.7pt]
  (7.23,10.12) -- (7.48,10.12) -- (7.48,10.27) -- (7.73,10.27)
  -- (7.73,10.12) -- (7.98,10.12) -- (7.98,10.27) -- (8.23,10.27);
\draw[draw=beEnv, line width=0.7pt]
  (8.67,10.27) -- (9.00,10.27) -- (9.00,10.12) -- (9.34,10.12);
\draw[config] (3.20,10.42) -- (6.35,10.42)
  node[midway, above=2pt, small] {schedule};
\draw[config] (1.90,9.95) -- (1.90,9.48);
\draw[config] (0.60,10.35) -- (0.35,10.35) -- (0.35,6.85) -- (0.60,6.85);

\path[environment] (0.60,7.95) rectangle (3.20,9.48);
\node[heading] at (1.90,9.16) {Input driver};
\foreach \x/\token in {0.96/3,1.52/2,2.08/1} {
  \path[draw=beEnv!70, fill=white, line width=0.45pt]
    (\x,8.36) rectangle ++(0.43,0.40);
  \node[small] at (\x+0.215,8.56) {\(u_{\token}\)};
}
\draw[config] (2.63,8.56) -- (2.99,8.56);

\path[environment] (0.60,6.13) rectangle (3.20,7.42);
\node[heading] at (1.90,7.11) {Output receiver};
\draw[draw=beEnv, line width=0.75pt]
  (0.95,6.75) -- (1.24,6.75) -- (1.24,6.56) -- (1.73,6.56)
  -- (1.73,6.75) -- (2.23,6.75) -- (2.23,6.56) -- (2.78,6.56);
\node[small] at (1.90,6.34) {backpressure};

\path[logic, fill=beEnv!5, line width=1.05pt]
  (6.55,6.22) rectangle (10.00,9.33);
\foreach \y in {8.90,8.22,7.08,6.42}
  \path[draw=beRtl, fill=white, line width=0.65pt]
    (6.50,\y-0.05) rectangle (6.60,\y+0.05);
\foreach \y in {8.90,8.22}
  \path[draw=beRtl, fill=white, line width=0.65pt]
    (9.95,\y-0.05) rectangle (10.05,\y+0.05);
\node[heading] at (8.275,9.00) {Submitted RTL \(\hat R\)};
\foreach \y in {7.61,7.87,8.13}
  \path[draw=beRtl!35, fill=white, line width=0.5pt]
    (6.98,\y) rectangle ++(0.78,0.16);
\path[draw=beRtl!35, fill=white, line width=0.5pt]
  (8.68,8.37) -- (9.35,8.37) -- (9.57,7.95)
  -- (9.35,7.53) -- (8.68,7.53) -- (8.90,7.95) -- cycle;
\draw[datawire, draw=beRtl!35, line width=0.5pt]
  (7.76,7.95) -- (8.88,7.95);
\draw[datawire, draw=beRtl!35, line width=0.5pt]
  (9.57,7.95) -- (9.76,7.95) -- (9.76,7.18)
  -- (7.37,7.18) -- (7.37,7.61);
\draw[envwire] (8.275,10.05) -- (8.275,9.33);

\foreach \lo/\hi in {8.10/9.03,6.30/7.21}
  \path[observe, rounded corners=1pt] (3.40,\lo) rectangle (3.72,\hi);
\path[observe, rounded corners=1pt] (10.45,8.10) rectangle (10.77,9.03);
\path[observe, rounded corners=1pt] (8.18,9.60) rectangle (8.37,9.92);
\draw[observe] (3.72,8.56) -- (3.94,8.56) -- (3.94,5.64);
\draw[observe] (3.72,6.75) -- (3.94,6.75);
\draw[observe] (10.45,8.56) -- (10.23,8.56);
\draw[observe] (8.37,9.76) -- (10.23,9.76) -- (10.23,5.64);
\draw[observe] (3.94,5.64) -- (13.15,5.64);
\foreach \x/\y in {3.94/6.75,10.23/8.56,10.23/5.64}
  \fill[beObs] (\x,\y) circle (0.9pt);

\draw[envwire] (3.20,8.90) -- (6.55,8.90)
  node[midway, above=2pt, pinlabel] {in\_valid, in\_data};
\draw[rtlwire] (6.55,8.22) -- (3.20,8.22)
  node[midway, above=2pt, pinlabel] {in\_ready};
\draw[rtlwire] (6.55,7.08) -- (3.20,7.08)
  node[midway, above=2pt, pinlabel] {out\_valid, out\_data};
\draw[envwire] (3.20,6.42) -- (6.55,6.42)
  node[midway, above=2pt, pinlabel] {out\_ready};

\path[environment] (12.60,7.40) rectangle (15.40,9.48);
\node[heading, align=center] at (14.00,9.08) {External memory\\model};
\foreach \x in {13.13,13.57,14.01,14.45}
  \foreach \y in {7.66,8.00,8.34}
    \path[draw=beEnv!50, fill=beEnv!4, line width=0.4pt]
      (\x,\y) rectangle ++(0.34,0.22);
\draw[rtlwire] (10.00,8.90) -- (12.60,8.90)
  node[midway, above=2pt, wirelabel] {request};
\draw[envwire] (12.60,8.22) -- (10.00,8.22)
  node[midway, above=2pt, wirelabel] {response};

\coordinate (rtlInternalTap) at (8.825,6.22);
\coordinate (mappingTap) at (3.94,5.64);
\coordinate (monitorTap) at (13.15,5.64);
\end{scope}

\node[small, fill=white, inner sep=1pt] at (6.10,2.82) {Interface observations};
\path[check] (0.60,1.61) rectangle (6.85,2.22);
\node[heading] at (3.725,1.915) {Input/output mapping};
\draw[observation] (mappingTap) -- (3.94,2.22);
\path[check] (10.90,1.61) rectangle (15.40,2.22);
\node[heading] at (13.15,1.915) {Protocol / timing monitors};
\draw[observation] (monitorTap) -- (13.15,2.22);

\path[block, draw=beObs!75, fill=beEnv!3]
  (0.60,0.63) rectangle (15.40,1.17);
\draw[draw=beObs!45, line width=0.45pt]
  (3.80,0.63) -- (3.80,1.17) (7.70,0.63) -- (7.70,1.17);
\node[heading] at (2.20,0.90) {Execution records};
\node[small] at (5.75,0.90) {Logical I/O};
\node[small] at (11.55,0.90) {Goal hits};
\draw[datawire] (5.20,1.61) -- (5.20,1.17);
\draw[datawire] (13.15,1.61) -- (13.15,1.17);
\draw[observation, preaction={draw=white,line width=2.3pt}]
  (rtlInternalTap) -- (8.825,1.17);
\node[small, anchor=west, align=left] at (9.00,1.915) {Internal\\goal hits};

\draw[datawire] (2.50,0.25) -- (3.10,0.25);
\node[small, anchor=west] at (3.23,0.25) {Signal / record flow};
\draw[observation] (9.20,0.25) -- (9.80,0.25);
\node[small, anchor=west] at (9.93,0.25) {Passive observation};
\end{tikzpicture}}
\caption{\textbf{RTL execution environment in \behavesim.}
Expanded view of the RTL execution in Figure~\ref{fig:method-overview}(c),
using ready/valid interfaces and an optional task-declared external-memory
model with a schematic request/response interface. Task-defined mapping and
monitors record logical I/O, task-visible events, and goal hits. Internal
goal hits are collected by instrumentation.}
\label{fig:artifact-execution-architecture}
\end{figure}
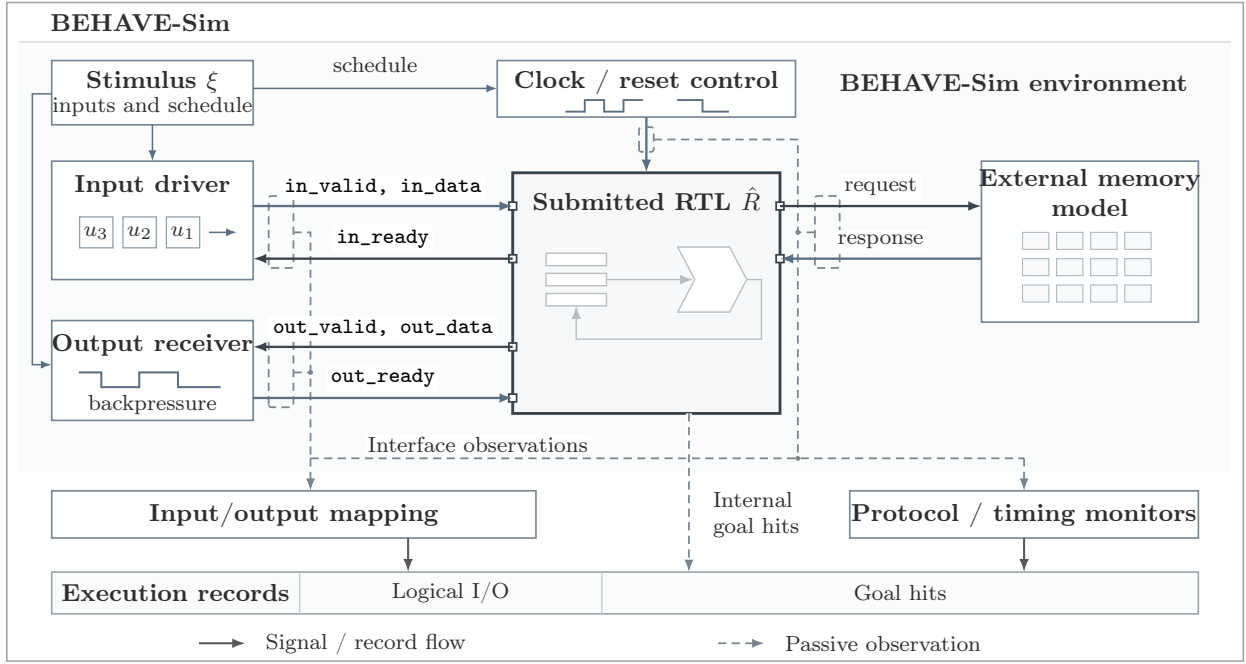

\subsection{Supported Features and Limitations}
\label{app:behave-support}

\paragraph{Interfaces and protocols.}
Table~\ref{tab:interface-capabilities} summarizes interface capabilities and
representative protocols. Each task specifies its protocol configurations and
may combine capabilities across its interfaces.
\par

\begin{table}[!t]
\centering
\renewcommand{\arraystretch}{1.08}
\setlength{\tabcolsep}{6pt}
\caption{\textbf{Interface capabilities and representative protocols.}}
\label{tab:interface-capabilities}
\begin{tabular}{@{}
  l
  >{\raggedright\arraybackslash}p{5.6cm}
  l@{}}
\toprule
\textbf{Capability} & \textbf{Interaction requirements} &
\textbf{Typical interfaces} \\
\midrule
Sampling and flow control &
Sampling events, transfer detection and backpressure &
\makecell[tl]{Plain, valid-only and\\ready/valid interfaces} \\
\addlinespace
Framing and byte validity &
Frame boundaries, valid bytes and sideband fields &
AXI4-Stream~\citep{arm2010axistream} \\
\addlinespace
Phased transfers &
Setup, access, wait states and completion &
APB~\citep{arm2023apb} \\
\addlinespace
Multi-channel assembly &
Independently arriving address and data &
AXI4-Lite~\citep{arm2020axi} \\
\addlinespace
Outstanding requests &
Pending-request limits and response ordering &
\makecell[tl]{Memory request/response\\interfaces} \\
\addlinespace
\makecell[tl]{Invocation and\\completion} &
Start, busy, completion\newline and result consumption &
Start/done interfaces \\
\addlinespace
\makecell[tl]{Multi-interface\\composition} &
Mixed protocols and cross-interface ordering &
\makecell[tl]{AXI4-Lite configuration\\with AXI4-Stream data} \\
\bottomrule
\end{tabular}
\end{table}

\paragraph{Transaction mapping.}
The task defines how RTL interface events form logical inputs to the behavior
model and how observed outputs are matched with its predictions. A logical
operation may involve one transfer or several transfers. This mapping is fixed
before evaluation.

\paragraph{Interface composition.}
A task may combine different interfaces, such as configuration and data
interfaces. It specifies when configuration changes take effect and how
simultaneous interactions are handled, so that the behavior model captures
their combined effect.
\par

\paragraph{Capability separation.}
Table~\ref{tab:behave-support} distinguishes the features
supported by evaluation and replay, goal-guided stimulus generation and exact
bounded analysis. Generation and exact analysis require
their own encodings of the selected protocol and transaction mapping.
Goal-guided generation may extend the evaluation stimulus
pool, whose concrete execution records determine the artifact score. Exact bounded analysis is
optional, runs after scoring and changes neither the evaluation pool nor the
score.

\begin{table*}[!tp]
\centering
\renewcommand{\arraystretch}{1.10}
\setlength{\tabcolsep}{4pt}
\caption{\textbf{\behavesim features and limitations.} Rows list artifact
features, and the remaining columns report support for evaluation and replay,
goal-guided generation and exact bounded analysis.}
\label{tab:behave-support}
\resizebox{\textwidth}{!}{%
\begin{tabular}{@{}
  >{\raggedright\arraybackslash}p{2.3cm}
  >{\raggedright\arraybackslash}p{5.3cm}
  >{\raggedright\arraybackslash}p{5.3cm}
  >{\raggedright\arraybackslash}p{5.3cm}@{}}
\toprule
\textbf{Feature} & \textbf{Evaluation and replay} &
\textbf{Goal-guided generation} & \textbf{Exact bounded analysis} \\
\midrule
Behavior models &
Behavior models use a restricted Python subset. I/O values may be Boolean,
bounded integers, fixed-size arrays, records or enumerations. &
The \texttt{process()} implementation must have statically proved finite
bounds. &
Every analyzed operation and loop must have a finite exact encoding. \\
\addlinespace
RTL designs &
Fixed-width SystemVerilog designs that pass structural checks execute in
isolated Verilator. Native source-event observation does not require a
statically proved loop trip count. &
Each target goal requires a finite symbolic query relating selected input
fields to whether that goal is reached. &
A faithful encoding must cover the complete RTL transition relation within the
analysis bound. \\
\addlinespace
Protocols &
Behavior uses logical transactions. RTL supports plain,
valid-only, ready/valid and start/done interfaces, task-configured APB,
AXI4-Lite and framed AXI4-Stream subsets, and bounded request/response. &
Requires a finite symbolic query covering the selected
protocol and transaction mapping. Generation varies logical input fields under
fixed protocol rules, transaction order and environment schedule. &
Requires an exact encoding of the selected protocol,
transaction mapping and monitors. Schedules remain fixed except for
single-clock ready/valid with one input and one output. \\
\addlinespace
Endpoints and clock domains &
Replay supports multiple endpoints and clock domains for
Behavior and RTL. A task's RTL interfaces may use different protocols. &
Behavior generation supports multiple endpoints within one clock domain. RTL
generation with multiple endpoints or clocks requires a fixed global schedule. &
Behavior analysis supports multiple endpoints. Multi-clock Behavior and
multiple-endpoint or multi-clock RTL require all analyzed stimuli to share one
complete global clock schedule. \\
\addlinespace
Reset &
Replay follows the reset events in each stimulus. Reset during traffic requires
the task to define how pending transactions are handled. &
Goal-guided generation does not add or move reset events. &
Reset timing remains fixed unless it is explicitly represented as a symbolic
schedule choice. \\
\addlinespace
State and internal memory &
Replay supports persistent Behavior state, RTL registers and writable internal
memories. &
Behavior state must be finite and initialized. RTL goal-guided generation does
not support writable internal memories. &
All analyzed state must have a finite exact encoding. RTL exact analysis does
not support writable internal memories. \\
\addlinespace
External memory &
Replay supports a task-defined external-memory model for RTL and single-domain
Behavior models. &
Goal-guided generation is unavailable for artifacts that access external
memory. &
Exact analysis supports external memory only for RTL, using a fixed finite
memory model. \\
\addlinespace
Goals and monitors &
Replay records reached coverage goals and detected safety violations. For RTL,
monitor violations are recorded as safety-goal hits. &
A goal is targeted only when a finite symbolic query can propose inputs that
reach it. &
Every catalogued goal, including monitor violations, must have an exact
reachability encoding. \\
\bottomrule
\end{tabular}
}
\end{table*}

\paragraph{Unsupported and unresolved cases.}
If a symbolic capability cannot represent an artifact, goal or schedule, it
refuses without affecting evaluation and replay. A generation query
finding no candidate adds no stimulus and proves nothing about reachability.
A timeout or unknown result in exact analysis leaves the corresponding goal
unresolved. None of these outcomes changes a completed artifact score. An
artifact failure during concrete evaluation replay remains a scored evaluation failure.

\subsection{Transaction Mapping and Output Matching}
\label{app:bs-setting}

\paragraph{Overview.}
\behavesim evaluates \(\hat B\) and \(\hat R\) separately against \(B^\star\) on stimulus sets. For each stimulus, it records which inputs the target accepts under each schedule entry, obtains the expected output histories from \(B^\star\), and matches them against the target's observed output histories. This subsection formalizes stimuli, accepted inputs, output histories, and matching.

\begin{definition}[Transaction interface]
\label{def:txn-interface}
An \emph{endpoint} \(c\) of an artifact \(A\) is a named input or output with
a nonempty finite type \(T_c\). A \emph{transaction} on \(c\) is a value
\(x\in T_c\) exchanged between \(A\) and its environment. Transactions on each
endpoint form an ordered stream. The task declares finite families of input
endpoints \(C^{\mathrm{in}}\) and output endpoints \(C^{\mathrm{out}}\), which
form the logical transaction interface.
\end{definition}

\begin{definition}[Stimulus]
\label{def:stimulus}
For each input endpoint \(c\), let
\[
  \mathbf i_c=(x_{c,0},\ldots,x_{c,k_c-1}) \in T_c^{k_c}
\]
be an ordered queue of \(k_c\) transactions, and write
\(\mathbf i=(\mathbf i_c)_{c\in C^{\mathrm{in}}}\) for the queue family.
A \emph{stimulus} is a pair \(\xi=(\mathbf i,\sigma)\), where
\[
  \sigma=(\sigma_t)_{0\le t<m}
\]
is a finite schedule of length \(m\). Each entry
\(\sigma_t\) specifies clock and reset events and an environment action.
The schedule controls queue presentation, allowing simultaneous presentation
across endpoints while preserving each queue's order.
A runtime request supplies \(\mathbf i\) and an execution policy,
recording \(\sigma\) during execution. Pending inputs are offered until accepted.
After all inputs are accepted and expected outputs collected, a finite
observation tail follows. Task-declared timing requirements are checked
separately. Resource exhaustion alone is not a task violation.
\end{definition}

\begin{definition}[Accepted input trace]
\label{def:accepted-input-trace}
Fix \(A\in\{\hat B,\hat R\}\) and a stimulus \(\xi=(\mathbf i,\sigma)\)
with \(|\sigma|=m\). The index \(t\) identifies a schedule entry, not a cycle
count for any particular clock domain. For \(0\le t<m\) and
\(c\in C^{\mathrm{in}}\), let
\(\bot\) be a distinguished symbol outside every transaction type and let
\[
  a_{t,c}^A(\xi)\in T_c\cup\{\bot\}
\]
be the transaction from \(\mathbf i_c\) that \(A\) accepts on endpoint \(c\) under \(\sigma_t\), or \(\bot\) if it accepts none. The \emph{accepted input batch} under \(\sigma_t\) is
\[
  a_t^A(\xi)
  :=
  \bigl(a_{t,c}^A(\xi)\bigr)_{c\in C^{\mathrm{in}}}
  \in
  \prod_{c\in C^{\mathrm{in}}}
  \bigl(T_c\cup\{\bot\}\bigr).
\]
The \emph{accepted input trace} is
\[
  \mathbf a_A(\xi)
  =\bigl(a_0^A(\xi),\ldots,a_{m-1}^A(\xi)\bigr).
\]
Thus one batch corresponds to one schedule entry. A batch may be all-\(\bot\) when no endpoint accepts a transaction, and several non-\(\bot\) entries in one batch are accepted together.
For \(A=\hat B\), the Behavior interface has no input handshake, so every presented input batch is accepted and placed at its scheduled entry in \(\mathbf a_{\hat B}(\xi)\). For \(A=\hat R\), the non-\(\bot\) entries of \(\mathbf a_{\hat R}(\xi)\) are the acceptance events reported by its declared protocol adapter.
\end{definition}

\begin{definition}[Output histories]
\label{def:output-histories}
For each output endpoint \(c\), let \(\ell_{A,c}(\xi)\) be the number of output transactions observed on \(c\) under \(\xi\), and write
\[
  o_{A,c}(\xi)
  =
  \bigl(
    o_{A,c,0}(\xi),
    \ldots,
    o_{A,c,\ell_{A,c}(\xi)-1}(\xi)
  \bigr)
  \in T_c^{\ell_{A,c}(\xi)}.
\]
Collect these histories as
\[
  \mathbf o_A(\xi)
  =\bigl(o_{A,c}(\xi)\bigr)_{c\in C^{\mathrm{out}}}.
\]
Let \(\pi\) be the task-defined mapping from an execution
schedule and an accepted-input trace to the input history supplied to
\(B^\star\). Thus, \(\pi(\sigma,\mathbf a)\) contains the accepted
transactions, their declared ordering, and task-visible events such as resets
and external-memory responses, while omitting functionally irrelevant
schedule details.
The corresponding expected histories are
\[
  \mathbf o_A^\star(\xi)
  =B^\star\bigl(\pi(\sigma,\mathbf a_A(\xi))\bigr),
  \qquad
  o_{A,c}^\star(\xi)=\bigl(\mathbf o_A^\star(\xi)\bigr)_c.
\]
The task requires \(B^\star\) to determine every expected history completely
and deterministically, including its length and every
compared field. Behavior-model histories are observed directly, while RTL histories are
reconstructed from the emission events reported by its protocol adapter.
\end{definition}

\paragraph{Matching outputs.}
For each output endpoint \(c\), the observed history \(o_{A,c}(\xi)\) is matched against the expected history \(o^\star_{A,c}(\xi)\) under a task-declared policy. By default, transactions are paired in order. If the task permits reordered completion, it may declare an identifier that uniquely distinguishes outstanding transactions, and transactions with the same identifier are paired. An endpoint without such an identifier may instead use content-based matching under Assumption~\ref{asm:content-key}. Each unpaired expected or observed transaction and each paired transaction with unequal fields contributes one mismatch. For \(A\in\{\hat B,\hat R\}\), let \(n^A_{\mathrm{mis}}(\xi)\) be the total number of mismatches across output endpoints.

\begin{assumption}[Validity of content-based matching]
\label{asm:content-key}
For every output endpoint compared by content, output order carries no task-defined meaning. An observed history satisfies the task's output requirement exactly when it contains the same multiset of transaction values as the expected history.
\end{assumption}

\begin{definition}[Evaluation counts]
For \(A\in\{\hat B,\hat R\}\), let \(n^A_{\mathrm{txn}}(\xi)\) denote the number of input transactions that \(A\) accepts under \(\xi\):
\[
  n^A_{\mathrm{txn}}(\xi)
  = \sum_{t=0}^{|\mathbf a_A(\xi)|-1}
      \sum_{c\in C^{\mathrm{in}}}
      \mathbf 1\!\left[a^A_{t,c}(\xi)\neq\bot\right].
\]
For a nonempty finite stimulus set \(\Xi\), let
\[
  n^A_{\mathrm{txn}}(\Xi)
  =\sum_{\xi\in\Xi}n^A_{\mathrm{txn}}(\xi),
  \qquad
  n^A_{\mathrm{mis}}(\Xi)
  =\sum_{\xi\in\Xi}n^A_{\mathrm{mis}}(\xi).
\]
The accepted-transaction count is activity telemetry rather than a correctness
condition.
\end{definition}

\subsection{Goals, Coverage and Scoring}
\label{app:bs-goals}

\paragraph{Overview.}
Output matching checks whether an artifact agrees with \(B^\star\) on the generated stimuli. That comparison concerns only the observed outputs and does not characterize what occurs during those executions. We therefore examine which execution cases the stimuli exercise and which bad events they trigger. This subsection defines and classifies those cases and events in a finite catalogue.

\begin{definition}[Goals]
\label{def:goals}
For each target \(T\in\{B^\star,\hat B,\hat R\}\), let \(\Gamma_T\) be a finite catalogue of named events fixed before execution. A goal is \emph{declared} when it is added to this catalogue, whether it is derived automatically or supplied explicitly by the task. An execution of \(T\) \emph{hits} a goal \(\gamma\in\Gamma_T\) exactly when the event named by \(\gamma\) occurs during that execution. Every goal belongs to exactly one class:
\[
  \Gamma_T
  =
  \Gamma_T^{\mathrm{cov}}\cup\Gamma_T^{\mathrm{safe}},
  \qquad
  \Gamma_T^{\mathrm{cov}}\cap\Gamma_T^{\mathrm{safe}}
  =
  \emptyset.
\]
A coverage goal names an execution case that testing should exercise. A safety goal names a bad event that should not occur, either a designated artifact failure or a violation detected by an attached monitor. The catalogue of \(B^\star\) is used only to guide stimulus generation. The catalogues of \(\hat B\) and \(\hat R\) are evaluated and reported separately.
\end{definition}

\begin{figure}[!htbp]
\centering
\footnotesize
\resizebox{\linewidth}{!}{%
\begin{tikzpicture}[
  taxbox/.style={draw, rounded corners=2pt, align=left, inner sep=4pt,
                 text width=#1, font=\footnotesize},
  rootbox/.style={taxbox=1.6cm, align=center},
  classbox/.style={taxbox=2.5cm, align=center},
  familybox/.style={taxbox=7.3cm},
  taxline/.style={thick},
  taxarrow/.style={-{Latex[length=1.8mm]}, thick}
]
\node[rootbox] (goals) at (0,0) {Goals};
\node[classbox, right=1.6cm of goals, yshift=1.8cm]
  (coverage) {Coverage goals};
\node[classbox, right=1.6cm of goals, yshift=-2.45cm]
  (safety) {Safety goals};

\node[familybox, right=1.6cm of coverage, yshift=1.55cm] (structure)
  {\textbf{\textcolor{blue!70!black}{Program structure}}\quad
   location \(\cdot\) control choice \(\cdot\)
   expression \(\cdot\) iteration};
\node[familybox, right=1.6cm of coverage] (state)
  {\textbf{\textcolor{violet!75!black}{Value/state}}\quad
   value observation \(\cdot\) operation event \(\cdot\) state activity};
\node[familybox, right=1.6cm of coverage, yshift=-1.55cm] (scenario)
  {\textbf{\textcolor{orange!85!black}{Scenario/timing}}\quad
   sequence \(\cdot\) latency \(\cdot\) schedule/interleaving};
\node[familybox, right=1.6cm of safety, yshift=0.45cm] (failure)
  {\textbf{\textcolor{blue!70!black}{Artifact failure}}\quad
   partial operation \(\cdot\) internal error \(\cdot\) assertion};
\node[familybox, right=1.6cm of safety, yshift=-1.10cm] (obligation)
  {\textbf{\textcolor{violet!75!black}{Monitor obligation}}\quad
   protocol \(\cdot\) timing \(\cdot\) ordering \(\cdot\) reset \(\cdot\) memory};

\path (goals.east) -- ++(0.75cm,0) coordinate (rootfork);
\draw[taxline] (goals.east) -- (rootfork);
\draw[taxarrow] (rootfork) |- (coverage.west);
\draw[taxarrow] (rootfork) |- (safety.west);

\path (coverage.east) -- ++(0.75cm,0) coordinate (covfork);
\draw[taxline] (coverage.east) -- (covfork);
\draw[taxarrow] (covfork) |- (structure.west);
\draw[taxarrow] (covfork) |- (state.west);
\draw[taxarrow] (covfork) |- (scenario.west);

\path (safety.east) -- ++(0.75cm,0) coordinate (safefork);
\draw[taxline] (safety.east) -- (safefork);
\draw[taxarrow] (safefork) |- (failure.west);
\draw[taxarrow] (safefork) |- (obligation.west);
\end{tikzpicture}
}
\caption{\textbf{Goal taxonomy and catalogue sources.} Coverage goals name cases to exercise, while safety goals name bad events to avoid. Colors indicate default schema sources: \textcolor{blue!70!black}{blue} denotes automatic derivation from the artifact or protocol adapter, \textcolor{orange!85!black}{orange} task declarations, and \textcolor{violet!75!black}{purple} either source.}
\label{fig:goal-taxonomy}
\end{figure}
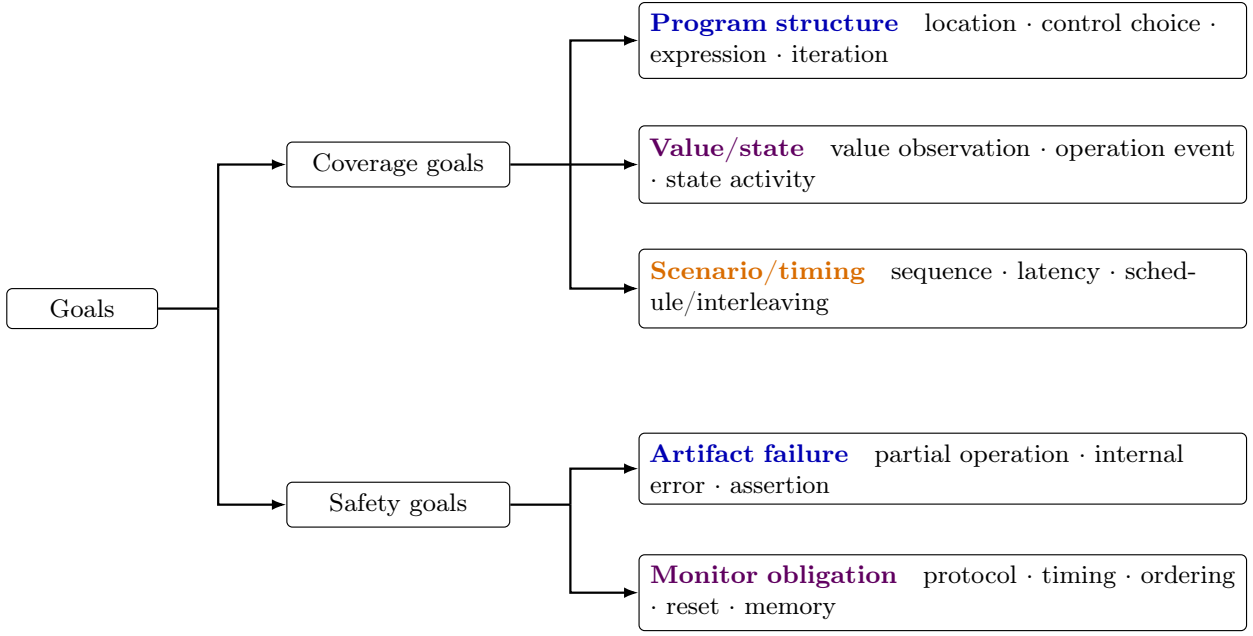

\paragraph{Goal families.}
Figure~\ref{fig:goal-taxonomy} groups goals by class and
family, with finite instance rules and hit conditions for each schema given
in Table~\ref{tab:goal-catalogue}. A finite monitor is a finite-state observer
that emits goals when a declared pattern completes or an obligation is
violated. Tasks and providers may add finite instances or define new schemas
by specifying their instance rules and hit conditions.

\begin{table}[!t]
\centering
\footnotesize
\renewcommand{\arraystretch}{1.20}
\caption{\textbf{Standard goal schemas.} Each row defines a finite instance rule and hit condition for Behavior or RTL.}
\label{tab:goal-catalogue}
\begin{tabular}{@{}
  >{\raggedright\arraybackslash}p{2.4cm}
  >{\raggedright\arraybackslash}p{3.2cm}
  >{\raggedright\arraybackslash}p{4.9cm}
  >{\raggedright\arraybackslash}p{4.7cm}@{}}
\toprule
\textbf{Family} & \textbf{Goal schema} &
\textbf{Eligible instances} & \textbf{Goal is hit when} \\
\midrule
Program structure & location: execution-site entry &
each admitted executable source site of the Behavior or RTL artifact &
execution enters any occurrence of that site \\
Program structure & control choice: binary branch or multiway selection
outcome & the true and false branches of each admitted binary decision, and
each alternative of an admitted multiway selection & execution takes that
branch or selects that alternative \\
Program structure & expression: Boolean-expression evaluation pattern &
each policy-generated partial Boolean assignment to the terms of an admitted
expression & the expression is evaluated, and every term named by the pattern
is evaluated with the specified value \\
Program structure & iteration: loop zero- or non-zero activation &
admitted run-time loops, including procedural RTL loops but excluding generate loops & on a
dynamic entry, control exits before body entry for the zero goal, or takes the
first body-entry edge for the non-zero goal \\
Value/state & value observation: state/value bin & a sampling point and a total Boolean
predicate over one or more values sampled there & the sampling point occurs and
the predicate is true \\
Value/state & operation event: semantic predicate & an admitted operation site
and a total predicate over its operands, result and semantic flags & the
operation occurs and the predicate is true \\
Value/state & state activity: transition or toggle & a sampling point and a
Boolean predicate over two consecutive observations, with one bit and one
direction fixed for a bit-toggle goal & at the second observation, the predicate holds for
the preceding and current values \\
Scenario/timing & sequence: event or transaction sequence & a task-declared finite
monitor over the artifact's transaction or event alphabet & the monitor emits
the named goal when an event completes the sequence \\
Scenario/timing & latency: transaction-latency bin & a task-declared event pair,
pairing rule, timebase and finite latency bins & on a matched end event, the
count from its paired start in the declared timebase lies in the named bin \\
Scenario/timing & schedule/interleaving: finite pattern & a finite monitor over
schedule-entry event sets and a declared finite order or simultaneity pattern &
the monitor emits the named goal when the observed entries complete the pattern \\
Artifact failure & failure transition & partial operations and internal
failures of the admitted Behavior semantics, or embedded errors and assertions
designated by the admitted RTL semantics & the designated failure transition
executes and emits the named goal \\
Monitor obligation & obligation violation & an attached finite obligation
monitor over observed Behavior or RTL execution, including task-supplied
assertions & the monitor emits the named violation goal when the obligation is
violated \\
\bottomrule
\end{tabular}
\end{table}

\paragraph{Catalogue construction.}
Before testing, \behavesim fixes a \emph{goal policy}
recording each schema's status, finite identifiers and hit
rules. Structural goals are identified at source sites, with scopes, instances
and call occurrences resolved by the compiler. Each enabled schema
contributes every prescribed instance, regardless of hits.
Optimization and loop unrolling do not redefine the catalogue.

\begin{definition}[Goal coverage]
\label{def:goal-coverage}
For a nonempty finite stimulus set \(\Xi\) for target \(T\), let \(\mathrm{Hit}_T(\Xi)\subseteq\Gamma_T\) be the union of goals observed in the concrete executions of \(T\) under all \(\xi\in\Xi\). For any nonempty set \(G\subseteq\Gamma_T^{\mathrm{cov}}\), define
\[
  \operatorname{cov}_T(\Xi;G)
  =
  \frac{|\mathrm{Hit}_T(\Xi)\cap G|}{|G|}.
\]
Choosing \(G\) as the goals of one family or schema gives its coverage. The safety goals observed in these executions are \(\mathrm{Hit}_T(\Xi)\cap\Gamma_T^{\mathrm{safe}}\).
\end{definition}

\begin{definition}[Artifact score]
\label{def:artifact-score}
For \(A\in\{\hat B,\hat R\}\), define its score on a
nonempty finite stimulus set \(\Xi\) with completed checks as
\[
  \operatorname{score}_A(\Xi)
  =
  \mathbf 1\!\left[
    n^A_{\mathrm{mis}}(\Xi)=0
    \land
    \mathrm{Hit}_A(\Xi)\cap\Gamma_A^{\mathrm{safe}}=\emptyset
  \right].
\]
\end{definition}

\paragraph{Training reward.}
The rollout reward is the mean of the two artifact scores,
\[
  r
  =
  \frac{
    \operatorname{score}_{\hat B}(\Xi_{\hat B})
    +
    \operatorname{score}_{\hat R}(\Xi_{\hat R})
  }{2}.
\]

\paragraph{Common coverage metrics.}
Definition~\ref{def:goal-coverage} recovers branch, expression and toggle coverage by taking \(G\) to be, respectively, the control-choice, Boolean-expression evaluation-pattern and directional bit-toggle goals. FSM state and transition coverage use state-bin and state-transition goals, covergroup coverage uses its declared value-observation goals, and a finite cover property uses a scenario/timing monitor goal. Separately, we report RTL line coverage: the fraction of executable RTL source lines exercised by the stimuli. It is coarser than location-goal coverage because all execution sites on the same source line count as one item.

\subsection{Goal-Guided Exploration and Evaluation}
\label{app:bs-generation}

\paragraph{Overview.}
Stimulus generation combines random testing with solver-guided
exploration~\citep{godefroid2005dart, sen2005cute, majumdar2007hybrid}.
For each \(A\in\{\hat B,\hat R\}\), evaluation has two
phases. First, isolated tests run from reset: the baseline is
executed before reference-directed and self-directed searches generate and
execute further stimuli. Second, compatible sequences form continuous
workloads that preserve state across sequence boundaries. Results from both
phases are aggregated using the score and coverage definitions in
Subsection~\ref{app:bs-goals}.

\paragraph{Search baseline.}
Let \(\mathcal S_A\) denote legal isolated tests with bounded
input queues and finite execution limits. The nonempty baseline
\(\Xi^A_{\mathrm{base}}\subseteq\mathcal S_A\) contains edge-case patterns
and seeded pseudorandom sequences. Search uses baseline and generated
stimuli as seeds for further exploration.

\paragraph{Search targets.}
For \(T\in\{B^\star,A\}\), \(T\to A\) denotes search on \(T\) to generate
stimuli for execution on \(A\). Let \(\Delta_{T\to A}\subseteq\Gamma_T\)
denote goals whose local encodings are enabled and represent
the corresponding catalogue hit rules.
Excluded goals remain in the catalogue and may still be hit during execution.

\begin{definition}[Artifact replay]
\label{def:artifact-replay}
Let \(\mathcal R_A\) be the type of raw execution records for \(A\). For a
generated artifact \(A\) and any legal finite test \(\xi\),
artifact replay is the map
\[
  \textsc{Replay}:\quad
  (A,\xi)
  \longmapsto
  \bigl(\mathbf a_A(\xi),\mathbf o_A(\xi),e_A(\xi)\bigr).
\]
For a runtime request, \textsc{Replay} follows its execution
policy and retains both the request and resolved stimulus in \(e_A(\xi)\).
Here, \(\xi=(\mathbf i,\sigma)\) uses the schedule recorded during execution.
\(e_A(\xi)\in\mathcal R_A\) records the sampled state, inputs, outputs and
goal hits together with a record-completeness flag, an artifact outcome and a
domain identifier. The flag states whether the full scheduled execution was
recorded. The outcome distinguishes normal completion from an artifact failure.
\(\mathbf a_A(\xi)\) and \(\mathbf o_A(\xi)\) are the accepted input trace and
output histories derived from that record as in
Subsection~\ref{app:bs-setting}. The concrete evaluation stores the
collection of execution records
\[
  \mathcal E_A
  :=\{(\xi,e_A(\xi))\mid \xi\in\Xi_A\}.
\]
\end{definition}

\paragraph{Execution feedback.}
Let \(\mathcal M_A\) store execution traces and goal hits for \(A\) and
\(B^\star\), together with the results of completed output comparisons and
monitor checks. Each search obtains goal hits by reusing a suitable execution
record or running the stimulus and saving its record.
We denote this operation by \(\operatorname{Observe}_{T\to A}(\xi;\mathcal M_A)\).
It returns observed goal hits \(J\subseteq\Gamma_T\) from a
complete execution of \(T\) or a trusted prefix ending in an artifact failure.
If neither is available, it returns \(\bot\).
Each direction maintains its own hit set \(H_{T\to A}\subseteq\Gamma_T\),
initialized from baseline hits and updated only by execution feedback.
The unhit searchable goals are
\[
  U_{T\to A}=\Delta_{T\to A}\setminus H_{T\to A}.
\]

\begin{definition}[Solver-guided provider]
\label{def:probe}
Fix \(A\in\{\hat B,\hat R\}\), the generated artifact to be evaluated, and \(\mathcal S_A\) as above. For \(T\in\{B^\star,A\}\), a \emph{solver-guided provider} is a triple
\[
  \mathsf{Prov}_{T\to A}
  =
  \bigl(\Delta_{T\to A},
        \operatorname{Select}_{T\to A},
        \operatorname{Probe}_{T\to A}\bigr).
\]
Given a seed \(\xi=(\mathbf i,\sigma)\in\mathcal S_A\) and
unhit targets \(U\subseteq\Delta_{T\to A}\),
\(\operatorname{Select}_{T\to A}(\xi,U)\) uses the seed's complete execution
record of \(T\) to construct a finite, deterministically ordered list of
solver requests. The list may be empty. Each request \(q\) determines a
nonempty target set
\[
  \emptyset\neq G(q)\subseteq U
\]
and a nonempty set \(F(q)\) of transaction-field positions in \(\mathbf i\)
that the query may change. Before each probe,
\(\textsc{Restrict}(q,U)\) removes any targets hit after \(q\) was selected,
without changing its input fields or execution context.
It returns \(\bot\) if no targets remain.
Let \(\mathcal Q_{T\to A}(\xi)\) denote the requests that
\(\operatorname{Select}_{T\to A}\) can produce for seed \(\xi\), including
versions with reduced but nonempty target sets.
The probe takes a seed \(\xi\in\mathcal S_A\) and a request
\(q\in\mathcal Q_{T\to A}(\xi)\).
Let \(u_F=(u_f)_{f\in F(q)}\) be symbolic replacement values for these
fields, and let \(\mathbf h_q=(h_\gamma)_{\gamma\in G(q)}\) contain one
Boolean hit bit for each target goal. The provider constructs a finite symbolic-execution relation
\[
  \mathcal H_{T\to A}(\xi,q,u_F,\mathbf h_q)
\]
in which \(h_\gamma\) is true exactly when the encoded local execution of \(T\) hits \(\gamma\). The probe asks the solver to satisfy
\[
  \mathsf{Query}_{T\to A}(\xi,q)
  :=
  \mathcal H_{T\to A}(\xi,q,u_F,\mathbf h_q)
  \;\wedge\;
  \bigvee_{\gamma\in G(q)} h_\gamma
\]
over \(u_F\), \(\mathbf h_q\) and any provider-specific auxiliary execution variables in \(\mathcal H_{T\to A}\).
Each local query fixes the recorded schedule \(\sigma\).
For a model \(\nu\), let \(\xi[F(q)\leftarrow\nu(u_F)]\) denote the test
obtained by replacing only the fields in \(F(q)\). Fixed-schedule tests retain
\(\sigma\). Runtime requests retain their execution policy, so their realized
schedules may differ. Then
\[
  \operatorname{Probe}_{T\to A}(\xi,q)
  =
  \begin{cases}
    \xi[F(q)\leftarrow\nu(u_F)],
      & \text{if }\nu\models \mathsf{Query}_{T\to A}(\xi,q)
        \text{ and is decodable},\\
    \bot,
      & \text{otherwise.}
  \end{cases}
\]
\end{definition}

\paragraph{Search budget.}
The baseline is executed and checked before the search clock
starts. Both searches share a wall-clock budget \(\beta\),
including provider preparation, selection, solver calls and candidate
execution. Reference-directed search runs
up to the budget midpoint, followed by self-directed search using the
remaining time. Baseline evaluation, completion of missing isolated
checks and continuous workloads use separate finite execution limits outside
the search budget.

\paragraph{Isolated exploration.}
Algorithm~\ref{alg:explore} initializes each direction's hit set from the
baseline records, then alternates solver-guided stimulus generation with
execution from reset, with scheduled backpressure where
supported. Only seeds with complete execution records are used
to select requests. Search stops when its targets or worklist are exhausted,
its provider is unavailable, or its deadline is reached, retaining the
candidates and records. Write
\[
  \Xi_{T\to A}
  =
  \textsc{Explore}
  \bigl(A,T,\Xi^A_{\mathrm{base}},\mathsf{Prov}_{T\to A},d;\mathcal M_A\bigr)
\]
for the new stimuli retained by that direction.
\par

\begin{algorithm}[!t]
\caption{\textsc{Explore}: Search from baseline execution records}
\label{alg:explore}
\KwIn{artifact \(A\), target \(T\in\{B^\star,A\}\), baseline \(\Xi^A_{\mathrm{base}}\subseteq\mathcal S_A\), provider \(\mathsf{Prov}_{T\to A}\), deadline \(d\), shared record store \(\mathcal M_A\)}
\KwOut{new candidate set \(\Xi_{T\to A}\subseteq\mathcal S_A\), with \(\mathcal M_A\) updated in place}
\tcp{On expiry, retain candidates and records already obtained.}
\(\mathrm{Seen} \gets \Xi^A_{\mathrm{base}}\)\;
\(\mathrm{Work} \gets \textsc{Queue}(\Xi^A_{\mathrm{base}})\)\;
\(\Xi_{T\to A} \gets \emptyset\)\;
\(H\gets\) baseline hits of \(T\) from complete records or trusted failure prefixes in \(\mathcal M_A\)\;
\(U\gets\Delta_{T\to A}\setminus H\)\;
\While{\(\mathrm{Work}\neq\emptyset\) and \(U\neq\emptyset\) and \(\textsc{Now}()<d\)}{
  \(\xi \gets \textsc{Pop}(\mathrm{Work})\)\;
  \If{no complete execution record of \(T\) for \(\xi\) in \(\mathcal M_A\)}{
    \textbf{continue}\;
  }
  \ForEach{\(q\in\operatorname{Select}_{T\to A}(\xi,U)\) while \(\textsc{Now}()<d\)}{
    \(q\gets\textsc{Restrict}(q,U)\)\;
    \If{\(q\neq\bot\)}{
      \(\xi' \gets \operatorname{Probe}_{T\to A}(\xi,q)\)\;
      \If{\(\xi' \neq \bot\) and \(\xi' \notin \mathrm{Seen}\)}{
        \(\mathrm{Seen} \gets \mathrm{Seen} \cup \{\xi'\}\)\;
        \(\textsc{Push}(\mathrm{Work},\xi')\)\;
        \(\Xi_{T\to A} \gets \Xi_{T\to A} \cup \{\xi'\}\)\;
        \(J\gets\operatorname{Observe}_{T\to A}(\xi';\mathcal M_A)\)\;
        \If{\(J\neq\bot\)}{\(H\gets H\cup J\)\;}
        \(U\gets\Delta_{T\to A}\setminus H\)\;
      }
    }
  }
}
\Return \(\Xi_{T\to A}\)\;
\end{algorithm}

\paragraph{Continuous workloads.}
After both searches, the isolated test pool is
\[
  \Xi_A^{\mathrm{iso}}=
    \Xi^A_{\mathrm{base}}\cup\Xi_{B^\star\to A}\cup\Xi_{A\to A}.
\]
For each \(\xi\in\Xi_A^{\mathrm{iso}}\),
\(\textsc{ExecuteAndCheck}(A,B^\star,\xi;\mathcal M_A)\) reuses records and
check results for the same artifacts, stimulus and evaluation configuration.
It performs only missing execution, output comparisons and monitor checks,
and stores the results in \(\mathcal M_A\). We then construct
\[
  \Xi_A^{\mathrm{cont}}=\textsc{Continuous}(\Xi_A^{\mathrm{iso}}).
\]
\(\textsc{Continuous}\) groups compatible sequences in a fixed order,
preserving the task's initial-state, interface and event requirements.
Sequences that cannot be combined are kept as isolated tests. Each workload
is evaluated as a new stimulus using \(\textsc{ExecuteAndCheck}\), with
\(A\) and \(B^\star\) initialized once and state preserved
between sequences. Expected outputs come from running \(B^\star\)
continuously, not concatenating isolated-test results.
By default, output receivers remain ready.

\paragraph{Result aggregation.}
The final evaluation pool is
\[
  \Xi_A=\Xi_A^{\mathrm{iso}}\cup\Xi_A^{\mathrm{cont}}.
\]
Completed records from both phases provide the output mismatches and
artifact goal hits used by the score and coverage definitions in
Subsection~\ref{app:bs-goals}. Algorithm~\ref{alg:twosided} gives the full
procedure, returning \(\Xi_A\), \(\operatorname{score}_A(\Xi_A)\) and
the execution records \(\mathcal E_A\).
The procedure below assumes completed checks.
Evaluation-budget interruptions are handled as described in
Appendix~\ref{app:exp-evaluation}.
\par

\begin{algorithm}[!t]
\caption{\textsc{GenerateAndEvaluate}: Two-sided search and two-phase evaluation}
\label{alg:twosided}
\KwIn{reference \(B^\star\), artifact \(A\), baseline \(\Xi^A_{\mathrm{base}}\), providers \(\mathsf{Prov}_{B^\star\to A}\) and \(\mathsf{Prov}_{A\to A}\), total search time \(\beta\geq0\)}
\KwOut{stimulus set \(\Xi_A\), score \(\operatorname{score}_A(\Xi_A)\), execution records \(\mathcal E_A\)}
\(\mathcal M_A\gets\emptyset\)\;
\tcp{Phase 1: evaluate the baseline, then start search timing.}
\ForEach{\(\xi\in\Xi^A_{\mathrm{base}}\)}{
  \(\textsc{ExecuteAndCheck}(A,B^\star,\xi;\mathcal M_A)\)\;
}
\(t_0\gets\textsc{Now}()\); \(t_{\mathrm{end}}\gets t_0+\beta\)\;
\(\Xi_{B^\star\to A} \gets \textsc{Explore}(A,B^\star,\Xi^A_{\mathrm{base}},\mathsf{Prov}_{B^\star\to A},t_0+\beta/2;\mathcal M_A)\)\;
\(\Xi_{A\to A} \gets \textsc{Explore}(A,A,\Xi^A_{\mathrm{base}},\mathsf{Prov}_{A\to A},t_{\mathrm{end}};\mathcal M_A)\)\;
\(\Xi_A^{\mathrm{iso}} \gets \Xi^A_{\mathrm{base}} \cup \Xi_{B^\star\to A} \cup \Xi_{A\to A}\)\;
\tcp{Complete isolated checks outside the search deadline.}
\ForEach{\(\xi \in \Xi_A^{\mathrm{iso}}\)}{
  \(\textsc{ExecuteAndCheck}(A,B^\star,\xi;\mathcal M_A)\)\;
}
\tcp{Phase 2: construct and evaluate continuous workloads.}
\(\Xi_A^{\mathrm{cont}}\gets\textsc{Continuous}(\Xi_A^{\mathrm{iso}})\)\;
\ForEach{\(\xi\in\Xi_A^{\mathrm{cont}}\)}{
  \(\textsc{ExecuteAndCheck}(A,B^\star,\xi;\mathcal M_A)\)\;
}
\tcp{Aggregate completed records from both phases.}
\(\Xi_A\gets\Xi_A^{\mathrm{iso}}\cup\Xi_A^{\mathrm{cont}}\)\;
\(\mathcal E_A\gets\{(\xi,e_A(\xi))\mid\xi\in\Xi_A\}\) from \(\mathcal M_A\)\;
\Return \(\bigl(\Xi_A,\operatorname{score}_A(\Xi_A),\mathcal E_A\bigr)\)\;
\end{algorithm}

\paragraph{Provider instantiations.}
Behavior and RTL instantiate \(\operatorname{Select}_{T\to A}\) and \(\mathcal H_{T\to A}\)
from Definition~\ref{def:probe}. They share the solve-and-update step using
Z3~\citep{demoura2008z3}.

\paragraph{Behavior provider.}
The Behavior provider handles \(B^\star\to\hat B\),
\(B^\star\to\hat R\) and \(\hat B\to\hat B\). For \(B^\star\to A\),
the provider extracts the accepted input trace
\(\mathbf a_A(\xi)\) from \(A\)'s execution record and collects goal hits
and local model states from \(B^\star\) running on
\(\pi(\sigma,\mathbf a_A(\xi))\). For \(\hat B\to\hat B\), the same
information comes directly from \(\hat B\). Local queries use these states and
retain the transaction fields' source positions in \(\mathbf i\).
In each equation, \(\Phi_q\) is the symbolic execution constraint over
replacement values \(u_F\), and \(\psi_{\gamma,q}\) is the Boolean condition
that the execution hits \(\gamma\in G(q)\). The hit relations are
\[
  \mathcal H_{B^\star\to A}(\xi,q,u_F,\mathbf h_q)
  \equiv
  \Phi_q^{B^\star\to A}(u_F)
  \;\wedge\;
  \bigwedge_{\gamma\in G(q)}
  \bigl(h_\gamma\leftrightarrow
        \psi_{\gamma,q}^{B^\star\to A}(u_F)\bigr),
  \qquad A\in\{\hat B,\hat R\},
\]
and
\[
  \mathcal H_{\hat B\to\hat B}(\xi,q,u_F,\mathbf h_q)
  \equiv
  \Phi_q^{\hat B\to\hat B}(u_F)
  \;\wedge\;
  \bigwedge_{\gamma\in G(q)}
  \bigl(h_\gamma\leftrightarrow
        \psi_{\gamma,q}^{\hat B\to\hat B}(u_F)\bigr).
\]

\paragraph{RTL provider.}
An enabled RTL provider handles \(\hat R\to\hat R\).
The seed's execution record supplies \((\mathbf i,\sigma)\) and, at each
schedule boundary \(k\), the state
\(s_k^{\mathrm{dyn}}=(r_k,\eta_k,z_k)\), comprising adapter bookkeeping,
RTL state and monitor state.
Each request \(q\) selects indices \(0\leq k\leq\ell<|\sigma|\), the replayed
state \(s_k^{\mathrm{dyn}}\) before \(\sigma_{k:\ell}\), target goals
\(\emptyset\neq G(q)\subseteq U\) and mutable fields
\(F(q)\) first presented in that segment. The query varies only \(F(q)\),
fixes the schedule and every other field, and requires the modified execution
prefix to reproduce \(s_k^{\mathrm{dyn}}\) at boundary \(k\).
The input queues \(\mathbf i\) are excluded from this state constraint.
Let
\(\Phi_q^{\hat R\to\hat R}(u_F)\) be the Yosys-generated RTL segment
constraint~\citep{wolf2013yosys}, composed with the protocol adapter
and finite monitors, with internal execution variables existentially
quantified. Let \(\psi_{\gamma,q}^{\hat R\to\hat R}(u_F)\)
encode the catalogue hit rule for \(\gamma\in G(q)\) over the
segment. For source goals, this requires a source-event encoding in addition
to the circuit constraint. The RTL hit relation is
\[
  \mathcal H_{\hat R\to\hat R}(\xi,q,u_F,\mathbf h_q)
  \equiv
  \Phi_q^{\hat R\to\hat R}(u_F)
  \;\wedge\;
  \bigwedge_{\gamma\in G(q)}
  \bigl(h_\gamma\leftrightarrow
        \psi_{\gamma,q}^{\hat R\to\hat R}(u_F)\bigr).
\]


\section{\behavesim: Formal Semantics and Bounded Guarantees}
\label{app:behave-guarantees}

\paragraph{Purpose.}
Appendix~\ref{app:behave-system} evaluates artifacts on a finite
stimulus pool. We add optional reachability analysis over a fixed
bounded domain. Given an exact Behavior or RTL instance, we use replay-audited
traces and solver queries to classify coverage and safety goals as
reached, bounded-unreachable or unresolved. These statuses supplement
evaluation without changing its stimulus set or score.

\paragraph{Overview.}
Subsections~\ref{app:bs-generic} through~\ref{app:bs-closure} define the common
formalism for bounded trace domains, finite representations, exact encodings,
replay audits and classification. Subsections~\ref{app:bs-signal}
through~\ref{app:bs-rtl} and Subsections~\ref{app:bs-ir}
through~\ref{app:bs-assumptions} construct the RTL and Behavior instances.
Subsection~\ref{app:bs-artifact-reporting} pairs each instance with its evaluated
pool and states the reported conclusions. The final subsection states the
trusted boundary, claim scope and conclusion. Figure~\ref{fig:formal-handoff}
summarizes these dependencies.

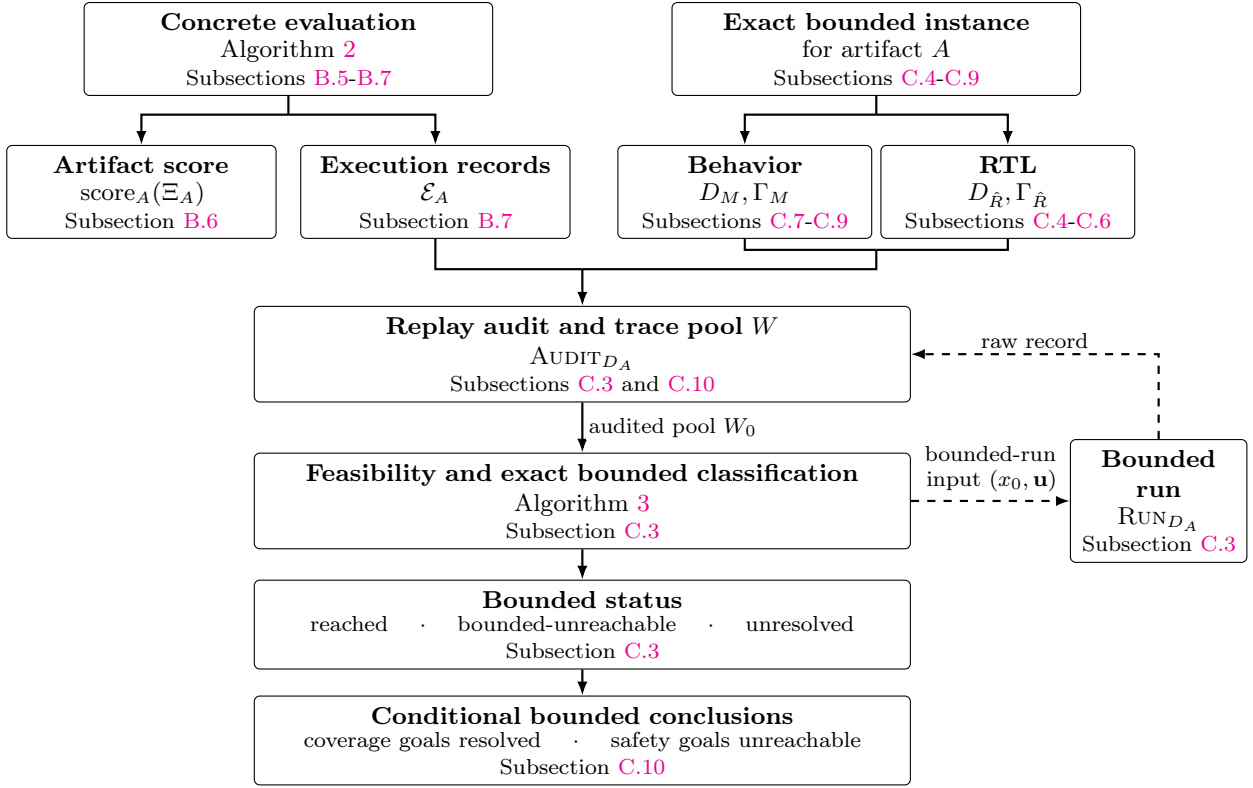
\begin{figure}[!t]
\centering
\footnotesize
\resizebox{\linewidth}{!}{%
\begin{tikzpicture}[
  box/.style={draw, rounded corners=2pt, align=center, inner sep=4pt,
              text width=#1, font=\footnotesize},
  stage/.style={box=5.0cm},
  small/.style={box=3.2cm},
  provider/.style={box=3.0cm},
  wide/.style={box=8.2cm},
  runbox/.style={box=2.0cm},
  label/.style={font=\scriptsize, fill=white, inner sep=1pt, align=center},
  flowline/.style={thick},
  flowarrow/.style={-{Latex[length=1.8mm]}, thick},
  looparrow/.style={-{Latex[length=1.8mm]}, thick, dashed}
]

\node[stage] (concrete) at (-3.8,0)
  {\textbf{Concrete evaluation}\\[-1pt]
   Algorithm~\ref{alg:twosided}\\[-1pt]\scriptsize
   Subsections~\ref{app:bs-setting}-\ref{app:bs-generation}};
\node[stage] (instance) at (3.8,0)
  {\textbf{Exact bounded instance}\\[-1pt]
   for artifact \(A\)\\[-1pt]\scriptsize
   Subsections~\ref{app:bs-signal}-\ref{app:bs-assumptions}};

\node[small] (score) at (-5.7,-1.85)
  {\textbf{Artifact score}\\[-1pt]
   \(\operatorname{score}_A(\Xi_A)\)\\[-1pt]\scriptsize
   Subsection~\ref{app:bs-goals}};
\node[small] (records) at (-1.9,-1.85)
  {\textbf{Execution records}\\[-1pt]
   \(\mathcal E_A\)\\[-1pt]\scriptsize
   Subsection~\ref{app:bs-generation}};
\node[provider] (behavior) at (2.1,-1.85)
  {\textbf{Behavior}\\[-1pt] \(D_M,\Gamma_M\)\\[-1pt]\scriptsize
   Subsections~\ref{app:bs-ir}-\ref{app:bs-assumptions}};
\node[provider] (rtl) at (5.5,-1.85)
  {\textbf{RTL}\\[-1pt] \(D_{\hat R},\Gamma_{\hat R}\)\\[-1pt]\scriptsize
   Subsections~\ref{app:bs-signal}-\ref{app:bs-rtl}};

\coordinate (concretefork) at (-3.8,-0.85);
\draw[flowline] (concrete.south) -- (concretefork);
\draw[flowarrow] (concretefork) -| (score.north);
\draw[flowarrow] (concretefork) -| (records.north);

\coordinate (instancefork) at (3.8,-0.85);
\draw[flowline] (instance.south) -- (instancefork);
\draw[flowarrow] (instancefork) -| (behavior.north);
\draw[flowarrow] (instancefork) -| (rtl.north);

\coordinate (instancejoin) at (3.8,-2.60);
\draw[flowline] (behavior.south) |- (instancejoin);
\draw[flowline] (rtl.south) |- (instancejoin);

\node[wide] (audit) at (0,-3.95)
  {\textbf{Replay audit and trace pool \(W\)}\\
   \(\textsc{Audit}_{D_A}\)\\[-1pt]\scriptsize
   Subsections~\ref{app:bs-closure} and~\ref{app:bs-artifact-reporting}};
\coordinate (auditjoin) at (0,-2.85);
\draw[flowline] (records.south) |- (auditjoin);
\draw[flowline] (instancejoin) |- (auditjoin);
\draw[flowarrow] (auditjoin) -- (audit.north);

\node[wide] (classifier) at (0,-5.85)
  {\textbf{Feasibility and exact bounded classification}\\
   Algorithm~\ref{alg:closure}\\[-1pt]\scriptsize
   Subsection~\ref{app:bs-closure}};
\node[runbox] (run) at (7.45,-5.85)
  {\textbf{Bounded run}\\[-1pt]$\textsc{Run}_{D_A}$\\[-1pt]
   \scriptsize Subsection~\ref{app:bs-closure}};
\draw[flowarrow] (audit) -- (classifier)
  node[midway, right=1pt, label] {audited pool \(W_0\)};
\draw[looparrow] (classifier.east) -- (run.west)
  node[midway, above=2pt, label, text width=1.7cm]
    {bounded-run\\input \((x_0,\mathbf u)\)};
\coordinate (runreturn) at (7.45,-3.95);
\draw[looparrow] (run.north) -- (runreturn) -- (audit.east);
\path (runreturn) -- (audit.east)
  node[midway, above=1pt, label] {raw record};

\node[wide] (status) at (0,-7.45)
  {\textbf{Bounded status}\\[-1pt]\scriptsize
   reached \quad $\cdot$ \quad bounded-unreachable \quad $\cdot$ \quad unresolved
   \\[-1pt]Subsection~\ref{app:bs-closure}};
\node[wide] (claims) at (0,-8.95)
  {\textbf{Conditional bounded conclusions}\\[-1pt]\scriptsize
   coverage goals resolved \quad $\cdot$ \quad safety goals unreachable
   \\[-1pt]Subsection~\ref{app:bs-artifact-reporting}};
\draw[flowarrow] (classifier) -- (status);
\draw[flowarrow] (status) -- (claims);

\end{tikzpicture}
}
\caption{\textbf{Concrete evaluation and bounded analysis.} Concrete
evaluation returns the artifact score and raw records. With an exact
instance, replay-audited traces initialize \(W_0\). The dashed arrows show the
validation path for a bounded-run input: artifact replay produces a raw
record, which must pass audit before it may update \(W\). This path is used for
feasibility initialization and every decoded solver witness.
Behavior statuses are mapped back to source goals before reporting.
Each bounded conclusion is reported only when its premises hold.
Classification does not change the artifact score.}
\label{fig:formal-handoff}
\end{figure}

\subsection{Trace Domains and Goal Reachability}
\label{app:bs-generic}

\paragraph{Overview.}
The analysis considers executions only up to a fixed bound and under fixed task
constraints. This subsection gives Behavior and RTL a common semantics for
those executions, defines when they reach a declared goal, and embeds declared
failures as terminal states.

\begin{definition}[Instrumented semantics]
\label{def:system}
An \emph{instrumented semantics} is
\[
  \mathcal M=(X,U,Y,I,\delta,\Gamma),
\]
where \(X\) is the state space, \(\emptyset\neq I\subseteq X\) is the
initial-state set, \(U\) and \(Y\) are the input and output sets, and
\(\Gamma\) is the full finite goal catalogue. Its total step
function is
\[
  \delta:X\times U\longrightarrow X\times Y\times 2^\Gamma.
\]
Writing \(\delta(x,u)=(x',y,H)\), a step from \(x\) under input \(u\)
moves to \(x'\), emits \(y\), and hits exactly the goals in \(H\). The Behavior
and RTL instances define separately what counts as one step. In both instances,
\(X\) also contains any finite adapter, harness and monitor state, and \(H\) records
all goal hits produced during the step. If early termination is represented,
the instantiation also designates terminal states
\(X^{\mathrm h}\subseteq X\) and a padding output \(\bot_Y\in Y\) such
that
\[
  \delta(x_{\mathrm h},u)
  =(x_{\mathrm h},\bot_Y,\emptyset)
  \qquad\text{for every }x_{\mathrm h}\in X^{\mathrm h},\ u\in U.
\]
Early termination is therefore padded to the horizon by stuttering steps.
\end{definition}

\begin{definition}[Fallible step and absorption]
\label{def:absorb}
Let \(X,U,Y\) and \(\mathit{Err}\) be sets, let \(\Gamma\) be finite, and fix
\(\emptyset\neq I^0\subseteq X\). A \emph{fallible step} is a total
function
\[
  \delta^0:X\times U\longrightarrow
  \bigl(X\times Y\times2^\Gamma\bigr)
  \cup
  \{\mathsf{Err}(\varepsilon,H)
      \mid \varepsilon\in\mathit{Err},\ H\subseteq\Gamma\}.
\]
Its result is either a completed step \((x',y,H)\) or a failure
\(\mathsf{Err}(\varepsilon,H)\). Regard the errors as terminal states disjoint
from \(X\), renaming them if necessary. For a fresh \(\bot_Y\notin Y\), define
\[
\begin{aligned}
  \bar X&=X\cup\mathit{Err},\\
  \bar Y&=Y\cup\{\bot_Y\}.
\end{aligned}
\]
The \emph{absorption} of \(\delta^0\) at \(I^0\) is
\[
  \mathrm{Abs}_{I^0}(\delta^0)
  =(\bar X,U,\bar Y,I^0,\bar\delta,\Gamma),
\]
where
\[
  \bar\delta(x,u)=
  \begin{cases}
    (x',y,H)
      & \delta^0(x,u)=(x',y,H),\\
    (\varepsilon,\bot_Y,H)
      & \delta^0(x,u)=\mathsf{Err}(\varepsilon,H),
  \end{cases}
\]
and
\[
  \bar\delta(\varepsilon,u)
  =(\varepsilon,\bot_Y,\emptyset)
  \qquad(\varepsilon\in\mathit{Err}).
\]
Thus this is an instrumented semantics of Definition~\ref{def:system} with
initial set \(I^0\) and terminal set \(\mathit{Err}\).
\end{definition}

\begin{definition}[Bounded trace domain and reachability]
\label{def:domain}
Write \(\mathsf{Bool}=\{\mathsf{false},\mathsf{true}\}\). To define bounded
reachability, fix a horizon \(N\in\mathbb N_{>0}\) and an admissibility
predicate
\[
  K_N:X^{N+1}\times U^N\times Y^N
  \longrightarrow\mathsf{Bool}
\]
that holds exactly for the state, input and output sequences allowed by the
fixed task and artifact configuration. The horizon and admissibility predicate
are fixed before artifact replay or solver analysis. The \emph{bounded trace
domain} is
\[
  D=(\mathcal M,N,K_N).
\]
The set of admitted \(N\)-step traces is
\[
\begin{aligned}
  \mathrm{Tr}(D):=\bigl\{(\mathbf x,\mathbf u,\mathbf y,\mathbf H)\ \bigm|\;&
  \mathbf x\in X^{N+1},\ \mathbf u\in U^N,\
  \mathbf y\in Y^N,\ \mathbf H\in(2^\Gamma)^N,\\
  &x_0\in I,\quad
  K_N(\mathbf x,\mathbf u,\mathbf y)=\mathsf{true},\\
  &\delta(x_t,u_t)=(x_{t+1},y_t,H_t)
  \text{ for every }t<N\bigr\}.
\end{aligned}
\]
Write \(w=(\mathbf x,\mathbf u,\mathbf y,\mathbf H)\) for an admitted trace in
\(\mathrm{Tr}(D)\). It contains \(N+1\) states and one input, output and hit set
per step. Define
\[
  \mathrm{Hit}(w):=\bigcup_{t<N}H_t,
  \qquad
  \mathrm{Hit}(W):=\bigcup_{w\in W}\mathrm{Hit}(w)
  \quad\text{for }W\subseteq\mathrm{Tr}(D).
\]
Here \(\mathrm{Hit}\) takes a trace or a set of traces,
whereas \(\mathrm{Hit}_T(\Xi)\) in Appendix~\ref{app:bs-goals} takes a
stimulus set. The target is fixed by \(\mathcal M\), so we omit its
subscript.
The goals reachable within \(D\) are
\[
  \mathrm{Reach}_D
  :=\{\gamma\in\Gamma\mid
      \exists(\mathbf x,\mathbf u,\mathbf y,\mathbf H)\in\mathrm{Tr}(D),
      \ \exists t<N\text{ such that }\gamma\in H_t\}.
\]
An admitted trace \(w\) is a \emph{witness} for \(\gamma\) when
\(\gamma\in\mathrm{Hit}(w)\).
The domain \(D\) is \emph{feasible} when \(\mathrm{Tr}(D)\neq\emptyset\).
\end{definition}

\subsection{Finite Representations and Exact Encodings}
\label{app:bs-representation}

\paragraph{Overview.}
Exact bounded analysis applies only when the fixed trace domain has a finite
layered representation. This subsection defines that representation and its
exact bounded encoding. Lemma~\ref{lem:finite-support} gives a sufficient
condition for finite layers to exist. The RTL and Behavior instances below
construct their carriers explicitly.

\begin{definition}[Finite layered representation]
\label{def:finite-representation}
A \emph{finite layered representation} of \(D\) consists of finite state sets
\(X_t^D\subseteq X\) for \(t\le N\), finite input sets
\(U_t^D\subseteq U\) for \(t<N\), and finite output sets
\(Y_t^D\subseteq Y\) for \(t<N\). It also supplies a finite initial carrier
\(I_D\) satisfying
\[
  \{x_0\mid
      (\mathbf x,\mathbf u,\mathbf y,\mathbf H)\in\mathrm{Tr}(D)\}
  \subseteq I_D\subseteq I\cap X_0^D.
\]
The remaining trace-support conditions are
\[
  x_t\in X_t^D\quad(t\le N),
  \qquad
  u_t\in U_t^D,\quad y_t\in Y_t^D\quad(t<N)
\]
for every
\((\mathbf x,\mathbf u,\mathbf y,\mathbf H)\in\mathrm{Tr}(D)\).
The sets must also satisfy the layer-closure condition
\[
  \delta(x,u)\in X_{t+1}^D\times Y_t^D\times2^\Gamma
  \quad\text{for every }x\in X_t^D,\ u\in U_t^D,\ t<N.
\]
The corresponding finite layered system is
\[
  \mathcal A_D
  =
  \big((X_t^D)_{t\le N},(U_t^D)_{t<N},(Y_t^D)_{t<N},I_D,
        (\delta\restriction(X_t^D\times U_t^D))_{t<N},
        \Gamma,K_N\big).
\]
Here \(f\restriction S\) denotes the restriction of a function \(f\) to
inputs in \(S\).
Its trace set is
\[
\begin{aligned}
  \mathrm{Tr}(\mathcal A_D)
  :=\bigl\{(\mathbf x,\mathbf u,\mathbf y,\mathbf H)\ \bigm|\;&
      \mathbf x\in\prod_{t\le N}X_t^D,\quad
      \mathbf u\in\prod_{t<N}U_t^D,\\
    & \mathbf y\in\prod_{t<N}Y_t^D,\quad
      \mathbf H\in(2^\Gamma)^N,\quad x_0\in I_D,\\
    & K_N(\mathbf x,\mathbf u,\mathbf y)=\mathsf{true},\quad
      \delta(x_t,u_t)=(x_{t+1},y_t,H_t)\ (t<N)\bigr\}.
\end{aligned}
\]
If no such representation can be constructed, bounded analysis of \(D\)
returns \(\mathsf{Refused}\).
\end{definition}

\begin{lemma}[Trace preservation]
\label{lem:trace-preservation}
Every finite layered representation of \(D\) satisfies
\[
  \mathrm{Tr}(\mathcal A_D)=\mathrm{Tr}(D).
\]
Consequently, its reachable-goal set is \(\mathrm{Reach}_D\).
\end{lemma}

\begin{proof}
Fix a finite layered representation of \(D\), and let
\(w=(\mathbf x,\mathbf u,\mathbf y,\mathbf H)\). If
\(w\in\mathrm{Tr}(\mathcal A_D)\), then the layer-membership conditions,
\(x_0\in I_D\subseteq I\), and the shared admissibility and transition
conditions give
\[
  w\in\mathrm{Tr}(\mathcal A_D)\Longrightarrow w\in\mathrm{Tr}(D),
  \qquad
  \mathrm{Tr}(\mathcal A_D)\subseteq\mathrm{Tr}(D).
\]
Conversely, let \(w\in\mathrm{Tr}(D)\). The trace-support conditions of the
fixed finite layered representation give
\[
  x_t\in X_t^D\ (t\le N),
  \qquad
  u_t\in U_t^D,\quad y_t\in Y_t^D\ (t<N),
\]
while the initial-carrier condition gives \(x_0\in I_D\). The admissibility and
transition conditions already hold, so
\[
  w\in\mathrm{Tr}(D)\Longrightarrow w\in\mathrm{Tr}(\mathcal A_D),
  \qquad
  \mathrm{Tr}(D)\subseteq\mathrm{Tr}(\mathcal A_D).
\]
The two inclusions prove equality. By Definition~\ref{def:domain}, the
reachable-goal set is therefore
\[
  \bigl\{\gamma\in\Gamma\bigm|
    \exists(\mathbf x,\mathbf u,\mathbf y,\mathbf H)
      \in\mathrm{Tr}(\mathcal A_D),\
    \ \exists t<N:\gamma\in H_t\bigr\}
  =\mathrm{Reach}_D.\qedhere
\]
\end{proof}

\begin{lemma}[Finite-layer construction]
\label{lem:finite-support}
Suppose the set of initial states occurring in \(\mathrm{Tr}(D)\) is finite
and, at every \(t<N\), only finitely many values occur as \(u_t\) among traces
in \(\mathrm{Tr}(D)\). Then \(D\) admits a finite layered representation.
\end{lemma}

\begin{proof}
Let \(I_D=X_0^D\) be the set of initial states occurring in
\(\mathrm{Tr}(D)\), and set
\[
  U_t^D
  =\{u_t\mid
      (\mathbf x,\mathbf u,\mathbf y,\mathbf H)\in\mathrm{Tr}(D)\}
  \qquad(t<N).
\]
Definition~\ref{def:domain} gives \(I_D\subseteq I\), so the initial-carrier
condition holds.
Each \(U_t^D\) is finite by the hypothesis of the lemma.
For \(t=0,\ldots,N-1\) in order, apply \(\delta\) to every pair in
\(X_t^D\times U_t^D\). Let \(X_{t+1}^D\) collect the resulting next states
and \(Y_t^D\) the resulting outputs. We show by induction that \(X_t^D\) is
finite and contains every \(x_t\) occurring in an admitted trace. The base
case is \(X_0^D=I_D\). For the induction step, \(X_t^D\times U_t^D\) is
finite, so its image under \(\delta\), and hence \(X_{t+1}^D\) and
\(Y_t^D\), is finite. For any admitted trace, \(u_t\in U_t^D\). Its
transition equation and the induction hypothesis then give
\(x_{t+1}\in X_{t+1}^D\) and \(y_t\in Y_t^D\). This establishes trace
support. Finally, because the construction includes every
\(x\in X_t^D\) and \(u\in U_t^D\), their result satisfies
\(\delta(x,u)\in X_{t+1}^D\times Y_t^D\times2^\Gamma\) for every
\(t<N\). This is the layer-closure condition.
\end{proof}

\begin{definition}[Exact bounded encoding]
\label{def:exact-enc}
An \emph{exact bounded encoding} translates the finite layered system
\(\mathcal A_D\) into solver formulas without adding or removing an initial
state, admitted sequence, transition or goal hit. Let
\(\mathbb B=\{0,1\}\). The encoding first supplies injective maps
\[
  \bigl(\mathrm{enc}^X_t:X_t^D\to\mathbb B^{d_t^X}\bigr)_{t\le N},
  \qquad
  \bigl(\mathrm{enc}^U_t:U_t^D\to\mathbb B^{d_t^U}\bigr)_{t<N},
  \qquad
  \bigl(\mathrm{enc}^Y_t:Y_t^D\to\mathbb B^{d_t^Y}\bigr)_{t<N},
\]
where the exponents are positive bit widths. Each \(\mathrm{enc}\) map converts
a semantic value into a bit vector. Write \(\mathrm{dec}^X_t\),
\(\mathrm{dec}^U_t\) and \(\mathrm{dec}^Y_t\) for the respective inverse maps
on bit vectors produced by these encodings. Since not every bit vector need
encode a value, the following Boolean predicates recognize exactly the valid
encodings at each trace position:
\[
\begin{aligned}
  \mathrm{Valid}^X_t(\widehat x)
  &\iff \widehat x\in
    \{\mathrm{enc}^X_t(x)\mid x\in X_t^D\},\\
  \mathrm{Valid}^U_t(\widehat u)
  &\iff \widehat u\in
    \{\mathrm{enc}^U_t(u)\mid u\in U_t^D\},\\
  \mathrm{Valid}^Y_t(\widehat y)
  &\iff \widehat y\in
    \{\mathrm{enc}^Y_t(y)\mid y\in Y_t^D\}.
\end{aligned}
\]
The encoding also supplies Boolean predicates for the semantic conditions:
\(\widehat I_D\) tests initial-state
membership, \(\widehat K_N\) tests sequence admissibility, and
\(\widehat{\delta}_t^D\) tests whether an encoded tuple represents a transition
at position \(t<N\). On every
\(\widehat x\) satisfying \(\mathrm{Valid}^X_0\),
\[
  \widehat I_D(\widehat x)
  \iff \mathrm{dec}^X_0(\widehat x)\in I_D.
\]
Hence \(\mathrm{Valid}^X_0\) recognizes every encoded state in \(X_0^D\),
whereas \(\widehat I_D\) selects those whose decoded states belong to
\(I_D\).
For encoded sequences
\[
  \widehat{\mathbf x}=(\widehat x_t)_{t\le N},
  \qquad
  \widehat{\mathbf u}=(\widehat u_t)_{t<N},
  \qquad
  \widehat{\mathbf y}=(\widehat y_t)_{t<N}
\]
whose components satisfy the corresponding position-indexed validity
predicates, namely \(\mathrm{Valid}^X_t(\widehat x_t)\) for \(t\le N\) and
\(\mathrm{Valid}^U_t(\widehat u_t)\) and
\(\mathrm{Valid}^Y_t(\widehat y_t)\) for \(t<N\),
\[
\begin{split}
  \widehat K_N(\widehat{\mathbf x},\widehat{\mathbf u},\widehat{\mathbf y})
  \iff{}&
  K_N\Bigl(
    \bigl(\mathrm{dec}^X_t(\widehat x_t)\bigr)_{t\le N},
    \bigl(\mathrm{dec}^U_t(\widehat u_t)\bigr)_{t<N},\\
    &\hspace{34mm}
    \bigl(\mathrm{dec}^Y_t(\widehat y_t)\bigr)_{t<N}
  \Bigr).
\end{split}
\]
For every \(t<N\), every
\(\mathbf h=(h_\gamma)_{\gamma\in\Gamma}
\in\{\mathsf{false},\mathsf{true}\}^{\Gamma}\), and every
\((\widehat x,\widehat u,\widehat x',\widehat y)\) satisfying
\(\mathrm{Valid}^X_t\), \(\mathrm{Valid}^U_t\),
\(\mathrm{Valid}^X_{t+1}\) and \(\mathrm{Valid}^Y_t\), respectively,
\[
\begin{split}
  \widehat{\delta}_t^D
  (\widehat x,\widehat u,\widehat x',\widehat y,\mathbf h)
  \iff{}&
  \delta\bigl(\mathrm{dec}^X_t(\widehat x),
                 \mathrm{dec}^U_t(\widehat u)\bigr)\\
  ={}&\bigl(\mathrm{dec}^X_{t+1}(\widehat x'),
            \mathrm{dec}^Y_t(\widehat y),
            \{\gamma\in\Gamma\mid h_\gamma=\mathsf{true}\}\bigr).
\end{split}
\]
\end{definition}

\subsection{Bounded Reachability Analysis}
\label{app:bs-closure}

\paragraph{Overview.}
Subsection~\ref{app:bs-generation} evaluates artifacts on a
finite stimulus pool. Here, exact queries analyze goal reachability over the
full bounded domain \(\mathrm{Tr}(D)\).
The goal set selected for classification is \(G\subseteq\Gamma\), and the
full-catalogue analysis takes \(G=\Gamma\).
Solver answers and replay audits
classify each goal as reached, bounded-unreachable or unresolved.

\begin{definition}[Bounded queries]
\label{def:bounded-query}
Fix a finite layered representation of \(D\) and an exact bounded encoding of
\(\mathcal A_D\). For a solver query, let
\(\widehat{\mathbf x},\widehat{\mathbf u},\widehat{\mathbf y}\) denote its
encoded state, input and output variables. For \(t<N\), let
\(\mathbf h_t=(h_{t,\gamma})_{\gamma\in\Gamma}\) be fresh Boolean variables
encoding the hit set at step \(t\).
The base formula encoding an admitted trace in \(D\) is
\[
\begin{aligned}
  F_D ={}&
  \widehat I_D(\widehat x_0)
  \wedge \bigwedge_{t\le N}\mathrm{Valid}^X_t(\widehat x_t) \\
  &\wedge \bigwedge_{t<N}
    \bigl(\mathrm{Valid}^U_t(\widehat u_t)
          \wedge\mathrm{Valid}^Y_t(\widehat y_t)\bigr) \\
  &\wedge \bigwedge_{t<N}
    \widehat{\delta}_t^D(\widehat x_t,\widehat u_t,
              \widehat x_{t+1},\widehat y_t,\mathbf h_t)
  \wedge \widehat K_N(\widehat{\mathbf x},
                       \widehat{\mathbf u},
                       \widehat{\mathbf y}).
\end{aligned}
\]
For a target goal set \(S\subseteq\Gamma\), define
\[
  \Phi_{D,S}
  =F_D\wedge
    \bigvee_{\gamma\in S}\ \bigvee_{t<N}h_{t,\gamma}.
\]
The single-goal query is \(\Phi_{D,\{\gamma\}}\).
\end{definition}

\begin{theorem}[Bounded-query correctness]
\label{thm:bounded}
Given a finite layered representation of \(D\) and an exact bounded encoding
of \(\mathcal A_D\),
\[
  F_D\text{ is satisfiable}
  \quad\Longleftrightarrow\quad
  \mathrm{Tr}(D)\neq\emptyset.
\]
Thus \(F_D\) is satisfiable exactly when \(D\) is feasible. For every
\(S\subseteq\Gamma\),
\[
  \Phi_{D,S}\text{ is satisfiable}
  \quad\Longleftrightarrow\quad
  S\cap\mathrm{Reach}_D\neq\emptyset.
\]
Every satisfying assignment to \(F_D\) decodes to a trace in
\(\mathrm{Tr}(D)\). Every satisfying assignment to \(\Phi_{D,S}\) decodes
to one that hits a goal in \(S\).
\end{theorem}

\begin{proof}
For \(w=(\mathbf x,\mathbf u,\mathbf y,\mathbf H)\in\mathrm{Tr}(D)\),
Lemma~\ref{lem:trace-preservation} gives
\(w\in\mathrm{Tr}(\mathcal A_D)\). Define an assignment \(\nu_w\) by
\[
\begin{aligned}
  \nu_w(\widehat x_t)&=\mathrm{enc}^X_t(x_t) &&(t\le N),\\
  \nu_w(\widehat u_t)&=\mathrm{enc}^U_t(u_t),\qquad
  \nu_w(\widehat y_t)=\mathrm{enc}^Y_t(y_t) &&(t<N),\\
  \nu_w(h_{t,\gamma})&=\mathsf{true}
    \quad\Longleftrightarrow\quad \gamma\in H_t
    &&(t<N,\ \gamma\in\Gamma).
\end{aligned}
\]
Exactness of the encoding gives \(\nu_w\models F_D\).
Conversely, let \(\nu\models F_D\). The validity predicates permit every
encoded component to be decoded. Set
\[
  H_t^\nu
  =\{\gamma\in\Gamma\mid
      \nu(h_{t,\gamma})=\mathsf{true}\}.
\]
The initial-state, transition and admissibility conjuncts of \(F_D\) give a
decoded trace \(w_\nu\in\mathrm{Tr}(\mathcal A_D)\). By
Lemma~\ref{lem:trace-preservation}, \(w_\nu\in\mathrm{Tr}(D)\). Therefore
\[
  F_D\text{ is satisfiable}
  \quad\Longleftrightarrow\quad
  \mathrm{Tr}(D)\neq\emptyset.
\]
Finally, expanding the hit disjunction and then decoding gives
\[
\begin{aligned}
  \Phi_{D,S}\text{ is satisfiable}
  &\Longleftrightarrow
    \exists\nu,\ \exists t<N,\ \exists\gamma\in S:
    \nu\models F_D\ \wedge\
    \nu(h_{t,\gamma})=\mathsf{true}\\
  &\Longleftrightarrow
    \exists(\mathbf x,\mathbf u,\mathbf y,\mathbf H)\in\mathrm{Tr}(D),\
    \ \exists t<N,\ \exists\gamma\in S:\gamma\in H_t\\
  &\Longleftrightarrow S\cap\mathrm{Reach}_D\neq\emptyset.
\end{aligned}
\]
The constructions of \(\nu_w\) and \(w_\nu\) also prove the decoding claims.
\end{proof}

\begin{definition}[Replay audit and evaluated pool]
\label{def:replay-audit}
Fix a generated artifact \(A\) together with its analysis domain \(D\), a
finite layered representation and an exact bounded encoding. The replay audit
is
\[
  \textsc{Audit}_D:
  \mathcal R_A\longrightarrow
  \mathrm{Tr}(D)\cup
  \{\mathsf{Outside},\mathsf{Inconsistent},\mathsf{Refused}\}.
\]
For a raw record \(e\in\mathcal R_A\), return
\(\mathsf{Refused}\) if it is incomplete or lacks a valid domain identifier.
For a record identifying \(D\), reconstruct
\(w=(\mathbf x,\mathbf u,\mathbf y,\mathbf H)\) and require exactly
\(N+1\) states and \(N\) inputs, outputs and hit sets, or
return \(\mathsf{Refused}\). Records are not truncated or padded.
Complete records of artifact failures are audited normally.
Return \(\mathsf{Outside}\) if the record belongs to another domain or,
for a record identifying \(D\), a state, input or output lies outside its
position-specific carrier, \(x_0\notin I_D\), or
\(K_N(\mathbf x,\mathbf u,\mathbf y)=\mathsf{false}\). If the record is not
outside \(D\), require \(H_t\subseteq\Gamma\) for every \(t<N\) and define
\[
\begin{aligned}
  \widehat x_t^w&:=\mathrm{enc}_t^X(x_t) &&(t\le N),\\
  \widehat u_t^w&:=\mathrm{enc}_t^U(u_t),\qquad
  \widehat y_t^w:=\mathrm{enc}_t^Y(y_t) &&(t<N).
\end{aligned}
\]
Let \(\mathbf h_t^w=(h_{t,\gamma}^w)_{\gamma\in\Gamma}\), where
\(h_{t,\gamma}^w=\mathsf{true}\) iff \(\gamma\in H_t\). Each step must satisfy
\[
  \widehat{\delta}_t^D(\widehat x^w_t,\widehat u^w_t,
            \widehat x^w_{t+1},\widehat y^w_t,\mathbf h^w_t)
  \qquad(t<N).
\]
A failed requirement returns \(\mathsf{Inconsistent}\). Otherwise exactness
gives \(w\in\mathrm{Tr}(D)\), and the audit returns \(w\). Audit every record
in \(\mathcal E_A\). Any \(\mathsf{Inconsistent}\) or \(\mathsf{Refused}\)
result stops bounded classification, while \(\mathsf{Outside}\) records are
discarded. If no audit stops classification, define
\[
  W_{\mathrm{eval}}^A
  :=\bigl\{w\in\mathrm{Tr}(D)\bigm|
      \textsc{Audit}_D(e)=w
      \text{ for some }(\xi,e)\in\mathcal E_A\bigr\}.
\]
\end{definition}

\paragraph{Bounded-run interface.}
Each artifact supplies
\[
  \textsc{Run}_D:
  I_D\times\prod_{t<N}U_t^D
  \longrightarrow\mathcal R_A\cup\{\mathsf{Refused}\},
\]
where \(\textsc{Run}_D(x_0,\mathbf u)\) realizes \(x_0\) as an
artifact-specific concrete state, runs \(A\) for \(N\) steps under
\(\mathbf u\), and returns the resulting raw record. It returns
\(\mathsf{Refused}\) if no complete raw record can be produced.

\paragraph{Initial audited pool.}
After the preceding audit processes every evaluation record without returning
\(\mathsf{Refused}\) or \(\mathsf{Inconsistent}\),
if \(W_{\mathrm{eval}}^A\neq\emptyset\), set
\(W_0=W_{\mathrm{eval}}^A\). Otherwise,
solve \(F_D\). By Theorem~\ref{thm:bounded}, an unsatisfiable result establishes
\(\mathrm{Tr}(D)=\emptyset\), while an unresolved result leaves feasibility
unresolved. From a satisfying assignment, decode \(x_0\in I_D\) and
\(\mathbf u\in\prod_{t<N}U_t^D\), then invoke
\(\textsc{Run}_D(x_0,\mathbf u)\). A refused run is reported as
\(\mathsf{Refused}\). Any returned raw record is passed to
\(\textsc{Audit}_D\). If the audit returns
\(w_{\mathrm f}\in\mathrm{Tr}(D)\), set \(W_0=\{w_{\mathrm f}\}\). A refused
audit is reported as \(\mathsf{Refused}\). An \(\mathsf{Outside}\) or
\(\mathsf{Inconsistent}\) result is reported as \(\mathsf{Inconsistent}\).

\paragraph{Goal classification.}
Given a nonempty \(W_0\) and the selected set \(G\),
Algorithm~\ref{alg:closure} marks the goals hit by \(W_0\) as reached and
queries subsets of the remaining goals. A satisfiable query marks goals reached
only after the artifact is run on the decoded values and the resulting record
passes the audit. An unsatisfiable query marks
its target set as bounded-unreachable. If a query over several goals is
unresolved, the algorithm retries smaller subsets. Only an unresolved
singleton is marked unresolved. The result is
\((\mathit{Rch},\mathit{Unr},\mathit{Unk},W)\), where the first three sets
contain the reached, bounded-unreachable and unresolved goals, and \(W\) is
the audited pool witnessing \(\mathit{Rch}\).

\begin{algorithm}[!t]
\caption{\textsc{ClassifyGoals}: Replay-audited bounded reachability}
\label{alg:closure}
\DontPrintSemicolon
\KwIn{goal set \(G\subseteq\Gamma\), nonempty audited pool
      \(W_0\subseteq\mathrm{Tr}(D)\)}
\KwOut{\(\mathsf{Inconsistent}\), \(\mathsf{Refused}\), or
\((\mathit{Rch},\mathit{Unr},\mathit{Unk},W)\)}
\(W\gets W_0\); \(\mathit{Rch}\gets\mathrm{Hit}(W_0)\cap G\)\;
\(\mathit{Unr}\gets\emptyset\); \(\mathit{Unk}\gets\emptyset\);
\(\mathit{Pend}\gets G\setminus\mathit{Rch}\)\;
\While{\(\mathit{Pend}\neq\emptyset\)}{
  choose nonempty \(S\subseteq\mathit{Pend}\)\;
  solve \(\Phi_{D,S}\)\;
  \While{the result is unresolved and \(|S|>1\)}{
    choose a nonempty proper subset \(S'\subsetneq S\)\;
    \(S\gets S'\)\;
    solve \(\Phi_{D,S}\)\;
  }
  \uIf{satisfiable}{
    decode the initial state and input sequence from the satisfying model
      \(\nu\)\;
    invoke \(\textsc{Run}_D\) on the decoded values\;
    \If{it returns \(\mathsf{Refused}\)}{\Return \(\mathsf{Refused}\)\;}
    let \(e\) be the returned raw record\;
    \(r_{\mathrm{aud}}\gets\textsc{Audit}_D(e)\)\;
    \If{\(r_{\mathrm{aud}}=\mathsf{Refused}\)}{
      \Return \(\mathsf{Refused}\)\;}
    \If{\(r_{\mathrm{aud}}\in
        \{\mathsf{Outside},\mathsf{Inconsistent}\}\)}{
      \Return \(\mathsf{Inconsistent}\)\;}
    \(w\gets r_{\mathrm{aud}}\)\;
    \If{\(\mathrm{Hit}(w)\cap S=\emptyset\) or
        \(\mathrm{Hit}(w)\cap\mathit{Unr}\neq\emptyset\)}{
      \Return \(\mathsf{Inconsistent}\)\;
    }
    \(W\gets W\cup\{w\}\)\;
    \(\mathit{Rch}\gets\mathit{Rch}\cup
       \bigl(\mathrm{Hit}(w)\cap G\bigr)\)\;
    \(\mathit{Unk}\gets\mathit{Unk}\setminus\mathrm{Hit}(w)\)\;
    \(\mathit{Pend}\gets\mathit{Pend}\setminus\mathrm{Hit}(w)\)\;
  }
  \uElseIf{unsatisfiable}{
    \(\mathit{Unr}\gets\mathit{Unr}\cup S\)\;
    \(\mathit{Pend}\gets\mathit{Pend}\setminus S\)\;
  }
  \Else(\tcp*[f]{unresolved singleton}){
    \(\mathit{Unk}\gets\mathit{Unk}\cup S\)\;
    \(\mathit{Pend}\gets\mathit{Pend}\setminus S\)\;
  }
}
\Return \((\mathit{Rch},\mathit{Unr},\mathit{Unk},W)\)\;
\end{algorithm}

\begin{theorem}[Bounded-status soundness]
\label{thm:closure-sound}
Fix a bounded trace domain \(D\), a finite layered representation of \(D\)
with an exact bounded encoding, a goal set \(G\subseteq\Gamma\), and a
nonempty audited pool \(W_0\subseteq\mathrm{Tr}(D)\). Assume that every solver
call returns a satisfying model, a correct unsatisfiable result or an
unresolved result, and that every decoding,
\(\textnormal{\textsc{Run}}_D\) invocation and
audit terminates. Then \(\textnormal{\textsc{ClassifyGoals}}(G,W_0)\) in
Algorithm~\ref{alg:closure} terminates. If it returns
\((\mathit{Rch},\mathit{Unr},\mathit{Unk},W)\), its three status sets
partition \(G\) and
\[
  \mathit{Rch}\subseteq\mathrm{Reach}_D,
  \qquad
  \mathit{Unr}\cap\mathrm{Reach}_D=\emptyset,
\]
and every goal in \(\mathit{Rch}\) has a witness in \(W\).
\end{theorem}

\begin{proof}
At the head of each outer-loop iteration, \(\mathit{Rch}\),
\(\mathit{Unr}\) and \(\mathit{Unk}\) contain the goals currently marked
reached, bounded-unreachable and unresolved. The set \(W\) is the accumulated
pool of audited traces, and \(\mathit{Pend}\subseteq G\) contains the goals
not yet assigned a status. Maintain
\[
  W\subseteq\mathrm{Tr}(D),
  \qquad
  \mathit{Rch}=\mathrm{Hit}(W)\cap G,
\]
and require \(\mathit{Rch}\), \(\mathit{Unr}\), \(\mathit{Unk}\) and
\(\mathit{Pend}\) to be pairwise disjoint with union \(G\). These properties
hold initially. Within one outer-loop iteration,
\(S\subseteq\mathit{Pend}\) is the nonempty goal set targeted by the current
query. Each unresolved inner-loop iteration replaces
\(S\) with a nonempty proper subset, so the inner loop terminates. For a
satisfiable query, a successful audit gives
\(w\in\mathrm{Tr}(D)\). Let \(J=\mathrm{Hit}(w)\cap G\). The checks give
\(J\cap S\neq\emptyset\) and \(J\cap\mathit{Unr}=\emptyset\). The updates add
\(w\) to \(W\), move the goals in
\(J\cap(\mathit{Pend}\cup\mathit{Unk})\) to \(\mathit{Rch}\), and preserve
the invariant. They also strictly decrease \(|\mathit{Pend}|\) because
\(J\cap S\neq\emptyset\) and \(S\subseteq\mathit{Pend}\). For an
unsatisfiable query,
Theorem~\ref{thm:bounded} gives
\(S\cap\mathrm{Reach}_D=\emptyset\). Moving the nonempty set \(S\) from
\(\mathit{Pend}\) to \(\mathit{Unr}\) therefore preserves the partition,
adds only bounded-unreachable goals and strictly decreases
\(|\mathit{Pend}|\). If the result
remains unresolved after the inner loop, then
\(|S|=1\). Moving this singleton from \(\mathit{Pend}\) to \(\mathit{Unk}\)
again preserves the invariant and decreases \(|\mathit{Pend}|\). Any failed
check returns
immediately. Since \(G\) is finite, the outer loop terminates.
At a normal return, \(\mathit{Pend}=\emptyset\), so the three status sets
partition \(G\).
For any \(\gamma\), the invariant gives
\[
  \gamma\in\mathit{Rch}
  \quad\Longrightarrow\quad
  \gamma\in\mathrm{Hit}(W)
  \quad\Longrightarrow\quad
  \exists w\in W\text{ such that }\gamma\in\mathrm{Hit}(w).
\]
Since \(W\subseteq\mathrm{Tr}(D)\), this \(w\) is a witness within \(D\), and
Definition~\ref{def:domain} gives
\(\gamma\in\mathrm{Reach}_D\). Finally, every set added to \(\mathit{Unr}\)
was proved disjoint from \(\mathrm{Reach}_D\) by
Theorem~\ref{thm:bounded}. Hence
\(\mathit{Unr}\cap\mathrm{Reach}_D=\emptyset\).
\end{proof}

\begin{corollary}[Complete bounded classification]
\label{cor:closure}
Under the hypotheses of Theorem~\ref{thm:closure-sound}, suppose
Algorithm~\ref{alg:closure} returns normally with \(\mathit{Unk}=\emptyset\).
Then
\[
  \mathit{Rch}=\mathrm{Reach}_D\cap G,
  \qquad
  \mathit{Unr}=G\setminus\mathrm{Reach}_D.
\]
\end{corollary}

\begin{proof}
Since \(\mathit{Unk}=\emptyset\), Theorem~\ref{thm:closure-sound} gives
\(G=\mathit{Rch}\cup\mathit{Unr}\), with the two sets disjoint. Consider any
\(\gamma\in\mathrm{Reach}_D\cap G\). The partition places \(\gamma\) in
\(\mathit{Rch}\) or \(\mathit{Unr}\), and
\(\mathit{Unr}\cap\mathrm{Reach}_D=\emptyset\) excludes the latter. Hence
\(\mathrm{Reach}_D\cap G\subseteq\mathit{Rch}\). Conversely, if
\(\gamma\in\mathit{Rch}\), then Theorem~\ref{thm:closure-sound} gives
\(\gamma\in\mathrm{Reach}_D\), while the partition gives \(\gamma\in G\).
Thus \(\mathit{Rch}\subseteq\mathrm{Reach}_D\cap G\), proving the first
equality. The partition then gives
\begin{align*}
  \mathit{Unr}
  &=G\setminus\mathit{Rch}\\
  &=G\setminus(\mathrm{Reach}_D\cap G)\\
  &=G\setminus\mathrm{Reach}_D.\qedhere
\end{align*}
\end{proof}

\subsection{RTL Event-Slot Semantics and Protocol Adapters}
\label{app:bs-signal}

\paragraph{Overview.}
This subsection instantiates the instrumented semantics of
Definition~\ref{def:system} for pin-level RTL. It defines the RTL transition
for one event slot. Behavior consumes and emits logical transactions directly.
RTL instead uses a task-declared protocol adapter to drive interface pins and
identify the transactions accepted and emitted in each event slot.

\paragraph{Goal producers.}
The coverage/safety partition of Definition~\ref{def:goals} classifies goals
by purpose. Independently, the RTL catalogue has the producer partition
\[
  \Gamma_{\hat R}
  =\Gamma_{\hat R}^{\mathrm{slot}}\cup\Gamma_{\hat R}^{\mathrm{mon}},
  \qquad
  \Gamma_{\hat R}^{\mathrm{slot}}\cap\Gamma_{\hat R}^{\mathrm{mon}}
  =\emptyset.
\]
The event-slot semantics emits only slot goals, and the finite monitors emit
only monitor goals. Either producer class may contain coverage and safety
goals.

\begin{definition}[RTL event-slot semantics]
\label{def:slot}
Let \(\mathcal E_R\) be the finite set of legal clock and reset event sets.
Each event is tagged by its signal, domain and polarity. Let \(Q\) be the
finite RTL-state set, let \(U_R\) and \(Y_R\) be the finite pin-input and
pin-observation sets, and let
\(\emptyset\neq I_R\subseteq Q\) be the initial-state set. The slot semantics
is the total function
\[
  \delta_{\hat R}^{\mathrm{slot}}:
  Q\times U_R\times\mathcal E_R
  \longrightarrow Q\times Y_R\times2^{\Gamma_{\hat R}^{\mathrm{slot}}}.
\]
One application of \(\delta_{\hat R}^{\mathrm{slot}}\) is an \emph{RTL event slot}.
Each slot is one scheduled step: inputs and resets are driven
and settled before pin sampling and active clock events.
Writing
\(\delta_{\hat R}^{\mathrm{slot}}(\eta_t,u_t^{\mathrm{pin}},e_t)
=(\eta_{t+1},o_t,H_t^{\mathrm{slot}})\),
\(o_t\) is sampled before active clock events,
\(\eta_{t+1}\) is the settled state after clocks return to inactive levels,
and \(H_t^{\mathrm{slot}}\) contains exactly the raw RTL goals hit at
declared sampling phases. \(N\) counts slots, not internal phases.
In the single-clock measurement profile, each slot spans one cycle.
\end{definition}

\begin{definition}[Finite protocol adapter]
\label{def:interface}
Fix the transaction interface of Definition~\ref{def:txn-interface} and the
RTL event-slot semantics of Definition~\ref{def:slot}. A finite protocol
adapter is a tuple
\[
  \Pi=(X_\Pi,I_\Pi,\mathcal A_\Pi,\lambda_\Pi,\tau_\Pi)
\]
selected by the task's endpoint bindings. These bindings declare each
endpoint's pins, their ownership, and its clock and reset domain. An adapter
state has the form \(d=(\mathbf i,r)\in X_\Pi\), where \(\mathbf i\) is an input-queue
family permitted by the fixed bounds of \(\mathcal S_{\hat R}\) and \(r\) is
finite environment and protocol bookkeeping. When the task declares a finite
external-memory model, its state is a component of \(r\), as shown in
Figure~\ref{fig:artifact-execution-architecture}. The queues remain fixed
during a run, while \(r\) may change. Let
\(I_\Pi\subseteq X_\Pi\) be the nonempty set of states
with initial bookkeeping. For a stimulus \(\xi=(\mathbf i,\sigma)\),
initialization selects \(d_0=(\mathbf i,r_0)\in I_\Pi\). For RTL, each
schedule entry introduced in
Definition~\ref{def:stimulus} has the form
\[
  \sigma_t=(e_t,\alpha_t)
  \in\Sigma_\Pi:=\mathcal E_R\times\mathcal A_\Pi.
\]
Here \(\mathcal A_\Pi\) is the finite environment-action set. The component
\(e_t\) contains the clock and reset events, while \(\alpha_t\) specifies which
queued inputs become available for presentation, which outputs to accept, any
external-memory response, and other environment choices. The total drive function
\[
  \lambda_\Pi:X_\Pi\times\Sigma_\Pi\longrightarrow U_R,
  \qquad
  u_t^{\mathrm{pin}}=\lambda_\Pi(d_t,\sigma_t)
\]
constructs the RTL pin inputs from the queued transactions, current
bookkeeping, slot events and environment choices.
After the RTL produces the pin observation \(o_t\in Y_R\), the adapter updates
its state and identifies the transfers in that slot. Let
\[
  \mathcal B^{\mathrm{in}}
  =\prod_{c\in C^{\mathrm{in}}}(T_c\cup\{\bot\}),
  \qquad
  \mathcal B^{\mathrm{out}}
  =\prod_{c\in C^{\mathrm{out}}}(T_c\cup\{\bot\}),
\]
where each component is the transaction accepted or emitted on that endpoint,
or \(\bot\) if there is none. The total update function is
\[
  \tau_\Pi:
  X_\Pi\times\Sigma_\Pi\times U_R\times Y_R
  \longrightarrow
  X_\Pi\times\mathcal B^{\mathrm{in}}\times\mathcal B^{\mathrm{out}},
\]
with
\[
  (d_{t+1},a_t,b_t)
  =\tau_\Pi(d_t,\sigma_t,u_t^{\mathrm{pin}},o_t).
\]
Writing \(a_{t,c}\) for the \(c\)-component of \(a_t\), the non-\(\bot\)
entries of \((a_{t,c})_t\) form a prefix of \(\mathbf i_c\). Thus the adapter
preserves queue order and reports no input transaction more than once.
\end{definition}

\paragraph{Ready/valid instance.}
The endpoint binding losslessly packs each transaction \(v\in T_c\) onto its
payload pins and unpacks each observed payload, with
\(\mathit{unpack}_c(\mathit{pack}_c(v))=v\). Figure~\ref{fig:artifact-execution-architecture}
shows the signal ownership. For an input endpoint \(c\),
\(\lambda_\Pi\) drives \(\mathit{in\_valid}\) and \(\mathit{in\_data}\), while
\(\tau_\Pi\) observes the RTL-driven \(\mathit{in\_ready}\). The bookkeeping
contains a pointer \(p_{c,t}\in\{0,\ldots,k_c\}\) to the first unaccepted input,
initially zero, and tracks the inputs made available by \(\alpha_t\) that remain
pending. For each input endpoint, the adapter presents the oldest pending
transaction and holds it stable until acceptance. Reset updates pending work
according to the task's declared reset policy. Let
\(\mathit{present}_{c,t}\) indicate that a pending transaction is being
presented. The driven signals are
\[
  \mathit{in\_valid}_{c,t}
  =\mathit{present}_{c,t}\wedge[p_{c,t}<k_c],
  \qquad
  \mathit{in\_data}_{c,t}=
  \begin{cases}
    \mathit{pack}_c(x_{c,p_{c,t}}) & \mathit{in\_valid}_{c,t},\\
    0 & \text{otherwise},
  \end{cases}.
\]
For an output endpoint, \(\lambda_\Pi\) drives \(\mathit{out\_ready}\) from
\(\alpha_t\), while \(\tau_\Pi\) observes the RTL-driven
\(\mathit{out\_valid}\) and \(\mathit{out\_data}\). At an active, non-reset
event for the corresponding endpoint, the input and output transfer conditions
are
\[
\begin{aligned}
  \mathit{fire}^{\mathrm{in}}_{c,t}
  &\iff
  \mathit{in\_valid}_{c,t}\wedge\mathit{in\_ready}_{c,t}
  &&(c\in C^{\mathrm{in}}),\\
  \mathit{fire}^{\mathrm{out}}_{c,t}
  &\iff
  \mathit{out\_valid}_{c,t}\wedge\mathit{out\_ready}_{c,t}
  &&(c\in C^{\mathrm{out}}).
\end{aligned}
\]
No transfer occurs for that endpoint in any other slot. An input transfer puts
\(x_{c,p_{c,t}}\) in \(a_t\) and advances the pointer. An output transfer puts
\(\mathit{unpack}_c(\mathit{out\_data}_{c,t})\) in \(b_t\). A non-transferring
endpoint contributes \(\bot\), and a reset applies the adapter's declared
reset update. If reset occurs with transactions in flight, the task declares
whether they are cancelled or remain outstanding. Without this declaration,
checks that depend on them remain unresolved. The adapter retains a presented
input transaction until transfer unless a reset intervenes. Thus, for every input
endpoint \(c\) and \(t<N_{\hat R}-1\)
such that neither \(e_t\) nor \(e_{t+1}\) resets \(c\),
\[
  \mathit{present}_{c,t}\wedge[p_{c,t}<k_c]
  \wedge\neg\mathit{fire}^{\mathrm{in}}_{c,t}
  \quad\Longrightarrow\quad \mathit{present}_{c,t+1}.
\]
The dual design obligation retains an output while it is stalled. For every
output endpoint \(c\) and \(t<N_{\hat R}-1\) such that neither \(e_t\) nor
\(e_{t+1}\) resets \(c\),
\[
  \mathit{out\_valid}_{c,t}\wedge\neg\mathit{out\_ready}_{c,t}
  \quad\Longrightarrow\quad
  \mathit{out\_valid}_{c,t+1}\wedge
  \mathit{out\_data}_{c,t+1}=\mathit{out\_data}_{c,t}.
\]
The corresponding monitor emits a safety goal if this implication fails.

\paragraph{Transaction projection.}
During an RTL replay of \(\xi\), write the per-slot batches returned by
\(\tau_\Pi\) as
\[
  a_t^{\hat R}(\xi)
  =\bigl(a_{t,c}^{\hat R}(\xi)\bigr)_{c\in C^{\mathrm{in}}},
  \qquad
  b_t^{\hat R}(\xi)
  =\bigl(b_{t,c}^{\hat R}(\xi)\bigr)_{c\in C^{\mathrm{out}}}.
\]
The accepted input trace is
\(\mathbf a_{\hat R}(\xi)=(a_t^{\hat R}(\xi))_{t<N_{\hat R}}\). For each output
endpoint \(c\), \(o_{\hat R,c}(\xi)\) is obtained from
\((b_{t,c}^{\hat R}(\xi))_{t<N_{\hat R}}\) by removing the \(\bot\) entries. The
output-history family is
\(\mathbf o_{\hat R}(\xi)=(o_{\hat R,c}(\xi))_{c\in C^{\mathrm{out}}}\).

\subsection{RTL Closed-Loop Semantics and Trace Domain}
\label{app:bs-rtl-domain}

\paragraph{Overview.}
Using the event-slot semantics and protocol adapter of
Subsection~\ref{app:bs-signal}, this subsection instantiates the instrumented
semantics and bounded trace domain of Subsection~\ref{app:bs-generic} for RTL.
It composes the adapter, RTL design and enabled finite monitors, then defines
the allowed clock, reset and environment-action sequences.

\begin{definition}[Finite monitor product]
\label{def:monitor}
Let \(Z\) be the finite state set of the product of all monitors enabled for
\(\hat R\), let \(z_0\in Z\) be its initial state, and let its total update
function be
\[
  \mu:Z\times\Sigma_\Pi\times Q\times U_R\times Y_R\times Q
       \times\mathcal B^{\mathrm{in}}\times\mathcal B^{\mathrm{out}}
       \longrightarrow Z\times2^{\Gamma_{\hat R}^{\mathrm{mon}}}
\,.
\]
It observes the schedule entry, the RTL step and the adapter's transaction
events, then returns the next monitor state and the monitor goals hit in that
slot. Its components emit the enabled coverage and safety goals assigned to
monitors, and their reset behavior is part of \(\mu\).
An optional task-declared response deadline pairs an accepted input with its first
response offer and emits a safety goal if the declared bound expires first. If
the task declares
reset restoration, another component tracks the declared reset sequence. It
emits a safety goal when the declared post-reset condition fails at completion,
and a scenario goal when the sequence completes after at least one non-reset
slot. If no monitor is enabled, then \(Z=\{z_0\}\) and \(\mu\) emits no goals.
\end{definition}

\begin{definition}[Closed-loop RTL semantics]
\label{def:product}
Define
\[
\begin{aligned}
  X_{\hat R}^{\mathrm{cl}}&=X_\Pi\times Q\times Z,\\
  Y_{\hat R}^{\mathrm{cl}}
    &=Y_R\times\mathcal B^{\mathrm{in}}\times\mathcal B^{\mathrm{out}},\\
  I_{\hat R}^{\mathrm{cl}}
    &=I_\Pi\times I_R\times\{z_0\}.
\end{aligned}
\]
The instrumented semantics used to analyze \(\hat R\) is
\[
  \mathcal M_{\hat R}
  =\bigl(X_{\hat R}^{\mathrm{cl}},\Sigma_\Pi,Y_{\hat R}^{\mathrm{cl}},
         I_{\hat R}^{\mathrm{cl}},\delta_{\hat R}^{\mathrm{cl}},
         \Gamma_{\hat R}\bigr),
\]
where, for one slot, \(d_t\in X_\Pi\), \(\eta_t\in Q\), \(z_t\in Z\),
\(\sigma_t=(e_t,\alpha_t)\), and
\(u_t^{\mathrm{pin}}=\lambda_\Pi(d_t,\sigma_t)\),
\[
\begin{aligned}
  (\eta_{t+1},o_t,H_t^{\mathrm{slot}})
    &=\delta_{\hat R}^{\mathrm{slot}}
      (\eta_t,u_t^{\mathrm{pin}},e_t),\\
  (d_{t+1},a_t,b_t)
    &=\tau_\Pi(d_t,\sigma_t,u_t^{\mathrm{pin}},o_t),\\
  (z_{t+1},H_t^{\mathrm{mon}})
    &=\mu(z_t,\sigma_t,\eta_t,u_t^{\mathrm{pin}},o_t,
          \eta_{t+1},a_t,b_t),\\
  H_t&=H_t^{\mathrm{slot}}\cup H_t^{\mathrm{mon}},\\
  \delta_{\hat R}^{\mathrm{cl}}((d_t,\eta_t,z_t),\sigma_t)
    &=((d_{t+1},\eta_{t+1},z_{t+1}),(o_t,a_t,b_t),H_t).
\end{aligned}
\]
\end{definition}

\begin{definition}[RTL trace domain]
\label{def:env-contract}
Set \(N=N_{\hat R}\). Let \(K_N^{\mathrm{env}}\) be the admissibility
predicate of Definition~\ref{def:domain} for \(\mathcal M_{\hat R}\). It holds
exactly for the \(N\)-slot closed-loop traces whose clock and reset events
satisfy the fixed harness configuration and whose environment actions satisfy
the obligations of the selected protocol adapter. Define
\[
  D_{\hat R}=(\mathcal M_{\hat R},N,K_N^{\mathrm{env}}).
\]
An environment-side violation therefore excludes the trace from
\(\mathrm{Tr}(D_{\hat R})\), whereas a design-side protocol violation is
recorded as a safety hit.
\end{definition}

\subsection{RTL Finite Representation, Exact Encoding and Audits}
\label{app:bs-rtl}

\paragraph{Overview.}
This subsection constructs the finite RTL representation and exact encoding,
states the required semantic, catalogue and replay fidelity assumptions, and
defines the catalogue and replay audits.

\paragraph{Admitted RTL.}
The RTL analysis admits elaborated designs whose fixed-width signals and
fixed-size memories range over a finite logic domain. It requires an explicit
initial-state predicate and a total deterministic event-slot transition. The
selected protocol adapter and monitors must likewise be finite and total.

\paragraph{Component encoding.}
Assign fixed injective encodings and exact validity predicates to the
components of \(X_{\hat R}^{\mathrm{cl}}\), \(\Sigma_\Pi\) and
\(Y_{\hat R}^{\mathrm{cl}}\), and to the internal pin-input set \(U_R\). At
slot \(t\), write the product encodings as
\[
  \widehat x_t=(\widehat d_t,\widehat\eta_t,\widehat z_t),
  \qquad
  \widehat\sigma_t=(\widehat e_t,\widehat\alpha_t),
  \qquad
  \widehat y_t=(\widehat o_t,\widehat a_t,\widehat b_t).
\]
A product encoding is valid exactly when each component encoding is valid.
Exact analysis requires the RTL frontend to supply
the encoded initial-state predicate \(\widehat I_R\) and
the one-slot relation
\[
  \widehat{\delta}_{\hat R}^{\mathrm{slot}}
  (\widehat\eta,\widehat u^{\mathrm{pin}},\widehat e,
   \widehat\eta',\widehat o,\mathbf h^{\mathrm{slot}}),
  \qquad
  \mathbf h^{\mathrm{slot}}
  =(h_\gamma^{\mathrm{slot}})_{\gamma\in\Gamma_{\hat R}^{\mathrm{slot}}}
  \in\mathsf{Bool}^{\Gamma_{\hat R}^{\mathrm{slot}}}.
\]
The hit bits must encode the catalogue hit rules from
Subsection~\ref{app:bs-goals}. For source goals, this includes preserving their
event semantics through compilation.
Here \(\widehat u^{\mathrm{pin}}\) encodes a value in \(U_R\). The adapter and
monitor functions in Definitions~\ref{def:interface} and~\ref{def:monitor}, and
\(K_N^{\mathrm{env}}\), have finite domains and therefore admit exact encodings.
Write these as \(\widehat\lambda_\Pi\), \(\widehat\tau_\Pi\),
\(\widehat\mu\) and \(\widehat K_N^{\mathrm{env}}\). On valid encodings, the
first three return the encoded results of \(\lambda_\Pi\), \(\tau_\Pi\) and
\(\mu\), while \(\widehat K_N^{\mathrm{env}}\) holds exactly for traces
satisfying \(K_N^{\mathrm{env}}\). For
\(\widehat x=(\widehat d,\widehat\eta,\widehat z)\), define the product
initial-state predicate by
\[
  \widehat I_{D_{\hat R}}(\widehat x)
  \quad\Longleftrightarrow\quad
  d\in I_\Pi\ \wedge\ \widehat I_R(\widehat\eta)\ \wedge\ z=z_0,
\]
where \(d\) and \(z\) are the values decoded from \(\widehat d\) and
\(\widehat z\). The monitor hit set is represented by
\(
  \mathbf h^{\mathrm{mon}}
  =(h_\gamma^{\mathrm{mon}})_{\gamma\in\Gamma_{\hat R}^{\mathrm{mon}}}
  \in\mathsf{Bool}^{\Gamma_{\hat R}^{\mathrm{mon}}}
\).

\begin{assumption}[RTL semantic fidelity]
\label{asm:lr}
Under the event-slot semantics of Definition~\ref{def:slot}, every valid
encoding \(\widehat\eta\) of \(\eta\) satisfies
\[
  \widehat I_R(\widehat\eta)
  \quad\Longleftrightarrow\quad \eta\in I_R.
\]
For every valid encoding of
\(\eta,u^{\mathrm{pin}},e,\eta',o\) and every
\(\mathbf h^{\mathrm{slot}}
\in\mathsf{Bool}^{\Gamma_{\hat R}^{\mathrm{slot}}}\),
\[
\begin{aligned}
  &\widehat{\delta}_{\hat R}^{\mathrm{slot}}
    (\widehat\eta,\widehat u^{\mathrm{pin}},\widehat e,
     \widehat\eta',\widehat o,\mathbf h^{\mathrm{slot}})\\
  &\quad\Longleftrightarrow\quad
    \delta_{\hat R}^{\mathrm{slot}}(\eta,u^{\mathrm{pin}},e)
    =\left(\eta',o,
      \{\gamma\in\Gamma_{\hat R}^{\mathrm{slot}}\mid
        h_\gamma^{\mathrm{slot}}=\mathsf{true}\}\right).
\end{aligned}
\]
\end{assumption}

\begin{theorem}[Exact RTL encoding]
\label{thm:rtl-exact}
Under Assumption~\ref{asm:lr} and the exact component encodings fixed above,
set
\[
  X_t^{D_{\hat R}}=X_{\hat R}^{\mathrm{cl}}\quad(t\le N),
  \qquad
  U_t^{D_{\hat R}}=\Sigma_\Pi,\quad
  Y_t^{D_{\hat R}}=Y_{\hat R}^{\mathrm{cl}}\quad(t<N),
  \qquad
  I_{D_{\hat R}}=I_{\hat R}^{\mathrm{cl}}.
\]
For \(t<N\), let
\(\mathbf h_t=(h_{t,\gamma})_{\gamma\in\Gamma_{\hat R}}\), set
\(\widehat u_t^{\mathrm{pin}}
=\widehat\lambda_\Pi(\widehat d_t,\widehat\sigma_t)\), and define
\[
\begin{aligned}
  &\widehat{\delta}_t^{D_{\hat R}}
    (\widehat x_t,\widehat\sigma_t,\widehat x_{t+1},
     \widehat y_t,\mathbf h_t)
  \\[-2pt]
  {}\Longleftrightarrow{}&
  \exists
  (\mathbf h_t^{\mathrm{slot}},\mathbf h_t^{\mathrm{mon}})
  \in
  \mathsf{Bool}^{\Gamma_{\hat R}^{\mathrm{slot}}}
  \times\mathsf{Bool}^{\Gamma_{\hat R}^{\mathrm{mon}}}:\\[-2pt]
  &\qquad
  \widehat{\delta}_{\hat R}^{\mathrm{slot}}
  (\widehat\eta_t,\widehat u_t^{\mathrm{pin}},\widehat e_t,
   \widehat\eta_{t+1},\widehat o_t,\mathbf h_t^{\mathrm{slot}})\\
  &\qquad\wedge\
      (\widehat d_{t+1},\widehat a_t,\widehat b_t)
      =\widehat\tau_\Pi(\widehat d_t,\widehat\sigma_t,
                         \widehat u_t^{\mathrm{pin}},\widehat o_t)\\
  &\qquad\wedge\ (\widehat z_{t+1},\mathbf h_t^{\mathrm{mon}})
      =\widehat\mu(\widehat z_t,\widehat\sigma_t,\widehat\eta_t,
                    \widehat u_t^{\mathrm{pin}},\widehat o_t,
                    \widehat\eta_{t+1},\widehat a_t,\widehat b_t)\\
  &\qquad\wedge\
      \bigwedge_{\gamma\in\Gamma_{\hat R}^{\mathrm{slot}}}
      \left(h_{t,\gamma}\leftrightarrow
            h_{t,\gamma}^{\mathrm{slot}}\right)\\
  &\qquad\wedge\
      \bigwedge_{\gamma\in\Gamma_{\hat R}^{\mathrm{mon}}}
      \left(h_{t,\gamma}\leftrightarrow
            h_{t,\gamma}^{\mathrm{mon}}\right).
\end{aligned}
\]
These carriers form a finite layered representation of \(D_{\hat R}\), and
the component encodings, validity predicates,
\(\widehat I_{D_{\hat R}}\), \(\widehat K_N^{\mathrm{env}}\) and
\((\widehat\delta_t^{D_{\hat R}})_{t<N}\) form an exact bounded encoding of
\(\mathcal A_{D_{\hat R}}\).
\end{theorem}

\begin{proof}
The displayed carriers are finite by
Definitions~\ref{def:txn-interface}, \ref{def:slot}, \ref{def:interface}
and~\ref{def:monitor}. Because they are the full carriers of
\(\mathcal M_{\hat R}\), every trace in \(\mathrm{Tr}(D_{\hat R})\) satisfies
trace support, and
\(I_{D_{\hat R}}=I_{\hat R}^{\mathrm{cl}}\) satisfies the initial-carrier
condition. Totality of \(\delta_{\hat R}^{\mathrm{cl}}\) gives layer closure.
Hence the displayed carriers form a finite layered representation. Fix
\(t<N\) and valid encodings of
\(x_t,\sigma_t,x_{t+1},y_t\). Exactness of
\(\widehat\lambda_\Pi\), Assumption~\ref{asm:lr}, and exactness of
\(\widehat\tau_\Pi\) and \(\widehat\mu\) give the three component updates in
Definition~\ref{def:product}. Assumption~\ref{asm:lr} identifies the slot hit
vector with \(H_t^{\mathrm{slot}}\), while exactness of \(\widehat\mu\)
identifies the monitor hit vector with \(H_t^{\mathrm{mon}}\). The producer
partition gives their union \(H_t\). Therefore
\[
\begin{aligned}
  &\widehat\delta_t^{D_{\hat R}}
    (\widehat x_t,\widehat\sigma_t,\widehat x_{t+1},
     \widehat y_t,\mathbf h_t)\\
  &\quad\Longleftrightarrow\quad
    \delta_{\hat R}^{\mathrm{cl}}(x_t,\sigma_t)
    =\left(x_{t+1},y_t,
      \{\gamma\in\Gamma_{\hat R}\mid
        h_{t,\gamma}=\mathsf{true}\}\right).
\end{aligned}
\]
This is the transition equivalence required by
Definition~\ref{def:exact-enc}. For valid encoded sequences, write
\(\widehat{\mathbf x},\widehat\sigma,\widehat{\mathbf y}\) for the encoded
values and \(\mathbf x,\sigma,\mathbf y\) for their decodings. The definition
of \(\widehat I_{D_{\hat R}}\) and Assumption~\ref{asm:lr} give the first
equivalence below. Exactness of \(\widehat K_N^{\mathrm{env}}\) gives the
second:
\[
  \widehat I_{D_{\hat R}}(\widehat x_0)
  \Longleftrightarrow x_0\in I_{D_{\hat R}},
  \qquad
  \widehat K_N^{\mathrm{env}}
    (\widehat{\mathbf x},\widehat\sigma,\widehat{\mathbf y})
  \Longleftrightarrow
  K_N^{\mathrm{env}}(\mathbf x,\sigma,\mathbf y).
\]
Together with the component validity predicates, these are all the conditions
of Definition~\ref{def:exact-enc}.
\end{proof}

\begin{assumption}[RTL catalogue fidelity]
\label{asm:gr}
The RTL instrumentation and monitors implement exactly the identifiers and hit
rules fixed by the goal policy. For
\(\gamma\in\Gamma_{\hat R}^{\mathrm{slot}}\),
\(h_\gamma^{\mathrm{slot}}\) is true exactly when an RTL occurrence assigned to
\(\gamma\) activates. For
\(\gamma\in\Gamma_{\hat R}^{\mathrm{mon}}\),
\(h_\gamma^{\mathrm{mon}}\) is true exactly when the corresponding monitor emits
\(\gamma\).
\end{assumption}

\begin{assumption}[RTL replay fidelity]
\label{asm:hr}
For every \(x_0\in I_{\hat R}^{\mathrm{cl}}\) and
\(\sigma\in\Sigma_\Pi^N\), if
\(\textsc{Run}_{D_{\hat R}}(x_0,\sigma)\) returns a raw record, that record is
a complete and faithful account of the actual \(N\)-slot closed-loop replay
from \(x_0\) under \(\sigma\).
\end{assumption}

\paragraph{Catalogue and replay audits.}
The RTL harness instantiates \(\textsc{Run}_{D_{\hat R}}\) by restoring the
selected initial product state and executing all \(N\) schedule entries. Its
record includes the adapter, RTL and monitor states, pin observations,
transaction events and both component hit sets, including reset and
no-transfer slots. Before bounded analysis, the catalogue audit checks that the
complete encoded goal identifiers and hit rules match \(\Gamma_{\hat R}\).
Each raw record is then passed to \(\textsc{Audit}_{D_{\hat R}}\) of
Definition~\ref{def:replay-audit} using
\(\widehat\delta_t^{D_{\hat R}}\). Only records returned as traces are retained.

\subsection{Behavior IR and Procedure Semantics}
\label{app:bs-ir}

\paragraph{Overview.}
RTL advances through event slots and uses a protocol adapter to recover logical
transactions from pins. A Behavior model already operates on transactions, so
one step is one complete model invocation. The RTL event-slot transition is a
semantic primitive, whereas the Behavior transition is derived from the small
IR below. This subsection defines the IR's state updates, outputs, goal hits and
declared failures.

\begin{definition}[Behavior IR]
\label{def:ir}
The semantic types are
\[
  \tau ::= \mathsf{bool}\mid\mathsf{int}\mid\mathsf{array}[n,\tau],
  \qquad n\in\mathbb N_{>0},
\]
with value sets
\[
  [\![\mathsf{bool}]\!]=\mathsf{Bool},\qquad
  [\![\mathsf{int}]\!]=\mathbb Z,\qquad
  [\![\mathsf{array}[n,\tau]]\!]
  = [\![\tau]\!]^n.
\]
Thus arrays have fixed length and value semantics. Let \(\mathcal T\) be the
set of semantic types and \(\mathsf{Var}\) the set of variable names. A typed
context is a finite partial map
\(\Theta:\mathsf{Var}\rightharpoonup\mathcal T\), and
\(\operatorname{dom}(\Theta)\) is the finite set of variables it declares. Its
total typed valuations form the set
\[
  \mathrm{Val}(\Theta)
  =\left\{\rho:
      \operatorname{dom}(\Theta)
      \to\bigcup_{\tau\in\mathcal T}[\![\tau]\!]
      \ \middle|\;
      \rho(z)\in[\![\Theta(z)]\!]
      \text{ for every }z\in\operatorname{dom}(\Theta)\right\}.
\]
Thus each \(\rho\in\mathrm{Val}(\Theta)\) maps every declared variable to a
value of its declared type.
A Behavior procedure \(M\) represents one complete model invocation and
consists of
\[
  M=(\Theta_S,\Theta_I,\Theta_L,\Theta_O,c,
    U_M,Y_M,\mathit{Err}_M,\Gamma_M^{\mathrm{cov}},
    \Gamma_M^{\mathrm{safe}}).
\]
The error names form a finite set \(\mathit{Err}_M\). The finite goal
catalogue is
\[
  \Gamma_M
  =\Gamma_M^{\mathrm{cov}}\cup\Gamma_M^{\mathrm{safe}},
  \qquad
  \Gamma_M^{\mathrm{cov}}\cap\Gamma_M^{\mathrm{safe}}=\emptyset.
\]
Its expressions and commands are
\[
\begin{array}{lcl}
  e &::=& z\mid\mathit{lit}\mid e_1[e_2]
          \mid p(e_1,\ldots,e_m),\\[2pt]
  c &::=& \mathsf{skip}\mid z:=e\mid z[e_1]:=e_2\mid c_1;c_2\\
    && \mid\ \mathsf{if}\ e\ \mathsf{then}\ c_T\ \mathsf{else}\ c_F
       \mid \mathsf{for}\ j<e\ \mathsf{do}\ c\\
    && \mid\ \mathsf{mark}(\gamma)
       \mid \mathsf{fail}(\varepsilon,\gamma_s).
\end{array}
\]
Here \(z\) ranges over variables, \(\mathit{lit}\) over typed literals,
\(p\) over registered primitives, and \(j\) over the read-only indices introduced
by \(\mathsf{for}\). Moreover,
\(\gamma\in\Gamma_M\), \(\gamma_s\in\Gamma_M^{\mathrm{safe}}\), and
\(\varepsilon\in\mathit{Err}_M\). The command \(\mathsf{mark}(\gamma)\)
records \(\gamma\) and continues, while
\(\mathsf{fail}(\varepsilon,\gamma_s)\) records \(\gamma_s\) and produces the
failure \(\mathsf{Err}(\varepsilon,\{\gamma_s\})\) of
Definition~\ref{def:absorb}. Each primitive has a fixed, typed, total
denotation.
Bit vectors, enumerations, and fixed- and floating-point values are represented
by integers and typed primitives. Records are flattened into typed variables,
and finite memories use arrays. Partial source operations, including array
reads and indexed updates, are guarded by explicit validity checks whose
invalid branches execute \(\mathsf{fail}\).
The pairwise-disjoint contexts \(\Theta_S,\Theta_I,\Theta_L\) and
\(\Theta_O\) declare its persistent state, input, temporary and output
variables, respectively. Write
\(\Theta_M=\Theta_S\cup\Theta_I\cup\Theta_L\cup\Theta_O\) for their union.
The command is \(c\), and the finite input and output domains are
\[
  U_M\subseteq\mathrm{Val}(\Theta_I),
  \qquad
  Y_M\subseteq\mathrm{Val}(\Theta_O)
\,.
\]
We regard each valuation of a context as one structured value. Thus \(U_M\)
is the finite set of admitted input records, and \(Y_M\) is the finite set of
normal outputs. Each \(y\in Y_M\) represents the complete ordered output
sequence of one invocation, including the empty and multi-transaction cases.
\end{definition}

\begin{definition}[Big-step command and procedure semantics]
\label{def:behavior-semantics}
A partial valuation \(\rho\) over \(\Theta_M\) satisfies
\(\operatorname{dom}(\rho)\subseteq\operatorname{dom}(\Theta_M)\) and
\(\rho(z)\in[\![\Theta_M(z)]\!]\) for every
\(z\in\operatorname{dom}(\rho)\). Write \(\rho|_{\Theta}\) for the restriction
of \(\rho\) to the variables in a context \(\Theta\). For an expression \(e\)
whose variables belong to \(\operatorname{dom}(\rho)\) and whose array accesses
are in bounds, let \([\![e]\!]_\rho\) denote its value in \(\rho\):
\[
\begin{aligned}
  [\![z]\!]_\rho&=\rho(z),
  & [\![\mathit{lit}]\!]_\rho&=\mathsf{val}(\mathit{lit}),\\
  [\![e_1[e_2]]\!]_\rho
    &=\mathsf{select}([\![e_1]\!]_\rho,[\![e_2]\!]_\rho),
  & [\![p(e_1,\ldots,e_m)]\!]_\rho
    &=[\![p]\!]([\![e_1]\!]_\rho,\ldots,[\![e_m]\!]_\rho).
\end{aligned}
\]
Here \(\mathsf{val}(\mathit{lit})\) denotes the value of a typed literal,
\(\mathsf{select}(a,i)\) returns the \(i\)-th element of \(a\), and
\([\![p]\!]\) is the fixed interpretation of the registered primitive \(p\).
The judgment
\[
  \langle c,\rho\rangle\Downarrow R
\]
relates an initial valuation to the result of executing the complete command
\(c\). A result is either a normal result \(\langle\rho',H\rangle\) or a
failure \(\mathsf{Err}(\varepsilon,H)\), where
\(\varepsilon\in\mathit{Err}_M\) and \(H\subseteq\Gamma_M\). The coverage and
safety hits are \(H\cap\Gamma_M^{\mathrm{cov}}\) and
\(H\cap\Gamma_M^{\mathrm{safe}}\), respectively. For an array \(a\),
a valid index \(k\), and a value \(v\), let
\(\mathsf{update}(a,k,v)\) be \(a\) with its \(k\)-th element replaced by
\(v\). The judgment is the least relation satisfying the following rules. The
primitive-command rules are
\[
  \frac{}
       {\langle\mathsf{skip},\rho\rangle
        \Downarrow\langle\rho,\emptyset\rangle}
  \qquad
  \frac{[\![e]\!]_\rho=v}
       {\langle z:=e,\rho\rangle
        \Downarrow\langle\rho[z\mapsto v],\emptyset\rangle},
\]
\[
  \frac{[\![e_1]\!]_\rho=k\qquad[\![e_2]\!]_\rho=v}
       {\langle z[e_1]:=e_2,\rho\rangle
        \Downarrow
        \langle\rho[z\mapsto\mathsf{update}(\rho(z),k,v)],
                \emptyset\rangle},
  \qquad
  \frac{}
       {\langle\mathsf{mark}(\gamma),\rho\rangle
        \Downarrow\langle\rho,\{\gamma\}\rangle},
\]
\[
  \frac{}
       {\langle\mathsf{fail}(\varepsilon,\gamma_s),\rho\rangle
        \Downarrow\mathsf{Err}(\varepsilon,\{\gamma_s\})}.
\]
Sequential composition has three rules:
\[
  \frac{\langle c_1,\rho\rangle\Downarrow\langle\rho_1,H_1\rangle
        \qquad
        \langle c_2,\rho_1\rangle\Downarrow\langle\rho_2,H_2\rangle}
       {\langle c_1;c_2,\rho\rangle
        \Downarrow\langle\rho_2,H_1\cup H_2\rangle},
\]
\[
  \frac{\langle c_1,\rho\rangle
        \Downarrow\mathsf{Err}(\varepsilon,H_1)}
       {\langle c_1;c_2,\rho\rangle
        \Downarrow\mathsf{Err}(\varepsilon,H_1)},
  \qquad
  \frac{\langle c_1,\rho\rangle\Downarrow\langle\rho_1,H_1\rangle
        \qquad
        \langle c_2,\rho_1\rangle
          \Downarrow\mathsf{Err}(\varepsilon,H_2)}
       {\langle c_1;c_2,\rho\rangle
        \Downarrow\mathsf{Err}(\varepsilon,H_1\cup H_2)}.
\]
The conditional rules are
\[
  \frac{[\![e]\!]_\rho=\mathsf{true}
        \qquad\langle c_T,\rho\rangle\Downarrow R}
       {\langle\mathsf{if}\ e\ \mathsf{then}\ c_T\
        \mathsf{else}\ c_F,\rho\rangle\Downarrow R},
  \qquad
  \frac{[\![e]\!]_\rho=\mathsf{false}
        \qquad\langle c_F,\rho\rangle\Downarrow R}
       {\langle\mathsf{if}\ e\ \mathsf{then}\ c_T\
        \mathsf{else}\ c_F,\rho\rangle\Downarrow R}.
\]
Thus the guard is evaluated once. For
\(m=\max(0,[\![e]\!]_\rho)\), let
\(c^{[m]}=c[j:=0];\cdots;c[j:=m-1]\), using capture-avoiding
substitution and taking \(c^{[0]}=\mathsf{skip}\). The loop rule is
\[
  \frac{m=\max(0,[\![e]\!]_\rho)
        \qquad\langle c^{[m]},\rho\rangle\Downarrow R}
       {\langle\mathsf{for}\ j<e\ \mathsf{do}\ c,\rho\rangle
        \Downarrow R}.
\]
Let \(X_M=\mathrm{Val}(\Theta_S)\). For \(s\in X_M\) and \(i\in U_M\), let
\(\rho_0=s\cup i\) be the partial valuation on
\(\Theta_S\cup\Theta_I\). Locals and outputs are initially undefined. With
this execution relation, call \(M\) \emph{well formed} when \(c\) is well typed
under \(\Theta_M\), guards are Boolean, loop bounds are integers, and only
state, local and output variables may be assigned. For every \(s\in X_M\) and
\(i\in U_M\), every variable read during execution from \(\rho_0\) must be
defined and every array access must be in bounds. Each normal result must also
satisfy
\[
  \langle c,\rho_0\rangle\Downarrow\langle\rho_f,H\rangle
  \quad\Longrightarrow\quad
  \rho_f|_{\Theta_O}\in Y_M.
\]
Define the fallible step induced by \(M\) as
\[
  \delta_M^0(s,i)=
  \begin{cases}
    \bigl(\rho_f|_{\Theta_S},\rho_f|_{\Theta_O},H\bigr)
      & \langle c,\rho_0\rangle
          \Downarrow\langle\rho_f,H\rangle,\\
    \mathsf{Err}(\varepsilon,H)
      & \langle c,\rho_0\rangle
          \Downarrow\mathsf{Err}(\varepsilon,H).
  \end{cases}
\]
\end{definition}

\begin{lemma}[Total instrumented Behavior semantics]
\label{lem:total}
For every well-formed \(M\),
\[
  \forall(s,i)\in X_M\times U_M,\qquad
  \exists!R:\ \langle c,s\cup i\rangle\Downarrow R.
\]
Hence \(\delta_M^0\) is total. Consequently, for every
\(\emptyset\neq I_M^0\subseteq X_M\),
\[
  \mathcal M_M(I_M^0)
  =\mathrm{Abs}_{I_M^0}(\delta_M^0)
\]
is an instrumented semantics in the sense of
Definition~\ref{def:system}. Write \(\bar\delta_M\) for its step function. Its
transitions from a normal state into an error state satisfy
\[
\begin{gathered}
  \forall s\in X_M,\ \forall i\in U_M,\
  \forall\varepsilon\in\mathit{Err}_M,\
  \forall H\subseteq\Gamma_M,\\
  \bar\delta_M(s,i)
  =(\varepsilon,\bot_Y,H)
  \quad\Longrightarrow\quad
  H\cap\Gamma_M^{\mathrm{safe}}\neq\emptyset.
\end{gathered}
\]
\end{lemma}

\begin{proof}
Fix \((s,i)\in X_M\times U_M\). For every command \(d\) obtained from \(c\) by
taking subcommands and substituting loop indices, and every valuation \(\rho\)
reached at \(d\), prove simultaneously by structural induction that
\[
\begin{gathered}
  \exists!R:\ \langle d,\rho\rangle\Downarrow R,\\
  \langle d,\rho\rangle\Downarrow\mathsf{Err}(\varepsilon,H)
  \quad\Longrightarrow\quad
  H\cap\Gamma_M^{\mathrm{safe}}\neq\emptyset.
\end{gathered}
\]
Well-formedness makes every evaluated expression defined, and its value is
unique. The primitive cases are immediate, and \(\mathsf{fail}\) is the only
one producing a failure, with a hit in \(\Gamma_M^{\mathrm{safe}}\). For
\(d=d_1;d_2\), the unique result of \(d_1\) selects the failure rule or
uniquely determines the input to \(d_2\). Union with earlier hits preserves a
safety hit. A Boolean guard selects exactly one branch. For a loop, the bound
fixes a unique finite \(m\). An inner induction on \(m\), using the structural
induction hypothesis for each \(c[j:=k]\), proves the claim for \(c^{[m]}\).
Taking \(d=c\) and \(\rho=s\cup i\) gives a unique result. On normal completion,
typing places the state projection in \(X_M\), and
well-formedness places the output projection in \(Y_M\). Thus \(\delta_M^0\)
is total, and absorption gives the stated instrumented semantics.
Absorption preserves the hit set on the transition from \(X_M\) into the
corresponding error state, proving the final implication.
\end{proof}

\subsection{Behavior Finite Representation and Exact Encoding}
\label{app:bs-encoding}

\paragraph{Overview.}
Definition~\ref{def:exact-enc} states what an exact bounded encoding must
provide. For a well-formed procedure \(M\) and a nonempty initial set
\(I_M^0\subseteq X_M\), let
\[
  D=(\mathcal M_M(I_M^0),N,K_N).
\]
This subsection constructs an exact bounded encoding of \(D\) using finite sets
that contain every reachable state and every intermediate value.

\begin{definition}[Behavior representation plan]
\label{def:plan}
A \emph{Behavior representation plan} \(P\) for \((M,D)\) assigns finite state
layers, finite sets of internal values, one encoding for each semantic type,
and an exact encoded term for each primitive occurrence, subject to the
conditions below.

\paragraph{State layers.}
For each \(t\le N\), the plan contains a finite set \(S_t\subseteq X_M\). These
sets satisfy
\[
\begin{aligned}
  &\{x_0\mid
      (\mathbf x,\mathbf u,\mathbf y,\mathbf H)
      \in\mathrm{Tr}(D)\}
  \subseteq S_0,\\
  &\{s'\mid \exists s\in S_t,\ u\in U_M,\ y,\ H:
      \delta_M^0(s,u)=(s',y,H)\}
    \subseteq S_{t+1}
    \qquad(t<N).
\end{aligned}
\]
The first condition covers every admitted initial state. The second closes the
normal states under one step. The sets \(S_t\) contain only normal states. If
\(\delta_M^0(s,u)=\mathsf{Err}(\varepsilon,H)\), absorption enters the terminal
state \(\varepsilon\in\mathit{Err}_M\). Program execution then stops. The state
\(\varepsilon\) is repeated only to pad the trace to horizon \(N\). The full
state layers are therefore
\[
  X_t^D:=S_t\cup\mathit{Err}_M
  \qquad(t\le N).
\]

\paragraph{Finite internal value sets.}
For every \(t<N\), command or expression occurrence \(\ell\), and quantity
\(q\) evaluated at \(\ell\), let \(\tau_q\) be the type of \(q\). The plan
contains a finite set
\[
  \mathcal R_{t,\ell}(q)\subseteq[\![\tau_q]\!].
\]
It contains every value taken by \(q\) when a step starts in \(S_t\) with an
input in \(U_M\), including values computed before a failure. The quantities
comprise variables, expression results, array indices and elements, outputs
and loop trip counts. A set may be empty when its occurrence is unreachable.

\paragraph{Encodings.}
For each semantic type \(\tau\) used by \(M\), collect its finite internal value
sets, the component values occurring in \(S_t\), \(U_M\), and \(Y_M\), every
typed literal in \(c\), and the finitely many index constants introduced by the
planned loop expansions. Their union is finite. The plan gives it one injective
fixed-width bit-vector encoding \(\mathrm{enc}_\tau\), so a value has the same
encoding at every program point.

\paragraph{Primitive exactness.}
For an occurrence of the registered primitive \(p(q_1,\ldots,q_m)\) at
\(\ell\), where \(q_j\) is its \(j\)-th operand, write
\(R_j=\mathcal R_{t,\ell}(q_j)\) and
\(R_{\mathrm{out}}=\mathcal R_{t,\ell}(p(q_1,\ldots,q_m))\), and let
\(\tau_j\) and \(\tau_{\mathrm{out}}\) be their types. Its semantic result must
remain in the planned output set:
\[
  \{[\![p]\!](v_1,\ldots,v_m)
      \mid (v_1,\ldots,v_m)\in R_1\times\cdots\times R_m\}
  \subseteq R_{\mathrm{out}}.
\]
The plan also provides a bit-vector term \(\widehat p_{t,\ell}\) satisfying
\[
  \widehat p_{t,\ell}
    \bigl(\mathrm{enc}_{\tau_1}(v_1),\ldots,
          \mathrm{enc}_{\tau_m}(v_m)\bigr)
  =
  \mathrm{enc}_{\tau_{\mathrm{out}}}
    \bigl([\![p]\!](v_1,\ldots,v_m)\bigr)
\]
for every \((v_1,\ldots,v_m)\in R_1\times\cdots\times R_m\). Boolean
connectives, array packing and array update use their exact solver operations.
The solver term \(\operatorname{ite}(b,v_T,v_F)\) returns \(v_T\) when \(b\) is
true and \(v_F\) otherwise. A plan is \emph{valid} if all the above conditions
hold.
\end{definition}

\begin{lemma}[Plan-induced finite representation]
\label{lem:from-absorbed}
Given a valid plan \(P\), put
\[
  \begin{aligned}
  I_D
    &=I_M^0\cap S_0,\\
  U_t^D&=U_M,\qquad Y_t^D=Y_M\cup\{\bot_Y\} &&(t<N).
  \end{aligned}
\]
Together with the full state layers \((X_t^D)_{t\le N}\) defined above, these
sets and initial carrier form a finite layered representation of \(D\), with
layered system
\(\mathcal A_D\). Moreover,
\[
  \mathrm{Tr}(\mathcal A_D)=\mathrm{Tr}(D),
\]
so \(\mathcal A_D\) and \(D\) have the same reachable goals.
\end{lemma}

\begin{proof}
Each \(S_t\) is finite by validity of \(P\), and
\(\mathit{Err}_M\), \(U_M\), and \(Y_M\) are finite by
Definition~\ref{def:ir}. Hence every \(X_t^D\), \(U_t^D\), and \(Y_t^D\) is
finite, and \(I_D\) is finite because \(I_D\subseteq S_0\).
Every trace of \(D\) starts in \(I_M^0\), and the initial-layer condition of
\(P\) places its initial state in \(S_0\). Therefore
\[
  \{x_0\mid
      (\mathbf x,\mathbf u,\mathbf y,\mathbf H)\in\mathrm{Tr}(D)\}
  \subseteq I_D\subseteq I_M^0\cap X_0^D.
\]
Fix \(t<N\), \(x\in X_t^D\), and \(u\in U_t^D=U_M\). If \(x\in S_t\), a
normally completing step reaches \(S_{t+1}\) by the state-layer closure of
\(P\) and produces an output in \(Y_M\). A failing step reaches some
\(\varepsilon\in\mathit{Err}_M\) and produces \(\bot_Y\). If
\(x\in\mathit{Err}_M\), absorption keeps the same error state and again
produces \(\bot_Y\). In every case,
\[
  \bar\delta_M(x,u)
  \in X_{t+1}^D\times Y_t^D\times 2^{\Gamma_M}.
\]
This proves layer closure. Starting from the initial inclusion above, induction
on \(t\) gives the remaining trace-support conditions. The displayed carriers
therefore form a finite layered representation of \(D\). Finally,
Lemma~\ref{lem:trace-preservation} gives the trace equality and equality of the
reachable-goal sets.
\end{proof}

\paragraph{State, input and output encoding.}
A normal state, input or output is a typed valuation. Encode it by concatenating
the type encodings of its components in their fixed context order. Disjoint
tags distinguish normal states from terminal errors and normal outputs from
\(\bot_Y\). A terminal error also carries an injective encoding of its name.
Any otherwise empty boundary encoding receives one fixed padding bit. This
gives the injective maps required by
Definition~\ref{def:exact-enc}, with positive bit widths
\(d_t^X\), \(d_t^U\), and \(d_t^Y\):
\[
\begin{gathered}
  \bigl(\mathrm{enc}^X_t:X_t^D\to\mathbb B^{d_t^X}\bigr)_{t\le N},
  \qquad
  \bigl(\mathrm{enc}^U_t:U_M\to\mathbb B^{d_t^U}\bigr)_{t<N},\\
  \bigl(\mathrm{enc}^Y_t:(Y_M\cup\{\bot_Y\})
        \to\mathbb B^{d_t^Y}\bigr)_{t<N}.
\end{gathered}
\]
Let \(\operatorname{Pack}^X_r\) and \(\operatorname{Pack}^Y_r\) denote the
fixed wiring terms that concatenate the selected component encodings and add
the normal-state or normal-output tag. Their defining identities are
\[
\begin{aligned}
  \operatorname{Pack}^X_r
    \bigl((\mathrm{enc}_{\Theta_S(z)}(s(z)))_{z\in\operatorname{dom}(\Theta_S)}\bigr)
    &=\mathrm{enc}^X_r(s)
      &&(r\le N,\ s\in S_r),\\
  \operatorname{Pack}^Y_r
    \bigl((\mathrm{enc}_{\Theta_O(z)}(y(z)))_{z\in\operatorname{dom}(\Theta_O)}\bigr)
    &=\mathrm{enc}^Y_r(y)
      &&(r<N,\ y\in Y_M).
\end{aligned}
\]
Their inverses on their images and their validity predicates are defined as in
Definition~\ref{def:exact-enc}. The semantic initial-state condition is
\(x_0\in I_D\). Its encoded form is
\[
  \widehat I_D(\widehat x_0)
  \equiv
  \bigvee_{s\in I_D}
    \widehat x_0=\mathrm{enc}^X_0(s).
\]
Thus \(\widehat I_D(\widehat x_0)\) holds exactly when the solver variable
\(\widehat x_0\) equals the encoding of some initial state in \(I_D\).
Define \(\widehat K_N\) by enumerating exactly the valid encoded sequences
whose decoded sequences satisfy \(K_N\). This enumeration is finite because
all carriers are finite.

\begin{definition}[Plan-induced command and step translation]
\label{def:symstep}
Fix a valid plan \(P\) and \(t<N\). For each expression occurrence \(e\), let
\(\widehat e_t(\widehat\rho)\) be its bit-vector translation. A symbolic
valuation \(\widehat\rho\) assigns every variable in \(\Theta_M\) a bit vector
of the width fixed for its declared type, including variables not yet defined by
the concrete execution. Write \(\widehat{\mathsf{select}}\) and
\(\widehat{\mathsf{update}}\) for exact encoded array selection and update. The
translation is
\[
\begin{aligned}
  \widehat z_t(\widehat\rho)
    &=\widehat\rho(z),\\
  \widehat{\mathit{lit}}_t(\widehat\rho)
    &=\mathrm{enc}_{\tau_{\mathit{lit}}}
        (\mathsf{val}(\mathit{lit})),\\
  \widehat{e_1[e_2]}_t(\widehat\rho)
    &=\widehat{\mathsf{select}}
        \bigl(\widehat{e_1}_t(\widehat\rho),
              \widehat{e_2}_t(\widehat\rho)\bigr),\\
  \widehat{p(e_1,\ldots,e_m)}_t(\widehat\rho)
    &=\widehat p_{t,\ell}
        \bigl(\widehat{e_1}_t(\widehat\rho),\ldots,
              \widehat{e_m}_t(\widehat\rho)\bigr),
\end{aligned}
\]
where \(\ell\) is the primitive occurrence. If every value read by \(e\) is
defined in \(\rho\) and \(\widehat\rho\) encodes those defined components, then
\[
  \widehat e_t(\widehat\rho)
  =\mathrm{enc}_{\tau_e}([\![e]\!]_\rho).
\]
For each command occurrence \(c'\) in the complete body \(c\), define
\[
  \operatorname{SymC}_{t}
  (c';\widehat\rho,\pi)
  =(\widehat\kappa,\widehat\rho',
    \widehat\varepsilon,\mathbf h).
\]
Here \(\pi\) states that \(c'\) is reached, and \(\widehat\kappa\) states that it
is reached and completes normally. The failure component is used only when
\(\pi\wedge\neg\widehat\kappa\) holds. Let \(\star\) be an arbitrary fixed
vector of the required width, let \(\mathbf 0\) be the all-false hit vector,
and let \(\mathbf h_{\{\gamma\}}\) be the one-hot vector for \(\gamma\).
Boolean operations and \(\operatorname{ite}\) act componentwise on valuations
and hit vectors. The primitive commands translate as
\[
\begin{aligned}
  \operatorname{SymC}_t(\mathsf{skip};\widehat\rho,\pi)
    &=(\pi,\widehat\rho,\star,\mathbf 0),\\
  \operatorname{SymC}_t(z:=e;\widehat\rho,\pi)
    &=\Bigl(\pi,
      \widehat\rho\bigl[z\mapsto
        \operatorname{ite}(\pi,\widehat e_t(\widehat\rho),
                            \widehat\rho(z))\bigr],
      \star,\mathbf 0\Bigr),\\
  \operatorname{SymC}_t(z[e_1]:=e_2;\widehat\rho,\pi)
    &=\Bigl(\pi,
      \widehat\rho\bigl[z\mapsto
        \operatorname{ite}\bigl(\pi,
          \widehat{\mathsf{update}}(
            \widehat\rho(z),\widehat{e_1}_t(\widehat\rho),
            \widehat{e_2}_t(\widehat\rho)),\widehat\rho(z)\bigr)\bigr],
      \star,\mathbf 0\Bigr),\\
  \operatorname{SymC}_t(\mathsf{mark}(\gamma);\widehat\rho,\pi)
    &=(\pi,\widehat\rho,\star,
       \pi\wedge\mathbf h_{\{\gamma\}}),\\
  \operatorname{SymC}_t(
      \mathsf{fail}(\varepsilon,\gamma_s);\widehat\rho,\pi)
    &=(\mathsf{false},\widehat\rho,
       \mathrm{enc}_{t+1}^X(\varepsilon),
       \pi\wedge\mathbf h_{\{\gamma_s\}}).
\end{aligned}
\]
For sequence, let
\[
\begin{aligned}
  (\kappa_1,\widehat\rho_1,\widehat\varepsilon_1,\mathbf h_1)
    &=\operatorname{SymC}_t(c_1;\widehat\rho,\pi),\\
  (\kappa_2,\widehat\rho_2,\widehat\varepsilon_2,\mathbf h_2)
    &=\operatorname{SymC}_t(c_2;\widehat\rho_1,\kappa_1).
\end{aligned}
\]
Then
\[
  \operatorname{SymC}_t(c_1;c_2;\widehat\rho,\pi)
  =\bigl(\kappa_2,\widehat\rho_2,
      \operatorname{ite}(\kappa_1,\widehat\varepsilon_2,
                          \widehat\varepsilon_1),
      \mathbf h_1\vee\mathbf h_2\bigr).
\]
For a conditional, put \(g=\widehat e_t(\widehat\rho)\) and
\[
\begin{aligned}
  (\kappa_T,\widehat\rho_T,\widehat\varepsilon_T,\mathbf h_T)
    &=\operatorname{SymC}_t(c_T;\widehat\rho,\pi\wedge g),\\
  (\kappa_F,\widehat\rho_F,\widehat\varepsilon_F,\mathbf h_F)
    &=\operatorname{SymC}_t(c_F;\widehat\rho,\pi\wedge\neg g).
\end{aligned}
\]
The conditional translation is
\[
  \operatorname{SymC}_t(
    \mathsf{if}\ e\ \mathsf{then}\ c_T\ \mathsf{else}\ c_F;
    \widehat\rho,\pi)
  =\bigl(
    \operatorname{ite}(g,\kappa_T,\kappa_F),
    \operatorname{ite}(g,\widehat\rho_T,\widehat\rho_F),
    \operatorname{ite}(g,\widehat\varepsilon_T,\widehat\varepsilon_F),
    \mathbf h_T\vee\mathbf h_F\bigr).
\]
For a loop occurrence \(\ell\), let \(m\) be the nonnegative trip count,
\(\tau_m\) its type, and \(\widehat m\) its exact encoded value at loop entry.
Define the finite comparison predicate
\[
  \widehat{\mathsf{lt}}_{t,\ell}(k,\widehat m)
  :=\bigvee_{\substack{n\in\mathcal R_{t,\ell}(m)\\k<n}}
      \bigl(\widehat m=\mathrm{enc}_{\tau_m}(n)\bigr),
\]
where an empty disjunction is false, and set
\[
  L_{t,\ell}=\max\bigl(\{0\}\cup\mathcal R_{t,\ell}(m)\bigr).
\]
Initialize
\(a_0=\pi\), \(\widehat\rho_0=\widehat\rho\),
\(\widehat\varepsilon_0=\star\), and \(\mathbf h_0=\mathbf 0\). For
\(k<L_{t,\ell}\), define
\[
\begin{gathered}
  b_k=a_k\wedge\widehat{\mathsf{lt}}_{t,\ell}(k,\widehat m),\\
  (\kappa_k,\widehat\rho_{k+1},
   \widehat\varepsilon'_k,\mathbf h'_k)
  =\operatorname{SymC}_t(c[j:=k];\widehat\rho_k,b_k),\\
  a_{k+1}=a_k\wedge
    \bigl(\neg\widehat{\mathsf{lt}}_{t,\ell}(k,\widehat m)
          \vee\kappa_k\bigr),\\
  \widehat\varepsilon_{k+1}
  =\operatorname{ite}(b_k\wedge\neg\kappa_k,
                      \widehat\varepsilon'_k,
                      \widehat\varepsilon_k),
  \qquad
  \mathbf h_{k+1}=\mathbf h_k\vee\mathbf h'_k.
\end{gathered}
\]
The loop translation is
\[
  \operatorname{SymC}_t(
    \mathsf{for}\ j<e\ \mathsf{do}\ c;\widehat\rho,\pi)
  = (a_{L_{t,\ell}},\widehat\rho_{L_{t,\ell}},
     \widehat\varepsilon_{L_{t,\ell}},\mathbf h_{L_{t,\ell}}).
\]
It emits
no copies for an empty range, executes exactly the first \(m\) copies when the
loop is reached, and disables every later copy after a failure.
At the step boundary, let
\[
  \widehat x_t=\mathrm{enc}^X_t(s_t),\qquad
  \widehat u_t=\mathrm{enc}^U_t(u_t),\qquad
  \rho_{\mathrm{in},t}=s_t\cup u_t
  \quad(s_t\in S_t,\ u_t\in U_M).
\]
For each local or output variable \(z\), fix an arbitrary bit vector
\(\widehat\bot_z\) of that width. The total symbolic input valuation is
\[
  \widehat\rho_{\mathrm{in},t}(z)=
  \begin{cases}
    \mathrm{enc}_{\Theta_S(z)}(s_t(z))
      &z\in\operatorname{dom}(\Theta_S),\\
    \mathrm{enc}_{\Theta_I(z)}(u_t(z))
      &z\in\operatorname{dom}(\Theta_I),\\
    \widehat\bot_z
      &z\in\operatorname{dom}(\Theta_L\cup\Theta_O).
  \end{cases}
\]
Only active reads must encode defined concrete values. The fixed vectors make
inactive symbolic branches well defined and do not affect an active result.
For \(t<N-1\), normal execution determines the next input valuation by
\[
  \langle c,\rho_{\mathrm{in},t}\rangle
    \Downarrow\langle\rho_{f,t},H_t\rangle
  \quad\Longrightarrow\quad
  \rho_{\mathrm{in},t+1}
  =\bigl(\rho_{f,t}|_{\Theta_S}\bigr)\cup u_{t+1}.
\]
Translating the complete body gives
\[
  (\widehat\kappa_t,\widehat\rho_{f,t},
    \widehat\varepsilon_t,\mathbf h_t)
  =\operatorname{SymC}_t
    (c;\widehat\rho_{\mathrm{in},t},\mathsf{true}).
\]
The symbolic step projects the final valuation onto its state and output
components and repacks them:
\[
\begin{aligned}
  \operatorname{SymStep}_{M,t}(\widehat x_t,\widehat u_t)
  &=(\widehat x_{t+1},\widehat y_t,\mathbf h_t),\\
  (\widehat x_{t+1},\widehat y_t,\mathbf h_t)
  &=
  \Bigl(
    \operatorname{ite}\bigl(
      \widehat\kappa_t,
      \operatorname{Pack}^X_{t+1}(\widehat\rho_{f,t}|_{\Theta_S}),
      \widehat\varepsilon_t\bigr),
    \operatorname{ite}\bigl(
      \widehat\kappa_t,
      \operatorname{Pack}^Y_t(\widehat\rho_{f,t}|_{\Theta_O}),
      \mathrm{enc}^Y_t(\bot_Y)\bigr),
    \mathbf h_t
  \Bigr).
\end{aligned}
\]
Thus
\[
  \operatorname{SymStep}_{M,t}:
  \mathbb B^{d_t^X}\times\mathbb B^{d_t^U}
  \longrightarrow
  \mathbb B^{d_{t+1}^X}\times
  \mathbb B^{d_t^Y}\times\mathsf{Bool}^{\Gamma_M}.
\]
\end{definition}

\begin{lemma}[Symbolic-step correctness]
\label{lem:cmd}
For \(s\in S_t\) and \(u\in U_M\),
\(\operatorname{SymStep}_{M,t}\) exactly encodes \(\bar\delta_M\) at layer
\(t\). Namely, if \(\bar\delta_M(s,u)=(s',y,H)\), then
\[
  \operatorname{SymStep}_{M,t}
  \bigl(\mathrm{enc}^X_t(s),
        \mathrm{enc}^U_t(u)\bigr)
  =
  \bigl(\mathrm{enc}^X_{t+1}(s'),
        \mathrm{enc}^Y_t(y),\mathbf h_H\bigr),
\]
where \((\mathbf h_H)_\gamma=\mathsf{true}\) exactly when \(\gamma\in H\).
\end{lemma}

\begin{proof}
Let \(\rho_{\mathrm{in}}=s\cup u\). We prove a stronger claim by structural
induction on each command occurrence \(c'\). Fix any path condition \(\pi\)
and any assignment under which \(\widehat\rho\) encodes the defined components
of \(\rho\). If
\[
  (\widehat\kappa,\widehat\rho',
   \widehat\varepsilon,\mathbf h)
  =\operatorname{SymC}_t(c';\widehat\rho,\pi),
\]
then
\[
\begin{array}{ll}
  \pi=\mathsf{false}
  &\Longrightarrow
    \widehat\kappa=\mathsf{false},\quad
    \widehat\rho'=\widehat\rho,\quad
    \mathbf h=\mathbf 0,\\[2pt]
  \pi=\mathsf{true},\quad
  \langle c',\rho\rangle\Downarrow\langle\rho'',H\rangle
  &\Longrightarrow
    \widehat\kappa=\mathsf{true},\quad
    \widehat\rho'=\widehat{\rho''},\quad
    \mathbf h=\mathbf h_H,\\[2pt]
  \pi=\mathsf{true},\quad
  \langle c',\rho\rangle\Downarrow\mathsf{Err}(\varepsilon,H)
  &\Longrightarrow
    \widehat\kappa=\mathsf{false},\quad
    \widehat\varepsilon=\mathrm{enc}^X_{t+1}(\varepsilon),\quad
    \mathbf h=\mathbf h_H.
\end{array}
\]
Only defined valuation components are compared. Internal-range containment
makes every reached value encodable, and primitive exactness makes every
translated expression equal to its concrete value. The displayed primitive
translations give the base cases. Sequence activates its second command
exactly after normal completion of the first and unions their hits. Exact guard
evaluation activates the same conditional arm as the concrete execution. For a
reached loop, range containment gives \(m\le L_{t,\ell}\). Induction on \(k\)
shows that \(a_k\) holds exactly while the first \(\min(k,m)\) copies have
completed normally, with \(\widehat\rho_k\) and \(\mathbf h_k\) encoding their
valuation and accumulated hits. A failing copy records its error and disables
every later copy. Hence the loop agrees with \(c^{[m]}\). This proves the
stronger claim. Apply it to the complete body with \(\pi=\mathsf{true}\) and
\(\rho=\rho_{\mathrm{in}}\). Normal execution gives
\[
  \bar\delta_M(s,u)
  =\bigl(\rho'|_{\Theta_S},\rho'|_{\Theta_O},H\bigr).
\]
Failure with \(\mathsf{Err}(\varepsilon,H)\) gives
\[
  \bar\delta_M(s,u)=(\varepsilon,\bot_Y,H).
\]
In the first case, the two packing identities turn the state and output
projections of \(\widehat{\rho'}\) into the encodings of \(s'\) and \(y\). In
the second, the failure component is already \(\mathrm{enc}^X_{t+1}(s')\), and
the output branch selects \(\mathrm{enc}^Y_t(\bot_Y)\). The two implications
therefore yield the claimed equality.
\end{proof}

\paragraph{Encoded transition.}
The symbolic step starts from a normal state and may return either a normal or
terminal successor. To also encode a step whose current state is already
terminal, define
\[
\begin{aligned}
  &\widehat{\delta}_t^D
  (\widehat x_t,\widehat u_t,\widehat x_{t+1},
   \widehat y_t,\mathbf h_t)\\
  &\quad\equiv
  \Big[
    \Big(\bigvee_{s\in S_t}
      \widehat x_t=\mathrm{enc}^X_t(s)\Big)
    \wedge
    (\widehat x_{t+1},\widehat y_t,\mathbf h_t)
      =\operatorname{SymStep}_{M,t}(\widehat x_t,\widehat u_t)
  \Big]\\
  &\qquad{}\vee{}
  \bigvee_{\varepsilon\in\mathit{Err}_M}
  \Big[
    \widehat x_t=\mathrm{enc}^X_t(\varepsilon)
    \wedge\widehat x_{t+1}=\mathrm{enc}^X_{t+1}(\varepsilon)\\
  &\qquad\qquad
    \wedge\widehat y_t=\mathrm{enc}^Y_t(\bot_Y)
    \wedge\!\bigwedge_{\gamma\in\Gamma_M}\neg h_{t,\gamma}
  \Big].
\end{aligned}
\]
The first disjunct handles normal execution. The second re-encodes the same
terminal error at layer \(t+1\), emits \(\bot_Y\), and records no new hits.

\begin{theorem}[Exact Behavior IR encoding]
\label{thm:exact-b}
Let \(P\) be a valid Behavior representation plan for \((M,D)\), and let
\(\mathcal A_D\) be the finite layered system induced by \(P\) in
Lemma~\ref{lem:from-absorbed}. The boundary encoding maps
\((\mathrm{enc}_t^X)_{t\le N}\),
\((\mathrm{enc}_t^U)_{t<N}\), and
\((\mathrm{enc}_t^Y)_{t<N}\), together with their inverse maps, validity
predicates, \(\widehat I_D\), \(\widehat K_N\), and
\((\widehat{\delta}_t^D)_{t<N}\), form an exact bounded encoding of
\(\mathcal A_D\) as defined in Definition~\ref{def:exact-enc}.
\end{theorem}

\begin{proof}
The component encodings are injective by validity of \(P\). Fixed-order
concatenation and the disjoint tags therefore make every boundary map
injective. Each inverse is well defined on its image, which the corresponding
validity predicate recognizes exactly. The finite enumerations defining
\(\widehat I_D\) and \(\widehat K_N\) accept exactly the encodings of states
in \(I_D\) and sequences satisfying \(K_N\), respectively.
Fix \(t<N\) and a valid encoded tuple
\((\widehat x,\widehat u,\widehat x',\widehat y,\mathbf h)\). Decode its first
four components as \(s,u,s',y\), and let
\(H=\{\gamma\in\Gamma_M\mid h_\gamma=\mathsf{true}\}\). The transition of
\(\mathcal A_D\) at layer \(t\) is \(\bar\delta_M\) restricted to
\(X_t^D\times U_M\). It remains to show that
\(\widehat\delta_t^D\) accepts exactly its encoded transitions:
\[
  \widehat\delta_t^D
    (\widehat x,\widehat u,\widehat x',\widehat y,\mathbf h)
  \iff
  \bar\delta_M(s,u)=(s',y,H).
\]
Because \(X_t^D=S_t\cup\mathit{Err}_M\), the decoded state \(s\) lies in
exactly one of these two sets. If \(s\in S_t\), only the first disjunct can
hold. Lemma~\ref{lem:cmd} states that \(\operatorname{SymStep}_{M,t}\) returns
the encoding of \(\bar\delta_M(s,u)\). The state and output encodings are
injective, and \(\mathbf h_H\) identifies \(H\) componentwise. Hence this
encoded result equals \((\widehat x',\widehat y,\mathbf h)\) exactly when
\(\bar\delta_M(s,u)=(s',y,H)\). If
\(s=\varepsilon\in\mathit{Err}_M\), only the second disjunct can hold, and it
is equivalent to
\((s',y,H)=(\varepsilon,\bot_Y,\emptyset)\). This is exactly
\(\bar\delta_M(\varepsilon,u)\) by absorption. Hence the encoded transition
relation satisfies the remaining condition of Definition~\ref{def:exact-enc}.
\end{proof}

\subsection{Behavior Frontend Fidelity and Audits}
\label{app:bs-assumptions}

\paragraph{Overview.}
Subsections~\ref{app:bs-ir} and~\ref{app:bs-encoding} define the Behavior IR
and its exact bounded encoding. This subsection connects the source execution
of \(\hat B\) to that IR through assumptions on step semantics, trace
admissibility, the goal catalogue and replay.

\begin{definition}[Source-to-IR correspondence]
\label{def:source-ir}
Let \(\hat B\) be an admitted Behavior model. Its source execution has state,
input, output and error sets
\(X_{\hat B},U_{\hat B},Y_{\hat B}\) and \(\mathit{Err}_{\hat B}\), nonempty
initial set \(\emptyset\neq I_{\hat B}^0\subseteq X_{\hat B}\), and the goal
catalogue \(\Gamma_{\hat B}\) of Definition~\ref{def:goals}. Its fallible step is
\[
  \delta_{\hat B}^0:X_{\hat B}\times U_{\hat B}
  \longrightarrow
  \bigl(X_{\hat B}\times Y_{\hat B}
          \times2^{\Gamma_{\hat B}}\bigr)
  \cup
  \{\mathsf{Err}(\varepsilon,H)
      \mid \varepsilon\in\mathit{Err}_{\hat B},\
             H\subseteq\Gamma_{\hat B}\}.
\]
For a stimulus with \(|\sigma|=N\), the harness constructs
\(u_t^{\hat B}\in U_{\hat B}\) for \(t<N\) from
\((\sigma_t,a_t^{\hat B}(\xi))\). Thus the source input includes both the
accepted transaction batch and the schedule events visible to the Behavior
semantics. Let \(M\) be the Behavior IR procedure obtained by lowering
\(\hat B\), as in Definition~\ref{def:ir}. Admission requires \(M\) to be well
formed and lowering to supply the source-to-IR maps
\[
  \mathrm{low}_X:X_{\hat B}\to X_M,\qquad
  \mathrm{low}_U:U_{\hat B}\to U_M,\qquad
  \mathrm{low}_Y:Y_{\hat B}\to Y_M,
\]
and
\[
  \mathrm{low}_E:\mathit{Err}_{\hat B}\to\mathit{Err}_M,
  \qquad
  \mathrm{low}_\Gamma:\Gamma_{\hat B}\to\Gamma_M.
\]
The first four maps relate source values to IR values, while
\(\mathrm{low}_\Gamma\) maps source goals to their corresponding IR goals.
The \(\mathrm{enc}\) maps of Definition~\ref{def:exact-enc} then encode the IR
values as solver bit vectors. Extend \(\mathrm{low}_\Gamma\) pointwise to goal
sets. The component maps induce the following map \(\mathrm{low}\) on step
results:
\[
\begin{aligned}
  \mathrm{low}(x',y,H)
    &=(\mathrm{low}_X(x'),\mathrm{low}_Y(y),\mathrm{low}_\Gamma(H)),\\
  \mathrm{low}\bigl(\mathsf{Err}(\varepsilon,H)\bigr)
    &=\mathsf{Err}\bigl(\mathrm{low}_E(\varepsilon),
                         \mathrm{low}_\Gamma(H)\bigr).
\end{aligned}
\]
Let \(\bot_{\hat B}\) be the padding output of the source absorption. Extend
\(\mathrm{low}_X\) to terminal errors using \(\mathrm{low}_E\), set
\(\mathrm{low}_Y(\bot_{\hat B})=\bot_Y\), and apply \(\mathrm{low}\)
componentwise to absorbed traces. For a set \(W\) of source traces, write
\(\mathrm{low}(W)=\{\mathrm{low}(w)\mid w\in W\}\). For the IR domain \(D\)
fixed in Subsection~\ref{app:bs-encoding}, write
\[
  D_M:=D=\bigl(\mathcal M_M(I_M^0),N,K_N\bigr).
\]
Let \(K_N^{\hat B}\) be the source trace predicate and set
\[
  D_{\hat B}^{\mathrm{src}}
  =\bigl(\mathrm{Abs}_{I_{\hat B}^0}(\delta_{\hat B}^0),
          N,K_N^{\hat B}\bigr).
\]
\end{definition}

\begin{assumption}[Behavior semantic fidelity]
\label{asm:lb}
For every \(s\in X_{\hat B}\) and \(u\in U_{\hat B}\),
\[
  \delta_M^0\bigl(\mathrm{low}_X(s),\mathrm{low}_U(u)\bigr)
  =\mathrm{low}\bigl(\delta_{\hat B}^0(s,u)\bigr).
\]
Thus lowering preserves the normal successor and output, or the failure,
together with the emitted goal set. The absorbed extensions in
Definition~\ref{def:source-ir} preserve the same correspondence after a
failure.
\end{assumption}

\begin{assumption}[Behavior catalogue fidelity]
\label{asm:gb}
The goal policy and lowering induce a bijection
\[
  \mathrm{low}_\Gamma:\Gamma_{\hat B}\xrightarrow{\ \sim\ }\Gamma_M,
\]
so every source goal has exactly one corresponding IR goal and every IR goal
arises in this way. A source goal \(\gamma\) is hit on a source step exactly
when \(\mathrm{low}_\Gamma(\gamma)\) is hit on the corresponding IR step. The
bijection preserves the goal classes:
\[
  \mathrm{low}_\Gamma(\Gamma_{\hat B}^{\mathrm{cov}})
    =\Gamma_M^{\mathrm{cov}},
  \qquad
  \mathrm{low}_\Gamma(\Gamma_{\hat B}^{\mathrm{safe}})
    =\Gamma_M^{\mathrm{safe}}.
\]
\end{assumption}

\begin{assumption}[Behavior domain and replay fidelity]
\label{asm:rb}
The lowering maps preserve the initial states and inputs:
\[
  \mathrm{low}_X(I_{\hat B}^0)=I_M^0,
  \qquad
  \mathrm{low}_U(U_{\hat B})=U_M.
\]
Thus every IR initial state and input has at least one source representative.
For every source state, input and output sequence of lengths \(N+1\), \(N\)
and \(N\), respectively,
\[
  K_N^{\hat B}(\mathbf x,\mathbf u,\mathbf y)
  \iff
  K_N\bigl(\mathrm{low}_X(\mathbf x),
            \mathrm{low}_U(\mathbf u),
            \mathrm{low}_Y(\mathbf y)\bigr),
\]
where
\[
  \mathrm{low}_X(\mathbf x)
    =\bigl(\mathrm{low}_X(x_t)\bigr)_{t\le N},\qquad
  \mathrm{low}_U(\mathbf u)
    =\bigl(\mathrm{low}_U(u_t)\bigr)_{t<N},\qquad
  \mathrm{low}_Y(\mathbf y)
    =\bigl(\mathrm{low}_Y(y_t)\bigr)_{t<N}.
\]
Every raw record returned by
\(\textsc{Run}_{D_M}(x_0,\mathbf u)\) carries
\(s_0\in I_{\hat B}^0\) and
\(\mathbf v\in U_{\hat B}^N\) such that
\[
  \mathrm{low}_X(s_0)=x_0,
  \qquad
  \mathrm{low}_U(v_t)=u_t \quad(t<N).
\]
It records the complete source replay from \((s_0,\mathbf v)\), and the IR
trace in the record is obtained by applying \(\mathrm{low}\) to that replay.
\end{assumption}

\begin{corollary}[Behavior trace and goal correspondence]
\label{cor:source}
Under Assumptions~\ref{asm:lb}, \ref{asm:gb} and~\ref{asm:rb},
\[
  \mathrm{low}\bigl(\mathrm{Tr}(D_{\hat B}^{\mathrm{src}})\bigr)
  =\mathrm{Tr}(D_M),
  \qquad
  \mathrm{low}_\Gamma\bigl(
    \mathrm{Reach}_{D_{\hat B}^{\mathrm{src}}}\bigr)
  =\mathrm{Reach}_{D_M}.
\]
If \(P\) is a valid Behavior representation plan for \((M,D_M)\), the
hypotheses of Theorem~\ref{thm:closure-sound} hold for \(D_M\), and
Algorithm~\ref{alg:closure} returns
\((\mathit{Rch}_M,\mathit{Unr}_M,\mathit{Unk}_M,W)\) for
\(G_M\subseteq\Gamma_M\). Define
\[
\begin{aligned}
  G_{\hat B}
    &=\mathrm{low}_\Gamma^{-1}(G_M),\\
  (\mathit{Rch}_{\hat B},\mathit{Unr}_{\hat B},\mathit{Unk}_{\hat B})
    &=\bigl(
    \mathrm{low}_\Gamma^{-1}(\mathit{Rch}_M),
    \mathrm{low}_\Gamma^{-1}(\mathit{Unr}_M),
    \mathrm{low}_\Gamma^{-1}(\mathit{Unk}_M)
    \bigr).
\end{aligned}
\]
These three sets partition \(G_{\hat B}\), with
\[
  \mathit{Rch}_{\hat B}\subseteq
    \mathrm{Reach}_{D_{\hat B}^{\mathrm{src}}},
  \qquad
  \mathit{Unr}_{\hat B}\cap
    \mathrm{Reach}_{D_{\hat B}^{\mathrm{src}}}=\emptyset,
\]
and every goal in \(\mathit{Rch}_{\hat B}\) has a source replay witness.
\end{corollary}

\begin{proof}
Step induction using Assumptions~\ref{asm:lb} and~\ref{asm:rb} gives
\[
\begin{aligned}
  w_{\hat B}\in\mathrm{Tr}(D_{\hat B}^{\mathrm{src}})
  &\Longrightarrow
  \mathrm{low}(w_{\hat B})\in\mathrm{Tr}(D_M),\\
  w_M\in\mathrm{Tr}(D_M)
  &\Longrightarrow
  \exists w_{\hat B}\in\mathrm{Tr}(D_{\hat B}^{\mathrm{src}}):
  \mathrm{low}(w_{\hat B})=w_M.
\end{aligned}
\]
The reverse implication chooses the source representatives guaranteed by
Assumption~\ref{asm:rb}. Hence the trace sets are equal after applying
\(\mathrm{low}\). Hit preservation and the catalogue bijection give, for every
\(\gamma\in\Gamma_{\hat B}\),
\[
  \gamma\in\mathrm{Reach}_{D_{\hat B}^{\mathrm{src}}}
  \iff
  \mathrm{low}_\Gamma(\gamma)\in\mathrm{Reach}_{D_M}.
\]
A valid \(P\) supplies the exact instance used by
Theorem~\ref{thm:closure-sound}. Applying
\(\mathrm{low}_\Gamma^{-1}\) to that theorem gives the stated classification.
The source replays carried by the audited records provide the witnesses.
\end{proof}

\paragraph{Catalogue and replay audits.}
Before any query, the catalogue audit checks the goal bijection, classes and
hit correspondence. For each raw record, the replay audit runs its source
representatives, checks that the resulting source and IR traces correspond and
satisfy their trace predicates, and applies Definition~\ref{def:replay-audit}
with \(\widehat{\delta}_t^{D_M}\). Only audited traces are retained.

\subsection{Artifact Instances and Reported Bounded Conclusions}
\label{app:bs-artifact-reporting}

\paragraph{Overview.}
The preceding subsections construct exact bounded instances for Behavior and
RTL. This subsection pairs each instance with the audited traces from concrete
evaluation and states when the common analysis may report a goal as reached,
bounded-unreachable or unresolved.

\paragraph{Analysis domains.}
Behavior is analyzed through its lowered IR domain \(D_M\), whereas RTL is
analyzed directly in \(D_{\hat R}\). Accordingly, for
\(A\in\{\hat B,\hat R\}\), write \(D_A=D_M\) when \(A=\hat B\), and
\(D_A=D_{\hat R}\) when \(A=\hat R\).
When the required exact instance can be constructed, a valid Behavior
representation plan and
Theorem~\ref{thm:exact-b} supply the instance for \(D_M\), while
Theorem~\ref{thm:rtl-exact} supplies the instance for \(D_{\hat R}\). The exact
Behavior instance classifies \(\Gamma_M\), and its statuses are mapped back to
\(\Gamma_{\hat B}\) through \(\mathrm{low}_\Gamma^{-1}\). The exact RTL instance
classifies \(\Gamma_{\hat R}\) directly. An RTL reset restoration claim is
reported only when its scenario goal lies in \(\mathit{Rch}\) and its violation
goal lies in \(\mathit{Unr}\).

\paragraph{Evaluated pool.}
Apply \(\textsc{Audit}_{D_A}\) to every record in \(\mathcal E_A\). A
\(\mathsf{Refused}\) or \(\mathsf{Inconsistent}\) result stops analysis, while
an \(\mathsf{Outside}\) record is discarded. Otherwise the retained records
form \(W_{\mathrm{eval}}^A\subseteq\mathrm{Tr}(D_A)\). The Behavior audit
reconstructs traces in \(D_M\), while the RTL audit reconstructs traces in
\(D_{\hat R}\). A nonempty \(W_{\mathrm{eval}}^A\) is the initial pool
\(W_0\). If it is empty, the initial-pool procedure of
Subsection~\ref{app:bs-closure} checks feasibility. Later queries may add
audited witnesses to \(W\), but they do not change \(\Xi_A\),
\(\mathcal E_A\), or the output score.

\begin{corollary}[Reported bounded conclusions]
\label{cor:reported-conclusions}
For either artifact, set
\[
  (\Gamma,\Gamma^{\mathrm{cov}},\Gamma^{\mathrm{safe}})
  =
  \begin{cases}
    (\Gamma_M,\Gamma_M^{\mathrm{cov}},\Gamma_M^{\mathrm{safe}})
      &\text{for Behavior},\\
    (\Gamma_{\hat R},\Gamma_{\hat R}^{\mathrm{cov}},
                     \Gamma_{\hat R}^{\mathrm{safe}})
      &\text{for RTL}.
  \end{cases}
\]
Under the hypotheses of
Theorem~\ref{thm:closure-sound} for \(D_A\), suppose
Algorithm~\ref{alg:closure} classifies \(\Gamma\) and returns
\[
  (\mathit{Rch},\mathit{Unr},\mathit{Unk},W).
\]
For \(C\in\{\Gamma^{\mathrm{cov}},\Gamma^{\mathrm{safe}}\}\),
\[
  \mathit{Unk}\cap C=\emptyset
  \quad\Longrightarrow\quad
  \begin{aligned}
    \mathit{Rch}\cap C&=\mathrm{Reach}_{D_A}\cap C,\\
    \mathit{Unr}\cap C&=C\setminus\mathrm{Reach}_{D_A}.
  \end{aligned}
\]
For \(C=\Gamma^{\mathrm{cov}}\), the premise
\(\mathit{Unk}\cap C=\emptyset\) establishes \emph{bounded coverage closure}.
For \(C=\Gamma^{\mathrm{safe}}\), if
\(\mathit{Unk}\cap C=\emptyset\) and \(\mathit{Rch}\cap C=\emptyset\), then
\(\mathrm{Reach}_{D_A}\cap C=\emptyset\), establishing \emph{bounded safety}.
\end{corollary}

\begin{proof}
Since the three status sets partition \(\Gamma\) and
\(\mathit{Unk}\cap C=\emptyset\),
\[
  C=(\mathit{Rch}\cap C)\cup(\mathit{Unr}\cap C).
\]
Theorem~\ref{thm:closure-sound} gives
\(\mathit{Rch}\subseteq\mathrm{Reach}_{D_A}\) and
\(\mathit{Unr}\cap\mathrm{Reach}_{D_A}=\emptyset\). The first inclusion gives
\[
  \mathit{Rch}\cap C\subseteq\mathrm{Reach}_{D_A}\cap C.
\]
Conversely, let \(\gamma\in\mathrm{Reach}_{D_A}\cap C\). The displayed
partition places \(\gamma\) in \(\mathit{Rch}\) or \(\mathit{Unr}\). The
second property excludes \(\mathit{Unr}\), so
\(\gamma\in\mathit{Rch}\cap C\). This proves the first equality. The
partition and that equality then give
\[
  \mathit{Unr}\cap C
  =C\setminus(\mathit{Rch}\cap C)
  =C\setminus\mathrm{Reach}_{D_A}.
\]
For \(C=\Gamma^{\mathrm{safe}}\), the additional premise
\(\mathit{Rch}\cap C=\emptyset\) and the first equality give
\(\mathrm{Reach}_{D_A}\cap C=\emptyset\).
\end{proof}

\subsection{Trusted Boundary, Claim Scope and Conclusion}
\label{app:bs-tcb}

\paragraph{Trusted boundary.}
The Behavior result depends on
Assumptions~\ref{asm:lb}-\ref{asm:rb} and a valid representation plan. The RTL
result depends on Assumptions~\ref{asm:lr}-\ref{asm:hr} and the exact
component encodings required by Theorem~\ref{thm:rtl-exact}. Both results rely
on correct implementations of the encoded queries, replay harness, catalogue
audit and replay audit. They also trust every unsatisfiable solver result used
to establish infeasibility or place a goal in \(\mathit{Unr}\).
Replay may expose fidelity violations, but successful checks
do not establish fidelity for every admitted trace.
The task-declared transaction interface, endpoint bindings and selected
protocol adapter are trusted to describe the intended hardware interface.

\paragraph{Claim scope.}
The output score concerns only \(\Xi_A\), while goal classification concerns
the fixed bounded domains. RTL is classified in \(D_{\hat R}\). Behavior is
analyzed in \(D_M\), with conclusions transferred to
\(D_{\hat B}^{\mathrm{src}}\) under the assumptions relating the source and IR
semantics. The results establish neither Behavior-RTL equivalence
nor properties outside the catalogues, and make no claim
beyond the fixed horizon and admissibility predicates.
Catalogue and predicate adequacy remain modeling choices.
The RTL claim covers discrete event slot executions over the admitted finite
logic domain under the fixed adapter and environment contract, not analog
timing or schedules outside that contract.

\paragraph{Conclusion.}
The construction provides a common audited bounded analysis
for admitted Behavior and RTL artifacts. Under the stated assumptions, every
reached goal has an audited witness. A goal classified as bounded-unreachable
cannot be hit within the fixed domain. If no coverage goal remains
unresolved, bounded coverage closure holds. If every safety goal is classified
as bounded-unreachable, bounded safety holds. The analysis leaves artifact
scores unchanged.


\section{Implementation, Experimental Details and Additional Results}
\label{app:exp-details}

Section~\ref{app:exp-environment} describes the agent environment and tools,
and Section~\ref{app:exp-setup} gives inference and evaluation settings.
Section~\ref{app:exp-search-coverage} examines the effect of
goal-guided search on RTL coverage.
We detail the Behavior IR-HLS comparison in Section~\ref{app:exp-hls},
pipeline comparison and PPA exploration in Section~\ref{app:exp-pipelines},
and training and self-improvement in Section~\ref{app:exp-training}.
Section~\ref{app:exp-resources} summarizes resource use and costs.
Sections~\ref{app:exp-additional} and~\ref{app:exp-self-improvement-case}
provide a modeling comparison and a self-improvement case study, respectively.

\subsection{Agent Environment and Tools}
\label{app:exp-environment}

\paragraph{Execution environment.}
The Linux namespace-based sandbox provides a persistent writable workspace, with the task
specification and public documentation available as read-only files.
Agents use \texttt{run\_command} for file operations and
shell-based development tests. Depending on the experimental setting,
additional tools support simulation comparison, PPA measurement,
and artifact submission or search completion. These interfaces are detailed
in the corresponding sections.

\paragraph{Development tools.}
Development uses Python for Behavior models and
Verilator~\citep{verilator} with cocotb~\citep{cocotb} for RTL simulation and
testing. ChipMATE and MAGE use Icarus Verilog~\citep{iverilog} during
development. The HLS experiments use AMD Vitis HLS 2025.2~\citep{amd2026vitis}
for C++-to-RTL synthesis. PPA exploration uses
Yosys/ABC~\citep{wolf2013yosys,brayton2010abc} for synthesis and technology
mapping, OpenROAD~\citep{ajayi2019openroad} for placement and netlist
optimization, and OpenSTA~\citep{opensta} for timing analysis with the
SkyWater 130\,nm high-density standard-cell
library~\citep{skywater2020pdk} and the Nangate 45\,nm standard-cell
library~\citep{nangate2008library}, which is based on the predictive FreePDK45
process design kit~\citep{stine2007freepdk}.\par

\subsection{Inference and Evaluation Settings}
\label{app:exp-setup}
\label{app:exp-models}
\label{app:exp-evaluation}
\label{app:exp-metrics}

\paragraph{Models and inference.}
Table~\ref{tab:exp-inference-config} lists model identifiers,
providers and reasoning settings.
We serve local models with SGLang~\citep{zheng2024sglang} on a single H200 node.
We retain default temperature and top-$p$ settings where configurable.
GPT-6 Astra is used for dataset construction.\par

\begin{table}[H]
\centering
\small
\caption{\textbf{Model and inference configurations.}}
\label{tab:exp-inference-config}
\setlength{\tabcolsep}{4pt}
\begin{tabular*}{\linewidth}{@{\extracolsep{\fill}}
>{\raggedright\arraybackslash}p{0.23\linewidth}
>{\raggedright\arraybackslash}p{0.39\linewidth}
>{\raggedright\arraybackslash}p{0.13\linewidth}
>{\raggedright\arraybackslash}p{0.19\linewidth}@{}}
\toprule
Model & Model ID & Access & Reasoning \\
\midrule
Qwen3.5-4B & \texttt{Qwen/Qwen3.5-4B} & SGLang & Thinking enabled \\
Qwen3.5-9B & \texttt{Qwen/Qwen3.5-9B} & SGLang & Thinking enabled \\
Qwen3.8-27B & \texttt{Qwen/Qwen3.8-27B} & SGLang & Extra-high \\
GPT-5.6 Luna & \texttt{gpt-5.6-luna} & OpenAI & Medium (default) \\
GPT-5.6 Terra & \texttt{gpt-5.6-terra} & OpenAI & Medium (default) \\
GPT-6 Astra & \texttt{gpt-6-astra} & OpenAI & Extra-high / High \\
Claude Sonnet 5 & \texttt{claude-sonnet-5} & Anthropic & Adaptive \\
Claude Fable 5.1 & \texttt{claude-fable-5-1} & Anthropic & Adaptive \\
Claude Haiku 4.5 & \texttt{claude-haiku-4-5-20251001} & Anthropic & Thinking enabled \\
DeepSeek-V4.1-Flash & \texttt{deepseek-v4p1-flash} & Fireworks & High (default) \\
\bottomrule
\end{tabular*}
\end{table}

\paragraph{Rollout settings.}
We run one rollout per task and model, using a 128K-token (131{,}072-token)
context window and at most 60 turns. When the window fills, we discard the oldest
interaction turns while retaining the task instructions and latest
tool feedback. Workspace files persist throughout the rollout.\par

\paragraph{Test configuration.}
We use 50 baseline sequences combining edge-case patterns and pseudorandom
inputs, each requesting 100 transactions per input endpoint. Tests run
independently from reset and as continuous workloads.
Benchmark scoring uses a 1200\,s total budget per artifact,
including up to 600\,s for goal-guided search.
A pass requires successful completion of all baseline tests
and no detected artifact failure. Goal-guided search may stop at its budget
without invalidating completed checks. If no artifact failure is found,
an incomplete baseline is unresolved.
Unresolved evaluations count as zero in reporting and RL rewards.
Stimulus-generation details are given in Appendix~\ref{app:bs-generation}.\par

\subsection{Effect of Goal-Guided Search on RTL Coverage}
\label{app:exp-search-coverage}

\paragraph{Search coverage.}
Table~\ref{tab:exp-search-coverage} compares average per-task RTL coverage-goal
hit rates on the same reference implementations.
Each search setting extends the same baseline with a 600\,s budget, shared
between directions for two-sided search.
Search improves coverage on all four datasets; two-sided search has the
highest overall average.

\begin{table}[H]
\centering
\small
\setlength{\tabcolsep}{4pt}
\caption{\textbf{Cumulative RTL coverage-goal hit rates (\%).}
Each search setting augments the same baseline. Avg. weights tasks equally;
bold marks column maxima.}
\label{tab:exp-search-coverage}
\begin{tabular*}{\linewidth}{@{\extracolsep{\fill}}lrrrrr@{}}
\toprule
Configuration & VerilogEval-v2 & RTLLM-v2 & CVDP-cid003 & \makecell{BEHAVE\\(Train + Eval)} & Avg. \\
\midrule
Baseline & 91.75 & 96.61 & 87.83 & 94.76 & 93.48 \\
\quad + Reference-directed & 92.23 & \textbf{97.61} & 88.88 & 95.24 & 94.06 \\
\quad + Artifact-directed & \textbf{92.24} & 97.10 & \textbf{89.10} & 95.39 & 94.16 \\
\quad + Two-sided & \textbf{92.24} & 97.43 & 88.39 & \textbf{95.59} & \textbf{94.21} \\
\bottomrule
\end{tabular*}
\end{table}

\paragraph{Example.}
RTL generated by GPT-5.6 Luna for the CVDP phase-rotation task passes the
full baseline (8,200 checked inputs). Artifact-directed search then finds
a sequence that exposes an error: the RTL reads the wait parameter
\texttt{i\_wait} after calculation instead of saving it when the operation
starts. Native replay confirms the failure, with \texttt{o\_subsampling}=1
instead of the expected 0. This example shows how search can expose bugs
in multi-stage control logic that baseline tests miss.\par

\subsection{Behavior IR and HLS Comparison Settings}
\label{app:exp-hls}

\nopagebreak[4]
\paragraph{Tasks and references.}
We select 20 tasks across seven domains whose source implementations can be
executed directly under a fixed input contract.
Both routes receive the same source code, invocation, numerical semantics
and input/output bit representations, with a target-specific interface
and a shared representation guide.
The golden Behavior models are validated against direct executions of the
source programs and remain hidden from the agent.
Figure~\ref{fig:hls-workflow} summarizes the comparison.

\begin{figure}[htbp]
\centering
\resizebox{\linewidth}{!}{
\begingroup
\providecommand{\hlsWorkflowTextColor}{black!82}
\definecolor{hlsEnv}{HTML}{5A7082}%
\definecolor{hlsRtl}{HTML}{38434C}%
\definecolor{hlsObs}{HTML}{7A858E}%
\begin{tikzpicture}[
  x=1cm, y=1cm, text=\hlsWorkflowTextColor,
  font=\normalfont\fontsize{9}{11}\selectfont,
  heading/.style={font=\normalfont\bfseries\fontsize{9}{11}\selectfont},
  small/.style={font=\normalfont\fontsize{8}{9.5}\selectfont},
  block/.style={draw=hlsEnv!85, fill=white, line width=0.6pt},
  group/.style={draw=hlsEnv!55, fill=hlsEnv!3, line width=0.6pt},
  flow/.style={-{Latex[length=1.7mm]}, draw=black!65, line width=0.8pt},
  feedback/.style={-{Latex[length=1.6mm]}, draw=hlsEnv, line width=0.7pt},
  optional/.style={feedback, densely dashed},
  wirelabel/.style={small, fill=white, inner sep=1.5pt}
]
\path[use as bounding box] (0,0) rectangle (16,4.7);
\path[draw=black!35, fill=white, line width=0.6pt]
  (0.02,0.02) rectangle (15.98,4.68);
\node[heading, anchor=west] at (0.35,4.43)
  {Behavior IR and HLS comparison};
\draw[draw=black!18, line width=0.45pt]
  (0.18,4.18) -- (15.82,4.18);

\path[block] (0.40,1.10) rectangle (2.50,2.60);
\node[heading] at (1.45,2.27) {Same source};
\node[small, align=center] at (1.45,1.71) {program +\\input contract};
\path[block, draw=hlsRtl, line width=0.85pt]
  (3.05,1.10) rectangle (4.40,2.60);
\node[heading] at (3.725,2.07) {Agent};
\node[small] at (3.725,1.65) {per route};

\path[block] (5.05,2.13) rectangle (9.10,2.93);
\node[heading] at (7.075,2.70) {Behavior \(B\)};
\node[small] at (7.075,2.34) {Frontend / interface checks};
\path[block] (5.05,0.88) rectangle (9.10,1.68);
\node[heading] at (7.075,1.45) {HLS C++ \(\to\) Vitis \(\to\) RTL};
\node[small] at (7.075,1.09) {RTL interface checks};
\path[group] (9.75,1.10) rectangle (12.95,2.60);
\node[heading] at (11.35,2.29) {BEHAVE-Sim};
\node[small] at (11.35,1.85) {Behavior / RTL};
\node[small] at (11.35,1.43) {vs.\ hidden \(B^\star\)};
\path[block, draw=hlsObs] (13.60,1.10) rectangle (15.60,2.60);
\node[heading] at (14.60,2.29) {Final score};
\node[small] at (14.60,1.85) {after submission};
\node[small] at (14.60,1.43) {(both settings)};

\draw[flow] (2.50,1.85) -- (3.05,1.85);
\draw[flow] (4.40,2.25) -- (4.72,2.25) -- (4.72,2.53) -- (5.05,2.53);
\draw[flow] (4.40,1.45) -- (4.72,1.45) -- (4.72,1.28) -- (5.05,1.28);
\draw[flow] (9.10,2.53) -- (9.43,2.53) -- (9.43,2.25) -- (9.75,2.25);
\draw[flow] (9.10,1.28) -- (9.43,1.28) -- (9.43,1.45) -- (9.75,1.45);
\draw[flow] (12.95,1.85) -- (13.60,1.85);

\draw[feedback] (8.40,2.93) -- (8.40,3.18) -- (4.05,3.18) -- (4.05,2.60);
\node[wirelabel] at (6.225,3.40) {Frontend feedback (both)};
\draw[feedback] (8.40,0.88) -- (8.40,0.40) -- (4.05,0.40) -- (4.05,1.10);
\node[wirelabel] at (6.225,0.62) {Synthesis feedback (both)};
\draw[optional] (11.35,2.60) -- (11.35,3.74) -- (3.38,3.74) -- (3.38,2.60);
\node[wirelabel] at (7.365,3.96) {Correctness feedback (Check only)};
\end{tikzpicture}
\endgroup}
\caption{\textbf{Behavior IR and HLS comparison workflow.}
Both settings provide frontend or synthesis feedback. Only Check returns
development-time correctness feedback (dashed arrow). Both settings evaluate
the submitted artifact against the same hidden reference.}
\label{fig:hls-workflow}
\end{figure}
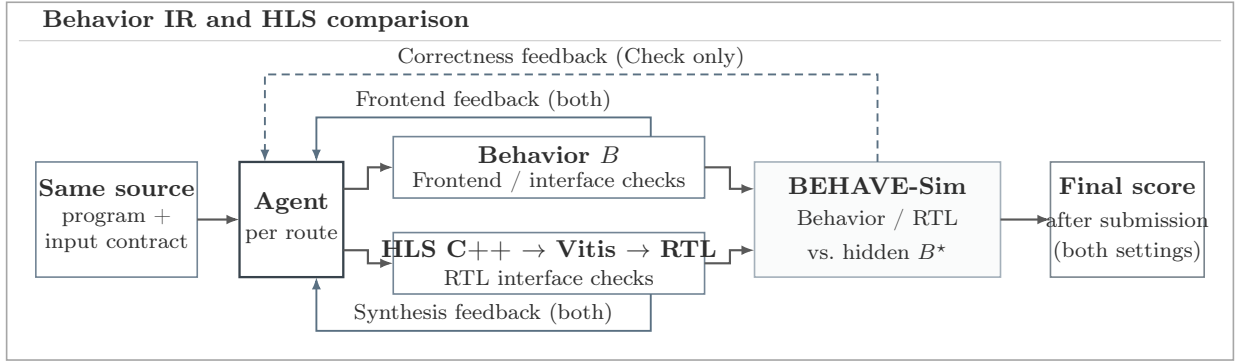
\par

\paragraph{Comparison settings.}
We use GPT-5.6 Luna with the rollout and baseline settings in
Appendix~\ref{app:exp-setup}.
Vitis HLS 2025.2 targets \texttt{xc7z020clg400-1} at a 25\,ns clock period,
with a 1200\,s synthesis budget.
Tool time sums development-check wall times, excluding model generation,
queueing and final evaluation. We report the median and interquartile range
across tasks.
\par

\subsection{Pipeline Comparison and PPA Exploration Settings}
\label{app:exp-pipelines}
\label{app:exp-ppa}

\paragraph{Comparison settings.}
RTL-only, ChipMATE and MAGE receive specifications without the Behavior
interface. ChipMATE~\citep{yu2026chipmate} uses three RTL/Python pairs per
round, up to three rounds and 1{,}000 random inputs per comparison.
MAGE~\citep{zhao2025mage} allows four simulation-review iterations in its
no-golden workflow and samples 20 RTL candidates, selecting two for up to
15 editor steps each.
Mean API cost includes all candidate-generation and repair calls, averaged
over all tasks.

\paragraph{Baseline adaptations.}
We adapt released implementations to our model and sandbox interfaces,
retaining their roles and development tests. For ChipMATE, we remove
repair-context clipping and correct output matching and zero-score candidate
selection. For MAGE, we fix the no-golden prompt, editor dispatch and
compilation/simulation failure handling. Our code includes all patches.

\paragraph{PPA exploration.}
For MXFP4, both settings jointly develop RTL and Behavior models with area and
timing feedback. The Sim-enabled setting additionally compares the agent's RTL and
Behavior during development. For each model and setting, we run three searches
of at most 60 turns, with initial hints of 1, 4 or 32 parallel processing lanes.
These hints guide exploration without constraining the final architecture.
Each search explores area-time trade-offs and retains intermediate candidates.
For NVFP4, each model runs five searches without development-time
\behavesim feedback, with initial hints of 1, 2, 4, 8 or 16 lanes.

\paragraph{Measurement and reporting.}
MXFP4 candidates use a fixed clock \(f_0=80\,\mathrm{MHz}\) and the same
64-block workload, with 32 bfloat16 (BF16) values per block. Inputs are offered
continuously under ready/valid handshaking. Execution time is
\(\tau_W(R)=N_W(R)/f_0\), where \(N_W(R)\) counts cycles from the first input
offered to the final output accepted. MXFP4 measurements use the
SkyWater 130\,nm high-density (HD) standard-cell library~\citep{skywater2020pdk}
at the typical-typical process corner
(25\,\textdegree C, 1.80\,V). Cell area and setup/hold timing are measured on
the same placed and repaired netlist, with interconnect resistance and
capacitance estimated from placement.
NVFP4 uses the Nangate 45\,nm standard-cell
library~\citep{nangate2008library}, based on FreePDK45~\citep{stine2007freepdk},
at 400\,MHz with a 64-block workload of 16 BF16 values per block.
We report nondominated area-time points among candidates that pass independent
golden-Behavior and timing checks and complete the full workload.

\subsection{Training and Self-Improvement Settings}
\label{app:exp-training}
\label{app:exp-self-improvement}

\paragraph{Task-based RL.}
We use Slime~\citep{zhu2025slime} with Megatron~\citep{narayanan2021efficient}
for synchronous GRPO~\citep{shao2024deepseekmath}.
Hyperparameters are listed in Table~\ref{tab:exp-training-config}.

\begin{table}[htbp]
\centering
\small
\caption{Training hyperparameters shared by fixed-dataset RL
and self-improvement.}
\label{tab:exp-training-config}
\begin{tabular}{@{}p{0.46\linewidth}p{0.48\linewidth}@{}}
\toprule
Setting & Value \\
\midrule
Model & Qwen3.8-27B \\
Parameter updates and precision & Full-parameter, BF16 \\
GPUs & 8 NVIDIA H200 \\
Training parallelism (tensor/context/pipeline) & 4 / 2 / 1 \\
Rollout serving & 2 SGLang engines, tensor parallelism 4 \\
\midrule
Maximum concurrent conversations & 32 \\
Update batch & 8 tasks $\times$ 8 rollouts \\
Microbatch size & 1 \\
Sampling & Temperature 1, top-$p$ 1, top-$k$ disabled \\
Repetition/presence/frequency penalties & None \\
Reasoning & Thinking enabled, \texttt{xhigh} effort \\
Context window & 131{,}072 tokens \\
Maximum model turns & 60 \\
\midrule
Optimizer & Adam, $(\beta_1,\beta_2)=(0.9,0.98)$ \\
Learning rate & $10^{-6}$ (constant) \\
Weight decay & 0.1 \\
Gradient-norm clip & 1.0 \\
Policy-ratio clip & $[0.8,1.2]$ \\
KL and entropy loss coefficients & 0 \\
Attention and hidden dropout & 0 \\
CPU optimizer offload & Enabled \\
Activation recomputation & Full \\
\bottomrule
\end{tabular}
\end{table}

\paragraph{Context management.}
When the training context window fills, we clear interaction history
but retain task instructions and workspace files. Evicting individual turns
produces overlapping long sequences with high training cost.
Within each window, we merge turns with matching token prefixes
into one causal sequence to avoid recomputing shared context. Evaluation
follows the settings in Section~\ref{app:exp-setup}.

\paragraph{Loss normalization.}
A rollout's segments share one advantage.
We normalize each rollout's loss by its loss-bearing token count
across segments, then average across rollouts, so segmentation
does not increase its weight. Only model-generated tokens contribute, including
reasoning and tool calls.

\paragraph{Task pool.}
We oversample up to 12 concurrent task groups, collecting
the first eight complete groups with nonconstant rewards per update.
Uniform-reward groups are replaced with new ones.
Tasks with all-1 or all-0 rewards enter cooldown for ten or two updates,
respectively, then rejoin the pool. Sampling prioritizes tasks without
a completed group, then those least recently observed.

\paragraph{Training data and budgets.}
Fixed-dataset RL samples from the 540-task BEHAVE-Train pool for 40 updates,
with evaluation every five updates. Self-improvement starts
from a 60-task subset of BEHAVE-Train, selected to match domain proportions
while favoring diverse, higher-complexity tasks. After three
initial updates, we run \selfimprovementrounds task-acquisition rounds with three
updates each, for 18 updates in total. Evaluation occurs every three updates.
Self-improvement adds 100 tasks to reach a 160-task pool.
The seed-only baseline resumes from the same three-update checkpoint
and trains only on the original 60 tasks for 15 further updates, without task acquisition.

\paragraph{Self-improvement step 1: Analysis and acquisition.}
Every three updates, five parallel Analyst instances run a proposal batch
using the current training policy. They receive disjoint sets of
complete task groups from the latest rollout, keeping each task's attempts
together. Each Analyst targets six proposals within 100 turns, searching
a fixed source-only library of 374 files or code excerpts from 67 repositories.
Source files used by evaluation tasks and exact content duplicates are excluded.
No preconstructed specification-behavior pairs are provided.
Experience is consolidated between batches,
retained across rounds and visible only to Analysts.

\paragraph{Self-improvement step 2: Construction and admission.}
Construction proceeds in parallel, with up to 100 turns per
role attempt. Each task allows up to three construction-and-review rounds,
including the initial draft. Revisions update existing drafts based on
feedback. Training resumes at 20 admissions. If the first batch
falls short, we allow another batch, then proceed with
admitted tasks.

\paragraph{Self-improvement step 3: Model updates.}
Admitted tasks join the cumulative training pool under the same
sampling and cooldown rules. We continue training for three GRPO updates
and reuse the latest training rollouts for the next acquisition round.

\subsection{Computational Resources and Cost}
\label{app:exp-resources}

\paragraph{Resource accounting.}
Table~\ref{tab:exp-resource-cost} summarizes GPU-hours and recorded API costs
for selected experiments.

\begin{table}[htbp]
\centering
\small
\caption{\textbf{Resources and recorded costs of selected experiments.}}
\label{tab:exp-resource-cost}
\setlength{\tabcolsep}{4pt}
\begin{tabular*}{\linewidth}{@{\extracolsep{\fill}}llr@{}}
\toprule
Experiment & Scope & Resource use / API cost \\
\midrule
Fixed-dataset RL & 40 updates & 883.5 H200 GPU-hours \\
Self-improvement & 18 updates with task acquisition & 537.7 H200 GPU-hours \\
Seed-only & 15 updates after shared update 3 & 340.7 H200 GPU-hours \\
Evaluation & 5 API models $\times$ 381 tasks & \$959.12 \\
Pipeline comparison & 5 workflows $\times$ 60 tasks & \$10.54 \\
PPA exploration & 12 MXFP4 + 10 NVFP4 searches & \$236.92 \\
\bottomrule
\end{tabular*}
\end{table}

\subsection{Pipeline Comparison Example}
\label{app:exp-additional}

\paragraph{Comparison example.}
Figure~\ref{fig:functional-modeling} shows 32-bit unsigned division.
ChipMATE preserves a counter and partial results across clock-level
\texttt{eval()} calls, whereas the Behavior reference completes the
calculation within one \texttt{process()} call.\par

\begin{figure}[!t]
\centering
\input{figures/functional_modeling/figure}
\caption{\textbf{Comparison example.}
(a) ChipMATE Python excerpt from a GPT-5.6 Luna rollout.
(b) The task's Behavior reference.}
\label{fig:functional-modeling}
\end{figure}

\subsection{Additional Results for RL Training and Self-Improvement}
\label{app:exp-self-improvement-case}

\paragraph{Learning dynamics.}
Figure~\ref{fig:checkpoint-dynamics} reports performance across all four benchmarks.
Figure~\ref{fig:optimizer-dynamics} shows per-update optimization measurements
for fixed-dataset RL and self-improvement.

\begin{figure}[!tp]
\centering
\includegraphics[width=0.91\linewidth]{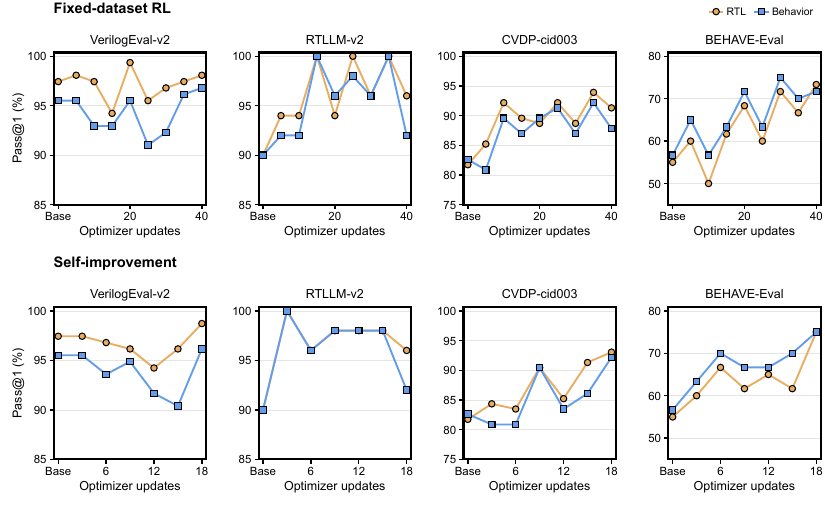}
\caption{\textbf{Checkpoint performance across benchmarks.}
Top: fixed-dataset RL; bottom: self-improvement. RTL and Behavior pass@1
are evaluated on the same benchmark tasks at each checkpoint.}
\label{fig:checkpoint-dynamics}
\vspace{\floatsep}
\includegraphics[width=0.91\linewidth]{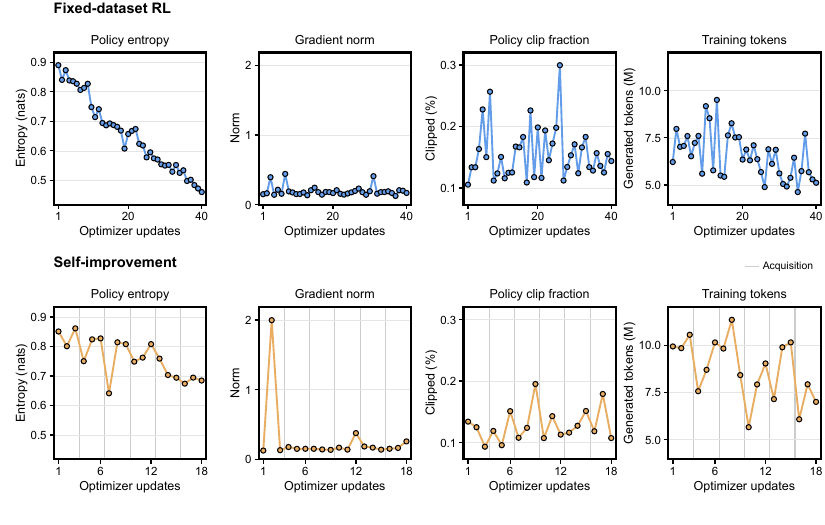}
\caption{\textbf{Training dynamics.}
Top: fixed-dataset RL; bottom: self-improvement.
Token counts include only generated tokens used for the update.
Gray vertical lines mark task acquisitions.}
\label{fig:optimizer-dynamics}
\end{figure}

\FloatBarrier
\paragraph{Case study.} Figure~\ref{fig:self-improvement-case} follows
task acquisition and rollouts on the seed task.
The agent turns rollout difficulties into focused practice through task
construction. Later progress on the original, more complex task shows that
improvement during continued training extends beyond the acquired task.

\begin{figure}[!ht]
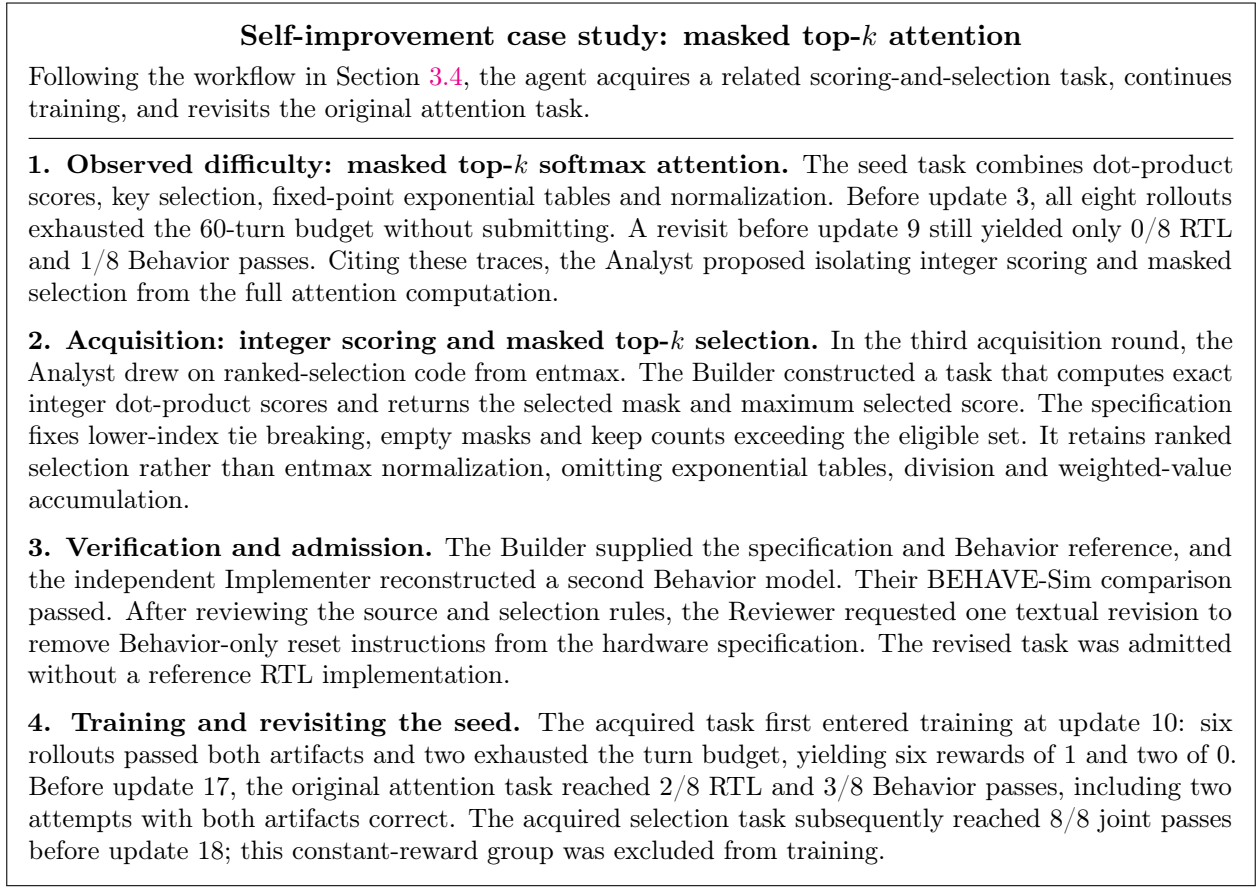

\centering
\begingroup
\setlength{\fboxsep}{8pt}
\setlength{\fboxrule}{0.4pt}
\fbox{\begin{minipage}{\dimexpr\linewidth-2\fboxsep-2\fboxrule\relax}
\small
{\centering
\normalsize\bfseries
Self-improvement case study: masked top-$k$ attention\par
}
\smallskip
\noindent Following the workflow in Section~\ref{sec:method-self-improvement},
the agent acquires a related scoring-and-selection task, continues training,
and revisits the original attention task.
\par\medskip
\hrule height 0.4pt
\medskip
\noindent\textbf{1. Observed difficulty: masked top-$k$ softmax attention.}
The seed task combines dot-product scores, key selection, fixed-point
exponential tables and normalization. Before update 3, all eight rollouts
exhausted the 60-turn budget without submitting. A revisit before update 9
still yielded only 0/8 RTL and 1/8 Behavior passes. Citing these traces, the
Analyst proposed isolating integer scoring and masked selection from the
full attention computation.

\medskip
\noindent\textbf{2. Acquisition: integer scoring and masked top-$k$ selection.}
In the third acquisition round, the Analyst drew on ranked-selection code
from entmax. The Builder constructed a task that computes exact integer
dot-product scores and returns the selected mask and maximum selected score.
The specification fixes lower-index tie breaking, empty masks and keep
counts exceeding the eligible set. It retains ranked selection rather than
entmax normalization, omitting exponential tables, division and
weighted-value accumulation.

\medskip
\noindent\textbf{3. Verification and admission.}
The Builder supplied the specification and Behavior reference, and the
independent Implementer reconstructed a second Behavior model. Their
\behavesim{} comparison passed. After reviewing the source and selection
rules, the Reviewer requested one textual revision to remove Behavior-only
reset instructions from the hardware specification. The revised
task was admitted without a reference RTL implementation.

\medskip
\noindent\textbf{4. Training and revisiting the seed.}
The acquired task first entered training at update 10: six rollouts passed
both artifacts and two exhausted the turn budget, yielding six rewards of
1 and two of 0. Before update 17, the original attention task reached
2/8 RTL and 3/8 Behavior passes, including two attempts with both artifacts
correct. The acquired selection task subsequently reached 8/8 joint passes
before update 18; this constant-reward group was excluded from training.
\end{minipage}}
\caption{\textbf{A recorded self-improvement case.}
After acquiring and training on a related task, the agent produces correct
RTL and Behavior solutions on the original seed task.}
\label{fig:self-improvement-case}
\endgroup
\end{figure}


\section{Dataset Construction and Validation}
\label{app:benchmark-review}

This appendix documents the benchmarks used in our experiments. We
describe their composition and sources (\ref{app:benchmark-sources}), BEHAVE
construction (\ref{app:benchmark-authored}), external benchmark adaptation
(\ref{app:benchmark-construction}), human review (\ref{app:benchmark-checks}),
and data splits and release (\ref{app:benchmark-splits}).

\subsection{Dataset Overview and Sources}
\label{app:benchmark-sources}

\paragraph{Dataset composition.}
BEHAVE covers real-world algorithms and system components
across the domains in Table~\ref{tab:dataset-sources}.
Appendix~\ref{app:benchmark-splits} describes its training and evaluation splits.
Table~\ref{tab:dataset-features} summarizes behavioral and
interface features across the two BEHAVE splits.
For evaluation only, we adapt all tasks from
VerilogEval-v2~\citep{pinckney2025verilogeval},
RTLLM-v2~\citep{liu2024openllmrtl}, and the cid003 subset of
CVDP~\citep{pinckney2025cvdp}.
Appendix~\ref{app:benchmark-construction} details these adaptations.

\begin{table}[!htb]
\centering
\small
\setlength{\tabcolsep}{4pt}
\renewcommand{\arraystretch}{1.15}
\caption{\textbf{Overview of the 600 source-derived tasks.}
Representative task families and sources are shown for each domain.
The complete task source list is provided in our code repository.}
\label{tab:dataset-sources}
\begin{tabular*}{\linewidth}{@{\extracolsep{\fill}}l r
  >{\raggedright\arraybackslash}p{0.32\linewidth}
  >{\raggedright\arraybackslash}p{0.34\linewidth}@{}}
\toprule
Domain & Tasks & Representative families & Representative sources \\
\midrule
AI/ML & 190 & Attention, normalization, quantization
  & Transformers, TorchAO \\
Numerics & 160 & Linear algebra, reductions, interpolation
  & CMSIS-DSP, LAPACK \\
Signal processing & 80 & Filtering, transforms, resampling, modulation
  & liquid-dsp, librosa, Opus \\
Cryptography & 80 & Symmetric primitives, field arithmetic, post-quantum kernels
  & PyCryptodome, liboqs, HEXL \\
Control & 30 & Coordinate transforms, controllers, estimation
  & SimpleFOC, FilterPy \\
Compression & 30 & Entropy coding, stream packing, match operations
  & FLAC, Brotli, Zstd \\
Data systems & 30 & Vector search, hashing, bitmap operations
  & Faiss, hnswlib, CRoaring \\
\midrule
\textbf{Total} & \textbf{600} & & \\
\bottomrule
\end{tabular*}
\end{table}

\begin{table}[!htb]
\centering
\small
\setlength{\tabcolsep}{4pt}
\renewcommand{\arraystretch}{1.15}
\caption{\textbf{Behavioral and interface features of BEHAVE.}
Features overlap.
A 1:1 mapping denotes one output transaction per input
transaction and does not imply low implementation complexity.}
\label{tab:dataset-features}
\arrayrulecolor{black}
\begin{tabular*}{\linewidth}{@{\extracolsep{\fill}}
  >{\raggedright\arraybackslash}p{0.34\linewidth}
  >{\raggedright\arraybackslash}p{0.38\linewidth}r r r@{}}
\toprule
Feature & Representative behavior & Train & Eval & Total \\
\midrule
Single-stream 1:1 input/output & Block quantization & 492 & 56 & 548 \\
Single-stream non-1:1 input/output & Filtering, accumulation, interpolation & 47 & 4 & 51 \\
Multiple independent streams & Three-input, two-output crossbar & 1 & 0 & 1 \\
External memory access & Memory-backed vectors, lookup tables & 26 & 2 & 28 \\
External memory writes & In-place updates, writeback & 8 & 1 & 9 \\
\bottomrule
\end{tabular*}
\arrayrulecolor{black}
\end{table}

\paragraph{Task sources.}
We select algorithms and system components from high-level implementations,
supplemented by algorithm descriptions where needed. The task source list in
our code repository identifies the files or documents used for each task, along
with available revisions and license notices.
Appendix~\ref{app:benchmark-authored} describes their adaptation into
hardware-oriented tasks.

\paragraph{Task size.}
Table~\ref{tab:dataset-code-size} summarizes the sizes of the golden Behavior
models and auxiliary RTL. We count nonblank, non-comment lines in task-local
source files, excluding shared libraries.

\begingroup
\newcommand{\behaveTaskSizeRows}{%
VerilogEval-v2 & 156 & 270.5 [211.5, 327.5] & 9 [7, 13] & 11.5 [7, 21] \\
RTLLM-v2 & 50 & 312 [231.75, 427] & 14 [9, 24] & 20.5 [15, 33.75] \\
CVDP-cid003 & 115 & 473 [358.5, 615.5] & 28 [16.5, 47] & 34 [21, 53] \\
\midrule
\textbf{BEHAVE (ours)} & \textbf{600} & \textbf{633 [496.75, 861.5]} & \textbf{40 [24, 69]} & \textbf{97 [56, 183]} \\
\quad Train & 540 & 623.5 [498.75, 869.75] & 39 [24, 67.25] & 97 [56, 177.25] \\
\quad Eval & 60 & 645.5 [481.25, 808] & 43.5 [22, 70.25] & 102.5 [60.75, 209.5] \\
}

\begin{table}[!htbp]
\centering
\small
\setlength{\tabcolsep}{3pt}
\renewcommand{\arraystretch}{1.15}
\caption{\textbf{Specification and code size by task collection.}
Sizes are medians [interquartile ranges].
Spec words count whitespace-separated units, including interface declarations.
LOC: lines of code.}
\label{tab:dataset-code-size}
\arrayrulecolor{black}
\begin{tabular*}{\linewidth}{@{\extracolsep{\fill}}l r r r r@{}}
\toprule
Collection & Tasks & Spec words & Behavior LOC & RTL LOC \\
\midrule
\behaveTaskSizeRows
\bottomrule
\end{tabular*}
\arrayrulecolor{black}
\end{table}
\endgroup

\subsection{Source-Based Task Construction}
\label{app:benchmark-authored}

\paragraph{Human-agent collaboration.}
Researchers defined application domains, guided source selection and
task scope, reviewed and revised tasks. The agent
developed specifications and Behavior models, revising them using
execution results and human feedback.

\paragraph{Construction workflow.}
Builder jointly constructed \(S\) and \(B_1\) from the sources, specifying numerical
and interface rules. In a separate context, Implementer produced
\(B_2\) using only \(S\) and public interface documentation. Reviewer checked
\(S\) against the sources and both models against \(S\), requesting revisions
as needed. Material changes to \(S\) triggered a fresh reconstruction of \(B_2\).
This follows the workflow in Section~\ref{sec:method-self-improvement}.
All tasks then underwent human review before
acceptance. Researchers made corrections directly or requested revisions from
the agent, then approved each task's final version. See
Appendix~\ref{app:benchmark-checks} for review details.

\paragraph{Auxiliary RTL and tools.}
We also asked the agent to generate reference RTL for each task to support
benchmark reuse. This auxiliary artifact is not required for evaluation,
which uses the Behavior reference. Both the agent and researchers could use
\behavesim\ throughout construction to test and compare implementations.
Executable agreement is insufficient because implementations may share a
misunderstanding. Human review therefore checked that the
NL specification clearly stated the intended behavior and that
the Behavior model faithfully implemented it.

\subsection{External Benchmark Adaptation}
\label{app:benchmark-construction}

\paragraph{Framework integration.}
Our goal was to evaluate candidate implementations across all datasets using
the same \behavesim\ procedure. We therefore integrated the external benchmarks
into our framework rather than using their original testbenches as separate
evaluators. For each task, we paired an NL specification with
a Behavior reference and provided auxiliary reference RTL.

\paragraph{Source issues.}
We also identified inconsistencies and omissions in the original materials used for
adaptation. For example, in the RTLLM-v2 \texttt{edge\_detect} task, the description first
requires outputs to return to zero after an edge, then states that they remain
high until another edge. The reference RTL implements single-cycle pulses.
In CVDP task \nolinkurl{cvdp_copilot_sorter_0001}, the specification requires
ascending order but does not define the element order on the packed output
bus. The testbench assumes that the smallest element occupies the
least-significant bits.

\paragraph{Adaptation and validation.}
We resolved ambiguities and inconsistencies with explicit rules,
then aligned the specification, Behavior reference and auxiliary RTL.
We retained native interfaces and specified timing wherever
possible, documenting substantive changes and their rationale. Adapted tasks
underwent construction checks and human review in
Appendices~\ref{app:benchmark-authored} and~\ref{app:benchmark-checks}.

\subsection{Human Review}
\label{app:benchmark-checks}

\paragraph{Review scope.}
Researchers reviewed all 921 tasks across BEHAVE, VerilogEval-v2, RTLLM-v2
and CVDP with agent assistance, using the five criteria in
Table~\ref{tab:dataset-human-review}. They checked task specifications and
interfaces, inspecting source material, implementations and execution records
where needed.

\paragraph{Revision and approval.}
Researchers made corrections directly or provided feedback to the agent for
revision. Revised tasks repeated the applicable reconstruction and executable
checks in Appendix~\ref{app:benchmark-authored}. A researcher approved the final
version before \(B_1\) was frozen as \(B^\star\). Review records link each finding
to its correction, checked version and final disposition.

\paragraph{Review outcomes.}
Table~\ref{tab:dataset-human-review} classifies researcher-requested revisions
by their primary reason, counting each revised task once.
Table~\ref{tab:dataset-review-stages} summarizes recorded revision rounds
across construction stages. Most Behavior comparisons required no revision.
\begingroup
\newcommand{\behaveReviewSpecification}{6}
\newcommand{\behaveReviewNumerical}{2}
\newcommand{\behaveReviewInterface}{14}
\newcommand{\behaveReviewImplementation}{7}
\newcommand{\behaveReviewEvidence}{1}
\newcommand{\behaveReviewTotal}{30}
\newcommand{\behaveReviewStageRows}{%
\(B_2\) comparison & 6 (0.65\%) & 915 & 6 & 0 \\
Agent review & 14 (1.52\%) & 907 & 14 & 0 \\
Human review & 30 (3.26\%) & 891 & 29 & 1 \\
}

\begin{table}[H]
\centering
\small
\setlength{\tabcolsep}{4pt}
\renewcommand{\arraystretch}{1.15}
\caption{\textbf{Human review and revisions.}
Each revised task is counted once by primary category.}
\label{tab:dataset-human-review}
\arrayrulecolor{black}
\begin{tabular*}{\linewidth}{@{\extracolsep{\fill}}
  >{\raggedright\arraybackslash}p{0.25\linewidth}
  >{\raggedright\arraybackslash}p{0.53\linewidth}
  >{\raggedleft\arraybackslash}p{0.15\linewidth}@{}}
\toprule
Review item & Researcher checks & Tasks revised, \(n\) \\
\midrule
Specification consistency & Requirements and examples agree, with valid inputs,
  parameter constraints and boundary behavior defined. & \behaveReviewSpecification \\
Numerical semantics & Widths, signedness, overflow, rounding, special values and
  operation order are explicit where relevant. & \behaveReviewNumerical \\
Interface and state & Packing, transaction counts and ordering, reset,
  backpressure and timing are defined where applicable. & \behaveReviewInterface \\
Implementation fidelity & Behavior models and auxiliary RTL match the specified
  function, state transitions and memory effects. & \behaveReviewImplementation \\
Evidence integrity & Source and execution records match reviewed versions
  and runtime settings, including retests. & \behaveReviewEvidence \\
\midrule
\textbf{Total} & & \textbf{\behaveReviewTotal} \\
\bottomrule
\end{tabular*}
\arrayrulecolor{black}
\end{table}

\begin{table}[H]
\centering
\small
\setlength{\tabcolsep}{4pt}
\renewcommand{\arraystretch}{1.15}
\caption{\textbf{Task revisions by construction stage.}
Counts cover all 921 tasks: 600 BEHAVE tasks and 321 external-benchmark
adaptations. Revised gives the number and
percentage of revised tasks. The remaining columns group tasks by
recorded feedback-driven revision rounds.}
\label{tab:dataset-review-stages}
\arrayrulecolor{black}
\begin{tabular*}{\linewidth}{@{\extracolsep{\fill}}l r r r r@{}}
\toprule
Stage & Revised, \(n\) (\%) & 0 rounds & 1 round & \(\geq 2\) rounds \\
\midrule
\behaveReviewStageRows
\bottomrule
\end{tabular*}
\arrayrulecolor{black}
\end{table}
\endgroup

\subsection{Data Splits and Release}
\label{app:benchmark-splits}

\paragraph{Data splits.}
We reserve 10\% of the tasks in each domain for evaluation, yielding
\trainingdatasetsize training tasks and \evaluationdatasetsize evaluation tasks
with the same domain proportions. Both sets cover a range of difficulties.
Tasks sharing a source implementation or differing only in
parameters remain in the same split, together with their source materials.
External benchmark adaptations are evaluation-only. Evaluation tasks and
references are frozen before training and excluded from task acquisition.

\paragraph{Release.}
The release includes task specifications, golden Behavior models, auxiliary
RTL and the interface information needed to run them. Source lists and review
records document where each task came from, what was changed and which checks
were run.

\paragraph{Licensing.}
Our original contributions will be released under Apache-2.0. Third-party
materials retain their applicable licenses and required notices, documented
in the task source list.

\end{document}